\documentclass[10pt,twocolumn,letterpaper]{article}

\usepackage{cvpr}              % To produce the CAMERA-READY version
\newcommand{\myequiv}{\equiv}
\usepackage{booktabs}       % professional-quality tables
\usepackage{amsfonts}       % blackboard math symbols
\usepackage{nicefrac}       % compact symbols for 1/2, etc.
\usepackage{microtype}      % microtypography
\usepackage{xcolor}         % colors
\usepackage{colortbl}
\usepackage{wrapfig}
\usepackage{subcaption}
\usepackage{graphicx}
\usepackage{adjustbox}
\usepackage{tabularx}
\usepackage{multirow}
\usepackage{pifont}
\usepackage{amsmath}
\usepackage{marginnote}
\usepackage{amssymb}
\usepackage{mathtools}
\usepackage{amsthm}
\usepackage{thmtools}
\usepackage{algorithm}
\usepackage{algpseudocode}
\usepackage{xstring}
\usepackage{ifthen}
\usepackage[accsupp]{axessibility}
\newcommand{\val}[2]{$#1_{\textcolor{gray}{\pm #2}}$}
\newcommand{\norm}[1]{\left\lVert#1\right\rVert}

\theoremstyle{definition}
\newtheorem{definition}{Definition}
\newtheorem{observation}{Observation}

\definecolor{cvprblue}{rgb}{0.21,0.49,0.74}
\usepackage[pagebackref,breaklinks,colorlinks,allcolors=cvprblue]{hyperref}
\usepackage{mathtools}
\renewcommand\myequiv{\stackrel{\mathclap{\normalfont\mbox{\tiny equiv}}}{\equiv}}

\def\confName{CVPR}
\def\confYear{2026}

\title{\ensuremath{\oslash} Source Models Leak What They Shouldn’t \ensuremath{\nrightarrow}%\ensuremath{\oslash}
: Unlearning Zero-Shot Transfer in Domain Adaptation Through Adversarial Optimization}

\author{
Arnav Devalapally$^{1,2,*,\dagger}$, 
Poornima Jain$^{1,*}$, 
Kartik Srinivas$^{1,3,\dagger}$, and 
Vineeth N. Balasubramanian$^{1,4}$ \vspace{5pt} \\ 
\small{$^1$Indian Institute of Technology, Hyderabad $\;$ 
$^2$University of Michigan $\;$ 
$^3$Carnegie Mellon University $\;$
$^4$Microsoft Research$\;$} \\ 
\tt\footnotesize darnav@umich.edu, \{ai24resch11002@, vineethnb@cse.\}iith.ac.in, kartiksr@cs.cmu.edu, \\ \tt\footnotesize vineeth.nb@microsoft.com \\
}

\begin{document}
\maketitle
\begin{abstract}
%Machine Unlearning (MU) is becoming an increasingly important field to ensure models comply with privacy and ownership-based data deletion requests. While most MU research has focused on unlearning in image classification tasks with stationary data domains, unlearning in scenarios where models adapt to new domains remains largely unexplored.Our work herein studies this unexplored problem in a setting where a model, adapted from a source to a target domain, must unlearn the influence of classes that are exclusive to the source domain and absent in the target, with no access to the source model training data. Existing source-free domain adaptation methods tend to leak source class information through the model, necessitating such a targeted effort. To this end, we propose a new unlearning method, where an adversarially generated forget class sample is thoroughly unlearnt by the model during the domain adaptation process using a novel rescaled labeling strategy through adversarial optimization.We also study the extensions of this method to a continual variant of this problem setting and to one where the specific source classes to be forgotten may be unknown.Alongside theoretical interpretations, our comprehensive empirical results show that our method consistently outperforms baseline approaches on benchmark datasets.

%(\p Poornima: Alternative flow: 
The increasing adaptation of vision models across domains, such as satellite imagery and medical scans, has raised an emerging privacy risk: models may inadvertently retain and leak sensitive source-domain specific information in the target domain. This creates a compelling use case for machine unlearning to protect the privacy of sensitive source-domain data. Among adaptation techniques, source-free domain adaptation (SFDA) calls for an urgent need for machine unlearning (MU), where the source data itself is protected, yet the source model exposed during adaptation encodes its influence. Our experiments reveal that existing SFDA methods exhibit strong zero-shot performance on source-exclusive classes in the target domain, indicating they inadvertently leak knowledge of these classes into the target domain, even when they are not represented in the target data. We identify and address this risk by proposing an MU setting called SCADA-UL: \textbf{U}n\textbf{l}earning \textbf{S}ource-exclusive \textbf{C}l\textbf{A}sses in \textbf{D}omain \textbf{A}daptation. Existing MU methods do not address this setting as they are not designed to handle data distribution shifts. We propose a new unlearning method, where an adversarially generated forget class sample is unlearned by the model during the domain adaptation process using a novel rescaled labeling strategy and adversarial optimization.
We also extend our study to two variants: a continual version of this problem setting and to one where the specific source classes to be forgotten may be unknown.
Alongside theoretical interpretations, our comprehensive empirical results show that our method consistently outperforms baselines in the proposed setting while achieving retraining-level unlearning performance on benchmark datasets.
Our code is available at \href{https://github.com/D-Arnav/SCADA}{https://github.com/D-Arnav/SCADA}.
\end{abstract}   
\let\thefootnote\relax\footnotetext{$^*$Equal Contribution.}
\let\thefootnote\relax\footnotetext{$^\dagger$ Majority of work done at Indian Institute of Technology, Hyderabad}
\vspace{-20pt}
\section{Introduction}
\label{sec:intro}

\begin{figure}
    \centering
    \includegraphics[width=0.99\columnwidth]{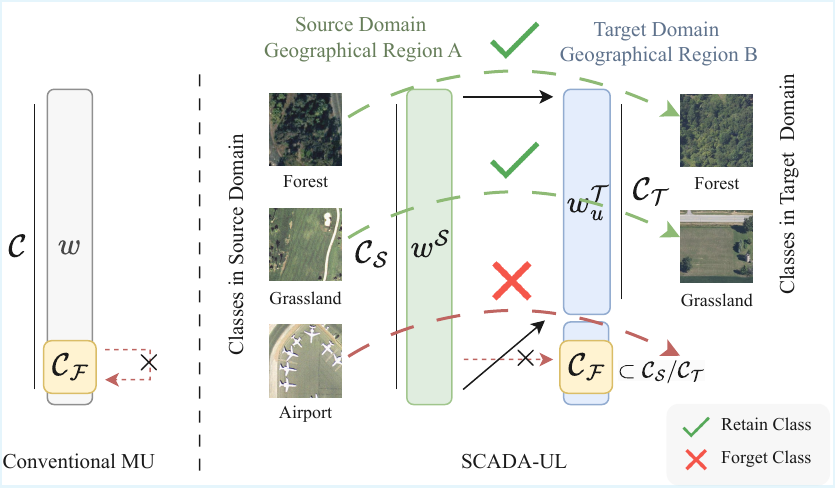}
    \captionsetup{font=footnotesize}
    \vspace{-9pt}
    \caption{\footnotesize \textbf{Comparison of Conventional MU and Proposed SCADA-UL.} Conventional class-wise machine unlearning focuses on forgetting a subset of classes from a trained model. On the other hand, the proposed SCADA-UL aims to remove knowledge of source-exclusive classes (classes absent in target domain) while adapting a model to a new domain. For instance, if a land-use categorization model is adapted to a new geography, sensitive classes such as airports must not be transferred to the target domain.}
    \vspace{-20pt}
    \label{fig:setting}
\end{figure}

% \begin{wrapfigure}[17]{r}{0.49\textwidth}
%     \centering
%     \begin{subfigure}{0.49\linewidth}
%         \centering
%         \includegraphics[width=0.38\linewidth]{fig/pdf/setting_reg.pdf}
%         \caption{\footnotesize Conventional MU}
%     \end{subfigure}
%     \hfill
%     \begin{subfigure}{0.49\linewidth}
%         \centering
%         \includegraphics[width=0.9\linewidth]{fig/pdf/setting_da.pdf}
%         \caption{\footnotesize SCADA-UL}
%     \end{subfigure}
%     \captionsetup{font=footnotesize}
%     \caption{\footnotesize \textbf{Comparison of Conventional MU and Proposed SCADA-UL.} Conventional class-wise machine unlearning focuses on forgetting a subset of classes from a trained model. On the other hand, the proposed SCADA-UL aims to remove knowledge of source-exclusive classes (classes absent in target domain) while adapting a model to a new domain.}
%     \label{fig:setting}
% \end{wrapfigure}
%(\p Poornima: If we use the second narrative, we can start from the "Why unlearning in DA?" section and shift first two paras after it) 
\vspace{-4pt}
The widespread use of large-scale learning models trained on vast corpora of data has raised significant concerns in data ownership, copyrighting and privacy in recent times \cite{haim2022reconstructing}. Recent legislation \cite{gdpr2016,ccpa2018} demands user data to be deleted on request, including any impact that the data may have on the output of a trained model. To address these concerns, Machine Unlearning (MU) has assumed increased importance, especially because trivial solutions such as retraining the model from scratch without the sensitive data may be expensive, slow or simply infeasible due to non-availability of the training data.

MU efforts over the last few years have broadly focused on methods to forget a subset of the training dataset \cite{foster2024zero, golatkar2020eternal} or specific classes of data \cite{baumhauer2022machine, chundawat2023zero, tarun2023fast, yamashita2023one, panda2025partially} in a given domain. These efforts have shown promising results and have validated the use of MU methods for protecting sensitive data as well as discarding unwanted data. However, finetuning and adaptation of models to different domains has played a key role in adoption of learning models for many years now, and unlearning in such non-stationary settings has seen little effort hitherto. We focus on addressing this need herein.  

\vspace{3pt}
\textbf{\textit{Why unlearning in domain adaptation?}} MU in domain adaptation (DA) settings has important real-world applications, especially in scenarios where the model is adapted from a source domain having classes with sensitive information. Consider a land-use categorization application (as in Fig~\ref{fig:setting}) where a model is adapted to a new geography but sensitive regions such as government facilities, army regions, airports and other private land categories must not be transferred to the target domain. Similarly, in a fraud surveillance application, a model being adapted between two different environments (say, two countries) may need to forget certain classes due to legal restrictions. %(\p Poornima: should we remove the e-commerce example (doesn't align with privacy exactly)?) An e-commerce product classification model that is being transferred between two retail outlets may need to unlearn certain item categories (adult items, for example). \f 
For another example, consider a disease diagnosis model initially trained on a dataset containing both mental and physical health conditions. If this model is later deployed in a hospital setting where patient privacy policies prohibit the use of mental health data, it becomes mandatory to forget the mental health-specific classes. Retaining such classes, even if their outputs are suppressed, poses a privacy risk. Internal representations might encode sensitive features associated with sensitive classes, inadvertently revealing protected information when processing new inputs. This can lead to unintended information leakage or even re-identification risks. 

Image classifiers are particularly vulnerable to model inversion attacks~\cite{classifier_risk1, classifier_risk2} that reconstruct inputs from confidence scores. ~\cite{classifier_risk1} lists other common privacy leaks in classification models such as: (i) membership inference attack that reveals whether a person's record was in training; (ii) attribute inference that infers hidden/sensitive features from outputs; (iii) gradient leakage during training in federated/distributed settings that can reveal pixel-accurate images and labels~\cite{classifier_risk3}; and (iv) black‑box model extraction (stealing a classifier via its prediction API). % further enabling an attacker to obtain a close surrogate and to run stronger white‑box inversion or membership attacks offline. 
Therefore, merely masking or suppressing classifier outputs may not constitute true unlearning. Robust defenses must ensure minimal (ideally zero) residual traces of sensitive information in the adapted model parameters, an essential requirement for compliance with data protection regulations such as HIPAA and the GDPR.

\setlength{\columnsep}{8pt}
\begin{wrapfigure}[9]{r}{0.6\linewidth}
    \centering
    \vspace{-17pt}    
    \includegraphics[trim=0 6 0 0,clip,width=\linewidth]{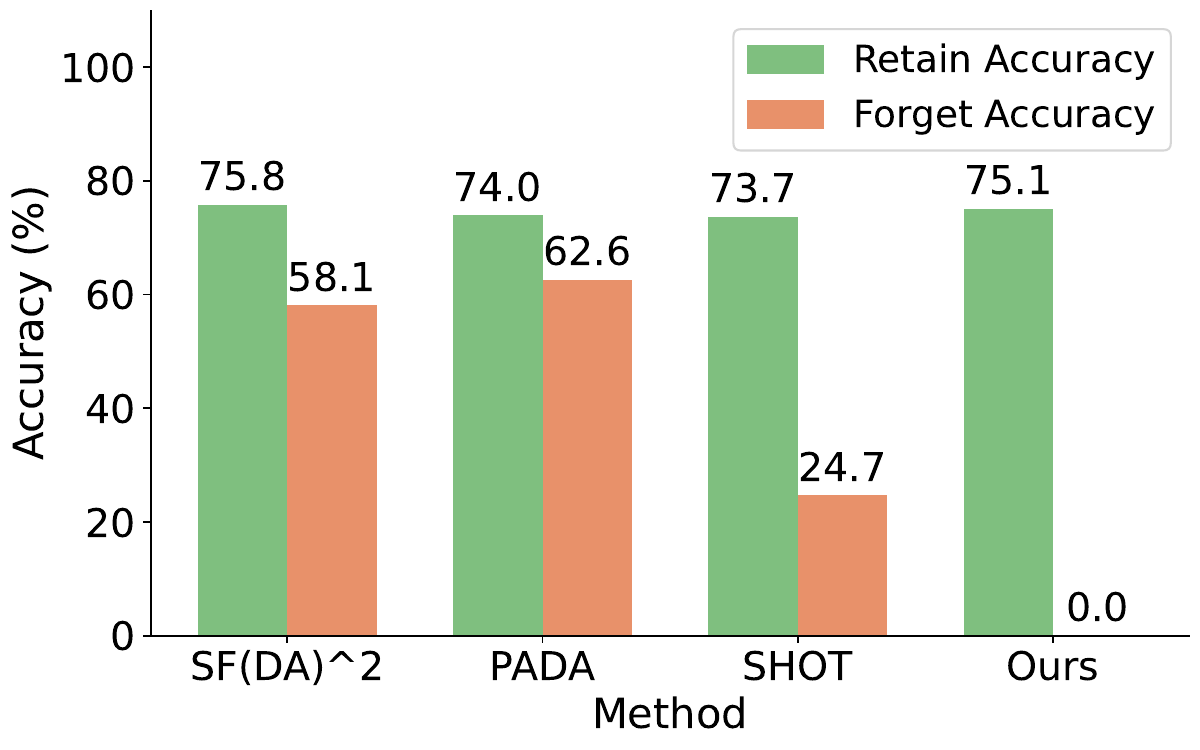}
    \vspace{-20pt}
    \captionsetup{font=footnotesize}
    \caption{\footnotesize Existing SFDA/PDA methods leak source-exclusive classes in the target domain.}
    \label{fig:sfda-leak}
\end{wrapfigure}
In this work, we propose a motivated methodology towards MU for source-free domain adaptation (SFDA) settings, where a model is adapted from a source domain to a target domain, and certain source-exclusive classes have to be unlearned, with no access to the source-domain data itself. This is a challenging, practical and new setting -- which we call \textbf{U}n\textbf{l}earning \textbf{S}ource-exclusive \textbf{C}l\textbf{A}sses in \textbf{D}omain \textbf{A}daptation (SCADA-UL) (Fig~\ref{fig:setting}) -- that has seen little concerted effort so far. Our empirical studies show that existing SFDA and Partial Domain Adaptation (PDA) methods \cite{liang2020we, hwang2024sf, cao2018partial} exhibit strong zero-shot performance on source-exclusive classes in the target domain, indicating they inadvertently leak knowledge of these classes into the target domain, even when they are not represented in the target data (Fig~\ref{fig:sfda-leak}). We note that our setting is stricter than PDA: while PDA also mitigates the transfer of source-exclusive classes to the target domain, SCADA-UL explicitly requires \textit{complete unlearning} of these classes. %, motivating the need for specialized strategies to address this setting. 
We compare against PDA methods in our experiments. Our method is also applicable in the Open Set Domain Adaptation setting, which we also explore in our empirical studies. %and our experiments show good performance in this scenario too. Just for the sake of simplicity, we assume the source free partial domain adaptation setting in our experiments. 
% \p Also, we perform successful ablation studies on our method in the Open Set Domain Adaptation scenario. For the sake of simplicity, we have assumed that the target domain is a subset of the source domain, and our method's utility does not depend on this assumption.
\begin{table}
    \centering
    \captionsetup{font=footnotesize}
    \caption{\footnotesize \textbf{Comparison of our setting with existing efforts}}
    \vspace{-6pt}
    \label{tab:comparison}
    \def\arraystretch{1.2}
    \scriptsize
    \begin{tabularx}{0.99\columnwidth}{Xccc}
        \toprule
        \multicolumn{1}{c}{\textbf{Existing Efforts}} & \textbf{Unlearning} & \textbf{Adaptation} & \textbf{Forget Data-free} \\
        \midrule
        SFDA/PDA \cite{liang2020we,hwang2024sf,cao2018partial} & \scalebox{1.1}{\textcolor{red!80!black}{\ding{55}}} & \scalebox{1.1}{\textcolor{green!70!black}{\ding{51}}} & \scalebox{1.1}{N/A} \\
        MU \cite{golatkar2020eternal,foster2024zero} & \scalebox{1.1}{\textcolor{green!70!black}{\ding{51}}} & \scalebox{1.1}{\textcolor{red!80!black}{\ding{55}}} & \scalebox{1.1}{\textcolor{red!80!black}{\ding{55}}} \\
        Zero Glance MU~\cite{tarun2023fast, chundawat2023zero} & \scalebox{1.1}{\textcolor{green!70!black}{\ding{51}}} & \scalebox{1.1}{\textcolor{red!80!black}{\ding{55}}} & \scalebox{1.1}{\textcolor{green!70!black}{\ding{51}}} \\
        SCADA-UL & \scalebox{1.1}{\textcolor{green!70!black}{\ding{51}}} & \scalebox{1.1}{\textcolor{green!70!black}{\ding{51}}} & \scalebox{1.1}{\textcolor{green!70!black}{\ding{51}}} \\
        \bottomrule
    \end{tabularx}
    \vspace{-8pt}
\end{table}

% \begin{wraptable}{r}{0.52\textwidth}
%     \vspace{-10px}
%     \centering
%     \captionsetup{font=footnotesize}
%     \caption{\footnotesize Comparison of our setting with existing efforts}
%     \label{tab:comparison}
%     \def\arraystretch{1.2}
%     \setlength{\tabcolsep}{2pt}
%     \scriptsize
%     \begin{tabularx}{0.5\textwidth}{Xccc}
%         \toprule
%         \multicolumn{1}{c}{\textbf{Existing Efforts}} & \textbf{Unlearning} & \textbf{Adaptation} & \textbf{Forget Data free} \\
%         \midrule
%         SFDA/PDA \cite{liang2020we,hwang2024sf,cao2018partial} & \scalebox{1.1}{\textcolor{red!80!black}{\ding{55}}} & \scalebox{1.1}{\textcolor{green!70!black}{\ding{51}}} & \scalebox{1.1}{N/A} \\
%         MU \cite{golatkar2020eternal,foster2024zero} & \scalebox{1.1}{\textcolor{green!70!black}{\ding{51}}} & \scalebox{1.1}{\textcolor{red!80!black}{\ding{55}}} & \scalebox{1.1}{\textcolor{red!80!black}{\ding{55}}} \\
%         Zero Glance MU \cite{tarun2023fast} & \scalebox{1.1}{\textcolor{green!70!black}{\ding{51}}} & \scalebox{1.1}{\textcolor{red!80!black}{\ding{55}}} & \scalebox{1.1}{\textcolor{green!70!black}{\ding{51}}} \\
%         SCADA-UL & \scalebox{1.1}{\textcolor{green!70!black}{\ding{51}}} & \scalebox{1.1}{\textcolor{green!70!black}{\ding{51}}} & \scalebox{1.1}{\textcolor{green!70!black}{\ding{51}}} \\
%         \bottomrule
%     \end{tabularx}
%     \vspace{-8px}
    
% \end{wraptable}

Existing MU methods typically require access to forget data for unlearning (Table~\ref{tab:comparison}), and therefore are not readily suitable for this task. Few MU methods do not require access to forget data~\cite{chundawat2023zero,tarun2023fast, Ahmed_2025_CVPR}; however, these are not designed for domain-adapted models. Our experiments (see Appendix) reveal that applying such data-free MU methods in our setting results in poor performance, since they were not designed to handle a shift in the data distribution. %This further highlights the need for a targeted approach which we propose herein. 
We propose a new strategy to address SCADA-UL based on adversarial optimization, where an adversarially generated sample from a source-exclusive class to be unlearned (henceforth called a \textit{forget class}) is thoroughly unlearned by the model during domain adaptation through a novel rescaled labeling strategy. Our analysis and empirical results show that this approach achieves effective unlearning in SFDA by progressively erasing representations of the forget classes. 
We also introduce two variants: Continual SCADA-UL (C-SCADA-UL), which addresses scenarios where classes must be forgotten across multiple unlearning requests, and Unknown Class SCADA-UL (UC-SCADA-UL), which deals with cases where source-exclusive classes are not known. %While our method naturally extends to address C-SCADA-UL, in case of UC-SCADA-UL, we incorporate a predictive term to identify classes to be forgotten based on target data.

Our key contributions are summarized as follows: (i) We propose a practically relevant MU setting called SCADA-UL to address unlearning of source-exclusive classes during DA; %, especially . Such a setting has practical applicability, as elaborated above, with the increasing use of adapting models to new domains; 
(ii) We propose a new strategy to address SCADA-UL using adversarial optimization, where an adversarially generated forget class sample is unlearned by the model during DA through a novel rescaled labeling strategy; (iii) Our comprehensive suite of experiments across multiple datasets shows that the proposed method achieves retraining-level performance, outperforming all baselines in the proposed setting; %. \p Poornima: In particular, due to the growing use of transformer-based models in domain adaptation~\cite{Sanyal2023DomainSpecificityIT, sfda_former}, we adapt the ViT-B/16 model in our experiments; 
(iv) We extend our work to two variants: C-SCADA-UL and UC-SCADA-UL, wherein our method shows promising performance; and (v) We also analyze the conceptual intuition behind our method and carry out ablation studies to study the impact of different design choices in our framework.%, which support our approach.

\section{Related Work}
\label{sec:related}

% \vspace{-3pt}
\noindent \textbf{Machine Unlearning.} 
Machine unlearning (MU) is the process of forgetting samples or entire classes of data from a trained model. Existing MU methods for classifiers can be broadly categorized into exact or approximate unlearning~\cite{MU_survey_recent1, li2025machine, musurvey1}. Exact unlearning is achieved by efficiently retraining the model without using the forget data~\cite{yan2022arcane}, while approximate unlearning methods aim to remove the influence of forget data without retraining; influence function-based methods estimate and remove the influence of the data to be removed (forget data) from the model weights~\cite{liu2025efficient, guo,tanno2022repairing,warnecke2023machineunlearningfeatureslabels}, gradient update-based methods perform gradient ascent on forget data~\cite{deltagrad, neel2020descenttodelete}, model optimization-based methods finetune the model using different losses for forget and retain data~\cite{ebrahimpour-boroojeny2025amun, golatkar2020eternal}. Other recent approaches to unlearning include bad-teacher and stochastic-teacher models, post-hoc dampening, and source-free unlearning methods~\cite{Wang2024Machine, Ahmed_2025_CVPR, cha2024learningunlearninstancewiseunlearning, foster2024fast, liu2022continual,chundawat2023can,zhang2023machineunlearningstochasticteacher}. MU has also been studied for Large Language Models~\cite{yao2024machineunlearningpretrainedlarge, geng2025comprehensivesurveymachineunlearning}, federated learning~\cite{zhong2025unlearningknowledgeoverwritingreversible, liu2022right}, generative models~\cite{musurvey3} and graph-based models~\cite{chen2022graph}. Most methods however require access to the forget data, making them inapplicable in our setting where the source-exclusive class data is not available. Although some recent efforts \cite{tarun2023fast, chundawat2023zero, zsmu_new} relax this requirement, they do not address unlearning in the context of domain adaptation.

%\par

\vspace{3pt}
\noindent \textbf{Source-Free Domain Adaptation.} Source-free domain adaptation (SFDA)~\cite{sfda_survey}, aims to adapt a model pre-trained on a source domain to a new target domain with access to only unlabeled target domain data. Finetuning-based SFDA methods such as SHOT~\cite{liang2020we} assign pseudo-labels to images based on the source model's prediction and refine the model in a self-supervised fashion in the target domain.~\cite{trust_guided} utilize target-sample selection to refine pseudo-labels for improved SFDA.~\cite{agarwal2022unsupervised, chen_contrastive, wang2022crossdomaincontrastivelearningunsupervised} perform contrastive learning for SFDA. Few methods~\cite{ubna, eastwood2022sourcefreeadaptationmeasurementshift, ishii2021sourcefreedomainadaptationdistributional} use target data to update statistical information from the source model (for e.g., in batch norm layers), while minimizing the distance between source and target distributions. Clustering-based methods~\cite{lo2023spatiotemporalpixellevelcontrastivelearningbased, lee2022confidencescoresourcefreeunsupervised} perform clustering in the target domain and update the source model through the cluster-assigned pseudo-labels. A recent method, SF(DA)$^2$~\cite{hwang2024sf}, uses spectral neighborhood clustering on the data augmentation graph formed by target data samples in the feature space of the source model, and represents one of the state-of-the-art for SFDA methods. Another recent method, UCon-SFDA~\cite{ucon_sfda}, leverages uncertainty control in SFDA to further the SOTA in the field.
Since source-free domain adaptation methods typically align the source and target domains globally, they transfer knowledge of all source classes to the target domain (as we show later). They do not address unlearning, thus necessitating this effort for simultaneous domain adaptation and unlearning. %This necessitates the need for explicit algorithms for unlearning in our setting. 
%\par

\vspace{3pt}
\textbf{Partial Domain Adaptation.}
Partial Domain Adaptation (PDA) addresses a related setting in domain adaptation where the target domain label space is a subset of the source domain label space. PDA methods use techniques such as suppressing source-exclusive classes through class- or instance-level reweighting~\cite{cao2018partial, gu2021adversarial}, attention-based entropy minimization~\cite{guo2022selective} and alignment of source and target distributions~\cite{sheng2021evolving, yang2022onering, yue2022source}. Their main objective is to alleviate negative transfer of source-exclusive classes while improving target domain performance. These methods require access to source data (we include adaptations of such methods in our empirical studies for completeness). 
% \par
\section{SCADA Unlearning: Problem Setting}
\label{sec:problem}

\begin{figure*}[t]
    \centering
    \includegraphics[width=0.99\linewidth]{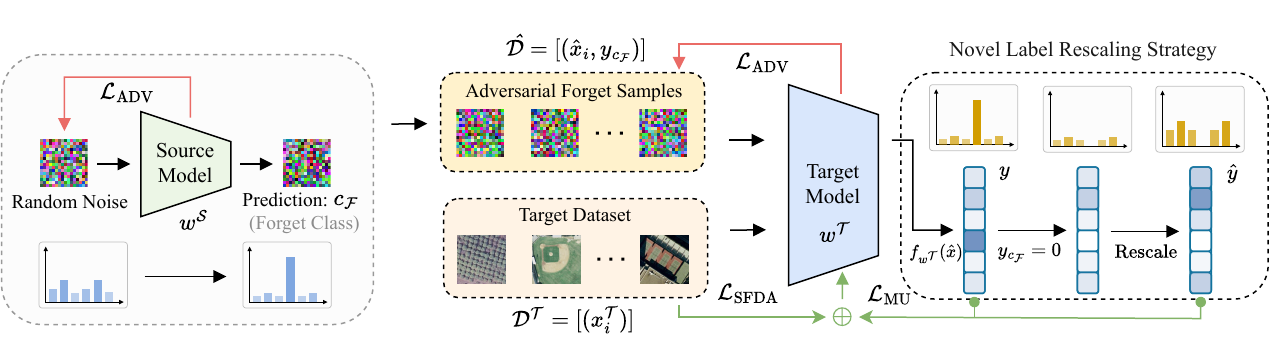}
    \vspace{-16pt}
    \captionsetup{font=footnotesize}
    \caption{\footnotesize \textbf{Our overall method.} We adapt a source trained model $w^{\mathcal{S}}$ to a target domain $\mathcal{D}^{\mathcal{T}}$ while forgetting classes $c_\mathcal{F}$.
    We first create an initial adversarial sample $\hat{x}$ by \textit{maximizing} its probability of belonging to the forget classes $c_\mathcal{F}$ by optimizing $\mathcal{L}_\text{ADV} (w^{\mathcal{S}}, \hat{x})$ (Eqn \ref{eq:adv-loss}). In each subsequent iteration, the model $\smash{w^{\mathcal{T}}}$ \textit{minimizes}: (i) $ \mathcal{L}_\text{MU}$ (\ref{eq:mu-loss}) using rescaled labels $\hat{y}$ to encourage unlearning the forget classes $c_{\mathcal{F}}$; and (ii) $\mathcal{L}_\text{SFDA}$ to simultaneously adapt to the target domain. Before the end of each iteration, the adversarial sample $\hat{x}$ is re-optimized to \textit{maximize} its probability of belonging to the forget classes by maximizing -$\mathcal{L}_\text{ADV}$.
    }
    \label{fig:diagram}
    \vspace{-16pt}
\end{figure*}

\noindent \textbf{Preliminaries and Notations.}
Let $x_i$ denote a training sample, and $y_i$ be its corresponding label. In the domain adaptation context, let $\mathcal{D^S} = \{(x^\mathcal{S}_i, y^\mathcal{S}_i)\}_{i=1}^{n_s}$ denote the source domain dataset, and $\smash{\mathcal{D}^\mathcal{T} = \{(x^\mathcal{T}_i)\}_{i=1}^{n_t}}$ denote the target domain dataset. The set of source domain classes is denoted by $\mathcal{C_S}$, and the set of target classes, assumed to be a subset of source classes in our setting, is denoted by $\mathcal{C_T}$. The set of source-exclusive classes (or `forget') classes is denoted by $\mathcal{C_F} = \mathcal{C_S}\setminus\mathcal{C_T}$. A single forget class is denoted by $c_\mathcal{F} \in \mathcal{C_F}$, and its complement, is denoted by by $c_\mathcal{R} = \mathcal{C_S} \setminus c_\mathcal{F}$.
The subset of source data where labels belong to $\mathcal{C_F}$ is the source forget set $\mathcal{D}_{\smash{f}}^\mathcal{S}$. Its complement, the source retain set is denoted by $\mathcal{D}_r^\mathcal{S} = \mathcal{D^S}\setminus\mathcal{D}_{\smash{f}}^\mathcal{S}$. To evaluate the performance of our method, we also introduce a target forget set $\mathcal{D}_{\smash{f}}^\mathcal{T}$ on which our domain-adapted model is expected to perform poorly. The corresponding target retain set is given by $\mathcal{D}_r^\mathcal{T}$. We denote the source classifier training algorithm by $\mathcal{A}(\cdot)$, domain adaptation training algorithm by $\mathcal{B}(\cdot)$, and the unlearning algorithm by $\mathcal{U}(\cdot)$. 

\vspace{3pt}
\noindent \textbf{Proposed Problem Setting.}
As stated earlier, we introduce an unlearning setting within the source-free domain adaptation context called \textbf{U}n\textbf{l}earning \textbf{S}ource-exclusive \textbf{C}l\textbf{A}sses in \textbf{D}omain \textbf{A}daptation, abbreviated to \textit{SCADA-UL}. At a high level, this problem setting, similar to unlearning, seeks to produce a post-unlearning model that behaves similar to one never exposed to forget data. However, this setting presents some unique challenges: adaptation to a target domain, the absence of source data $\mathcal{D^S}$, and the absence of target forget data $\mathcal{D}_{\smash{f}}^\mathcal{T}$. While unavailability of source data is also an issue in SFDA, the non-availability of forget data in the target domain that is motivated by practical use case settings makes this problem challenging and non-trivial.
Based on the definition of MU in~\cite{musurvey1}, we formally define SCADA-UL as: 
\begin{definition}
\vspace{-14pt}
\label{def:SCADA} \textbf{(SCADA-UL).} 
Given a source model $ w^\mathcal{S} = \mathcal{A}(\mathcal{D}^\mathcal{S}) $, SCADA-UL is a process $ \mathcal{U}: \{w^\mathcal{S}, \mathcal{D}_r^\mathcal{T}, \mathcal{C_F}\} \rightarrow w_u^\mathcal{T} $ that produces an unlearned, adapted model $ w_u^\mathcal{T} $. 
\end{definition} 
In other words, $ \mathcal{U}(w^\mathcal{S}, \mathcal{D}_r^\mathcal{T}, \mathcal{C_F}) $ maps the trained source model, target dataset, and forget classes to a model $ w_u^\mathcal{T} $ that behaves as though it were adapted from a source model not trained on $ \mathcal{D}_{\smash{f}}^\mathcal{S}$.
A trivial solution here is to retrain the source model from scratch on $\mathcal{D}_r^\mathcal{S}$ and adapt it to the target domain using $\mathcal{D}_r^\mathcal{T}$, i.e. $\mathcal{B}(w_r^\mathcal{S}, \mathcal{D}_r^\mathcal{T})$ where $w_r^\mathcal{S} = \mathcal{A}(\mathcal{D}_r^\mathcal{S})$. However, this is not possible due to unavailability of $\mathcal{D^S}$ (we, however, use this as an upper bound in our experiments). On the other hand, naively domain adapting $\mathcal{B}(w^\mathcal{S}, \mathcal{D}_r^\mathcal{T})$ would result in the model still containing information about the forget classes. This is used as an additional baseline.

We additionally define two variants of SCADA-UL under other constraints. In UC-SCADA-UL, we extend Defn~\ref{def:SCADA} to scenarios where the forget classes $\mathcal{C_F}$ are unknown, but their cardinality $\smash{\left|\mathcal{C_F}\right|}$ is known. Similarly, Continual SCADA-UL addresses Defn~\ref{def:SCADA} in continual settings where disjoint forget class sets arrive in multiple steps \textit{i}, denoted by $\smash{\mathcal{C_F}^i}$. Formal definitions are provided in the appendix.
% \vspace{-14pt}
\section{Proposed Methodology}
\label{sec:method}

\vspace{-4pt}
SCADA-UL requires approximating the behavior of a source model trained without forget data and adapted to the target domain. We model this task as obtaining a maximum a posteriori (MAP) estimate $p(w \mid w_r^\mathcal{S}, \mathcal{D}_r^\mathcal{T})$, where $w_r^\mathcal{S}$ denotes the source model trained only on $\mathcal{D}_r^\mathcal{S}$. This posterior cannot be directly evaluated as it requires retraining the source model on the source retain dataset $\mathcal{D}_r^\mathcal{S}$, which is unavailable in our setting. 
%\f 
To derive tractable objectives, we expand $p(w \mid \mathcal{D}_r^\mathcal{S}, \smash{\mathcal{D}_f}^\mathcal{S}, \mathcal{D}_r^\mathcal{T})$ using Bayes' rule: 
$p(\smash{\mathcal{D}_f}^\mathcal{S} \mid w, \mathcal{D}_r^\mathcal{S}, \mathcal{D}_r^\mathcal{T}) \cdot p(w \mid\mathcal{D}_r^\mathcal{S}, \mathcal{D}_r^\mathcal{T})$. This can be written as:
% $ p(w \mid\mathcal{D}_r^\mathcal{S}, \mathcal{D}_r^\mathcal{T}) \propto {p(w \mid \mathcal{D^S}, \mathcal{D}_r^\mathcal{T})} \, / \, {p(\smash{\mathcal{D}_f}^\mathcal{S} \mid w, \mathcal{D}_r^\mathcal{S}, \mathcal{D}_r^\mathcal{T})}$. 
% Next, we assume that datasets of disjoint label spaces are conditionally independent of the model weights i.e., $\smash{\mathcal{D}_f}^S \perp\!\!\!\perp \left( \mathcal{D}_r^S, \mathcal{D}_r^T \right) \mid w$. This yields:
$ p(w \mid\mathcal{D}_r^\mathcal{S}, \mathcal{D}_r^\mathcal{T}) \propto {p(w \mid \mathcal{D^S}, \mathcal{D}_r^\mathcal{T})} \, / \, {p(\smash{\mathcal{D}_f}^\mathcal{S} \mid w)}$ by rearranging and conditional independence (cf. Eq. 6 in~\cite{panda2025partially}). 
% Finally, we approximate the source domain datasets $\mathcal{D}_r^\mathcal{S}$ and $\mathcal{D}^\mathcal{S}$ with their corresponding trained models $w_r^\mathcal{S}$ and $w^\mathcal{S}$, giving:
Using model approximations gives us: $ p(w \mid w_r^\mathcal{S}, \mathcal{D}_r^\mathcal{T}) \propto {p(w \mid w^\mathcal{S}, \mathcal{D}_r^\mathcal{T})} \, /  \,{p(\smash{\mathcal{D}_f}^\mathcal{S} \mid w)}$ (cf. Eq. 3 in~\cite{ritter2018online}). We revisit our assumptions herein at the end of the section.
% Next, we make two key assumptions: (1) Datasets of disjoint label spaces are conditionally independent of model weights i.e., $\smash{\mathcal{D}_f}^S \perp\!\!\!\perp \left( \mathcal{D}_r^S, \mathcal{D}_r^T \right) \mid w$, and (2) The conditional probability on source domain datasets are approximated by their trained models i.e. $\mathcal{D}_r^\mathcal{S}$ and $\mathcal{D^S}$ are approximated with their trained models $w_r^\mathcal{S}$ and $w^\mathcal{S}$ respectively. We make these assumptions for tractability of our approach; while they may not hold strictly always, our empirical analysis suggests that our method inspired by these assumptions works in practice.  Following these assumptions, the term reduces to: $ p(w \mid w_r^\mathcal{S}, \mathcal{D}_r^\mathcal{T}) \propto {p(w \mid w^\mathcal{S}, \mathcal{D}_r^\mathcal{T})} \, /  \,{p(\smash{\mathcal{D}_f}^\mathcal{S} \mid w)}$
Replacing the terms and maximizing the logarithm on both sides leads to our unlearning and adaptation objectives below.
\begin{equation}
\begin{aligned}
&\max_w \log p(w \mid w_r^\mathcal{S},\mathcal{D}_r^\mathcal{T})
\myequiv \\
&\quad\max_w \bigl[
  \underbrace{\log p(w \mid w^\mathcal{S},\mathcal{D}_r^\mathcal{T})}_{\text{Adaptation}} -
  \underbrace{\log p(\mathcal{D}_{\smash{f}}^\mathcal{S} \mid w)}_{\text{Unlearning}}
\bigr]
\end{aligned}
\label{eq:main-4}
\end{equation}

The left-hand side captures the ideal SCADA-UL objective in Expr \ref{eq:main-4}; however, as discussed, this is infeasible to compute. The right-hand side represents an equivalent objective with similar source training and adaptation terms. As $w^\mathcal{S}$ is available in our setting, $\log p(w \mid w^\mathcal{S},\mathcal{D}_r^\mathcal{T})$ can be maximized by adapting $w^\mathcal{S}$ to the target domain. The other term on the RHS $-\log p(\mathcal{D}_{\smash{f}}^\mathcal{S} \mid w)$ corresponds to unlearning. Implementing this term is non-trivial as both $\mathcal{D}_{\smash{f}}^\mathcal{S}$ and $\mathcal{D}_{\smash{f}}^\mathcal{T}$ are unavailable, and this motivates our method below.

\subsection{Adversarial Optimization for SCADA-UL}
\label{subsec:ao-scada-ul}
Expr \ref{eq:main-4} showed that achieving SCADA-UL involves balancing both adaptation and unlearning objectives. 
One could achieve this objective by applying an existing data-free MU method on an adapted model or a source model before adaptation. However, existing data-free MU methods are designed for unlearning in a stationary single-domain setting and our empirical results reveal that they perform poorly when used both before and after adaptation (see Appendix). 
Thus, we propose carrying out both unlearning and adaptation processes simultaneously to achieve the proposed objective. 
The unlearning and adaptation processes, though, are misaligned goals - the unlearning process aims to avoid transfer of source-exclusive classes to the target domain, while the adaptation process tries to transfer source domain classes to the target domain. Hence, to achieve these simultaneously, the method must design objectives that are minimally conflicting (we elaborate this in our studies in Sec. ~\ref{sec:analysis}). In particular, we use a novel rescaled labeling strategy (Eq.~\ref{eq: bayesreescale}) in the unlearning process that helps reduce this conflict.

A key question in our SCADA-UL setting remains: \textit{How to identify representative samples for a forget class $c_\mathcal{F}$ without access to its data}? We address this using adversarial sample generation. (We call a sample \textit{adversarial}, since the model is expected to minimize its probability while maximizing its belongingness to the forget class it represents.) A second question that arises is: \textit{Once representative samples for forget classes are identified, how can the model unlearn these samples while adapting simultaneously in a way that involves minimal conflict?} The opposing objectives of forget-class sample generation and unlearning lead to a minmax optimization formulation given by: 
\vspace{-3pt}
\begin{equation}
\begin{split}
\label{eq:minmax-1}    
    \min_{w^\mathcal{T}} \max_{\hat{x}} \bigl[ \log p(\smash{\hat{\mathcal{D}}} \mid w^\mathcal{T}) - \log p(w^\mathcal{T} \mid w^\mathcal{S},\mathcal{D}_r^\mathcal{T}) \bigr]
    \vspace{-3pt}
\end{split}
\end{equation}
where $\smash{\hat{\mathcal{D}}}=\{(\hat{x}_i, y_{c_\mathcal{F}})\}_{i=1}^{N_\text{adv}}$ and ${N_\text{adv}}$ is the number of representative samples generated. 
This objective can be broken down into sample generation, unlearning, and adaptation steps which we detail below. For simplicity of explanation, without loss of generality, we focus our discussion on unlearning a single class $c_\mathcal{F}$, although this could be used to subsume a set of classes (or extended to multiple classes one after the other as described at the end of this section). 

\vspace{2pt}
\textbf{Generating Adversarial Samples.}
Generating representative samples in data-free settings has been studied in earlier work in other contexts~\cite{chundawat2023zero}. In our case, we aim to generate inputs that elicit confident predictions from the model, thereby producing representative samples for the forget class. This is achieved by using the cross-entropy loss as our adversarial sample generation loss, $\mathcal{L}_{\text{ADV}}$, i.e.:
\vspace{-3pt}
\begin{equation}\label{eq:adv-loss}
   \mathcal{L}_{\text{ADV}}(w^{\mathcal{T}}, \hat{x}) = \mathcal{L}_{\text{CE}} (w^{\mathcal{T}}, \hat{x}, {c_\mathcal{F}}) 
   \vspace{-2pt}
\end{equation}
giving \( \min_{\hat{x}} \mathcal{L}_{\text{ADV}}(w^{\mathcal{T}}, \hat{x}) \myequiv \max_{\hat{x}}  p({\smash{\hat{\mathcal{D}}} \mid w^\mathcal{T}}) \). Our unlearning step (see Algorithm \ref{alg:minimax}) starts by generating the adversarial samples $\hat{x}$ with a minimization of $\mathcal{L}_{\text{ADV}}$ with $w^\mathcal{T}$ initialized to the source model $w^{\mathcal{S}}$.
We also minimize $\mathcal{L}_{\text{ADV}}$ \textit{throughout the domain adaptation iterations}, allowing the generated samples to evolve alongside the adapting model.% $w^{\mathcal{T}}$. 

\vspace{2pt}
\textbf{Adapting and Unlearning with Minimal Conflict via Rescaled Labeling.}
Consider an adversarial sample $ \hat{x} $ generated based on Expr.~\ref{eq:adv-loss}, i.e. $ \hat{x} $ is representative of class $ c_{\mathcal{F}} $. Let the output softmax distribution over all classes, $f_{w^{\mathcal{T}}}(\hat{x})$, for the sample be $y \in \mathbb{R}^d$ where $d$ is the number of classes (retain and forget). Naturally, it would have the highest probability value for $c_\mathcal{F}$, as the sample is representative of the forget class.
To promote unlearning, we propose exposing the model to false information  \( \mathcal{I} \) that the forget sample belongs to a different class \(c_\mathcal{R}\) other than the forget class $c_\mathcal{F}$, i.e., there exists a class \( c \in c_\mathcal{R} \), \(c \neq c_\mathcal{F}\) with \( y_c = 1 \). From Bayes' Theorem, $p_{w^{\mathcal{T}}}(y = c \mid \mathcal{I} ) \propto p(\mathcal{I} \mid y = c) \cdot p(y = c)$.
Given $p(\mathcal{I} \mid y_{c_{\mathcal{F}}} = 1) = 0$ (since it contradicts the prior), and $p(\mathcal{I} \mid y_{c_{{R}}}) = 1$, we obtain the final renormalized distribution as:
\vspace{-3pt}
\begin{equation}
\label{eq: bayesreescale}
    \hat{y} = p(y = c \mid \mathcal{I}) = \begin{cases} 
            0 & \text{if } i = c_\mathcal{F} \\
            \frac{y_i}{\sum_{j \in c_\mathcal{R}} y_j} & \text{if $i \in c_\mathcal{R}$}
            \end{cases}
            \vspace{-3pt}
\end{equation}
\newcommand{\WrapState}[1]{%
  \State #1
}
\begin{algorithm}[t]
\footnotesize
\captionsetup{font=footnotesize}
\caption{\footnotesize \textbf{: Adversarial Optimization for SCADA-UL}}
\label{alg:minimax}
\begin{algorithmic}[0]
\State \textbf{Inputs:} Source model $w^\mathcal{S}$, target data $\mathcal{D}^\mathcal{T}$, forget classes $\mathcal{C_F}$, SFDA loss $\mathcal{L}_\text{SFDA}$, trade-off $\alpha$, learning rates $\eta_1, \eta_2$
\State \textbf{Init:} $w^\mathcal{T} \gets w^\mathcal{S}$
\For{each $c_\mathcal{F} \in \mathcal{C_F}$}
  \State $\hat{x}_{c_\mathcal{F}} \gets \arg\min_{\hat{x}} \mathcal{L}_\text{CE}(w^\mathcal{T}, \hat{x}, c_\mathcal{F})$ \Comment{Eq. (\ref{eq:adv-loss})}
\EndFor
\For{$\text{epoch} = 1$ to $M$}
  \For{each $c_\mathcal{F} \in \mathcal{C_F}$}
    \For{$\text{step} = 1$ to $N / |\mathcal{C_F}|$}
      \State $x^\mathcal{T} \sim \mathcal{D}^\mathcal{T}$; \quad $y \gets f_{w^\mathcal{T}}(\hat{x}_{c_\mathcal{F}})$
      \State $\hat{y}_i \gets \begin{cases}
        0 & \text{if } i = c_\mathcal{F} \\
        \frac{y_i}{\sum_{j \neq c_\mathcal{F}} y_j} & \text{otherwise}
      \end{cases}$
      \WrapState{%
        \(
        \begin{aligned}
        \varphi \;&= \mathcal{L}_\text{SFDA}(w^\mathcal{T}, x^\mathcal{T}) \\
                   &\quad + \alpha \cdot \mathcal{L}_\text{CE}(w^\mathcal{T}, \hat{x}_{c_\mathcal{F}}, \hat{y})
        \end{aligned}
        \)
      }
      \WrapState{%
        $
        w^\mathcal{T} \gets w^\mathcal{T} - \eta_1 \nabla_{w^\mathcal{T}} \varphi
        $
        \Comment{Eqs. (\ref{eq:finalsecondobjective}, \ref{eq: alternating})}
      }
      \WrapState{%
        \(
        \begin{aligned}
        \hat{x}_{c_\mathcal{F}} \;&\gets   \hat{x}_{c_\mathcal{F}} - \eta_2 \nabla_{\hat{x}_{c_\mathcal{F}}}
        \mathcal{L}_\text{CE}(w^\mathcal{T}, \hat{x}_{c_\mathcal{F}}, c_\mathcal{F}) 
        \end{aligned} \text{\hspace{-7px}\Comment{Eqs. (\ref{eq:adv-loss}, \ref{eq: alternating})}}
        \)
      }
    \EndFor
  \EndFor
\EndFor
\State \Return $w_u^\mathcal{T} \gets w^\mathcal{T}$
\end{algorithmic}
\end{algorithm}
\Cref{eq: bayesreescale} represents the proposed rescaled labeling we use for unlearning. This labeling strategy intuitively represents the ideal alternative label for the adversarial sample 
if it did not belong to the forget class. It redistributes the probabilities in proportion with the model's output for the input sample, hence causing minimal conflict with adaptation. Further, our analysis shows that this labeling strategy: (1) outperforms other strategies such as uniform or random labeling which either cause catastrophic forgetting or subpar unlearning and adaptation performance (Table~\ref{tab:ab-label}); (2) aligns the post-unlearning model's output distributions more closely to that of a retrained model (see Appendix); and (3) results in larger gradient updates to the weights associated with retain classes compared to forget classes, thus achieving the unlearning objective (see Theorem~\ref{thm:gradients} in \cref{sec: theory}). Using these labels, our MU loss term $\mathcal{L}_\text{MU}$ applies the cross entropy loss: $  \mathcal{L}_{\text{MU}} (w^{\mathcal{T}}, \hat{x}, \hat{y}) \ = \ \mathcal{L}_{\text{CE}}(w^{\mathcal{T}}, \hat{x}, \hat{y})$, giving us:
\vspace{-3pt}
\begin{equation}
\begin{split}
    \min_{w^\mathcal{T}} \mathcal{L}_{\text{MU}}(w^{\mathcal{T}}, \hat{x}, \hat{y}) \myequiv \min_{w^\mathcal{T}}  p({\smash{\hat{\mathcal{D}}} \mid w^\mathcal{T}})
    \vspace{-3pt}
\end{split}
\label{eq:mu-loss}
\end{equation}
For the domain adaptation process, we leverage foundational work in source-free domain adaptation (SFDA)~\cite{liang2020we, hwang2024sf} to adapt the model to the target domain $ \mathcal{D}_r^{\mathcal{T}} $ using a loss term that enforces latent alignment between source and target data distributions. We denote this loss as $\mathcal{L}_{\text{SFDA}}(w^{\mathcal{T}}, x^{\mathcal{T}})$:
\begin{equation} 
\begin{split}
    \min_{w^\mathcal{T}} \mathcal{L}_\text{SFDA}(w^\mathcal{T}, x^\mathcal{T}) \myequiv &\max_{w^\mathcal{T}} \bigl[ \log p(w^\mathcal{T} \mid w^\mathcal{S},\mathcal{D}_r^\mathcal{T}) \bigr] \label{eq:sfda-loss}
\end{split}
\end{equation}
Combining $\mathcal{L}_{\text{MU}}$ and $\mathcal{L}_{\text{SFDA}}$, our unlearn-and-adapt objective $\varphi$ is given by:
\begin{equation}
\begin{split}
\label{eq:finalsecondobjective}
     &\min_{w^\mathcal{T}} \bigl[ \log p({\smash{\hat{\mathcal{D}}} \mid w^\mathcal{T}}) - \log p(w^\mathcal{T} \mid w^\mathcal{S},\mathcal{D}_r^\mathcal{T}) \bigr] \myequiv \\
    &\quad\min_{w^{\mathcal{T}}} \underbrace{\mathcal{L}_{\text{SFDA}} (w^{\mathcal{T}}, x^{\mathcal{T}}) + \alpha \mathcal{L}_{\text{MU}} (w^{\mathcal{T}}, \hat{x}, \hat{y})}_{\varphi(w^{\mathcal{T}}, \hat{x}, x^{\mathcal{T}}, \hat{y})}
\end{split}
\end{equation}
where $\alpha \ge 0 $ is a hyperparameter.

\begin{table*}
  \captionsetup{font=footnotesize}
  \caption{\footnotesize \textbf{Results for Multi-Class SCADA-UL on OfficeHome, Office31, DomainNet datasets.} We compare against 10 baselines: Original = model adapted to target domain without applying unlearning; Retrain = model adapted with source model retrained without forget data; Finetune = model finetuned on subset of target retain data; MU methods~\cite{tarun2023fast,chundawat2023zero,foster2024zero,trippa2024tau,Ahmed_2025_CVPR}; PDA/SFPDA methods~\cite{cao2018partial,liang2020we}. $A_{\mathcal{D}_{\smash{f}}^\mathcal{T}}$ = accuracy on target forget set; $A_{\mathcal{D}_r^\mathcal{T}}$ = target retain accuracy; Score captures overall unlearning performance i.e. high $A_{\mathcal{D}_r^\mathcal{T}}$ + low $A_{\mathcal{D}_{\smash{f}}^\mathcal{T}}$. \textit{(Best result in bold, second-best underlined)}}
  \vspace{-6pt}
  \label{tab:mc-SCADA}
  \centering
  \resizebox{0.9\textwidth}{!}{%
    \begin{tabular}{lccccccccc}
      \toprule
      & \multicolumn{3}{c}{OfficeHome} & \multicolumn{3}{c}{Office31} & \multicolumn{3}{c}{DomainNet-126} \\
      \cmidrule(r){2-4} \cmidrule(r){5-7} \cmidrule(r){8-10}
      Method & $A_{\mathcal{D}_r^\mathcal{T}} \uparrow$ & $A_{\mathcal{D}_f^\mathcal{T}} \downarrow$ & Score $\uparrow$ 
             & $A_{\mathcal{D}_r^\mathcal{T}} \uparrow$ & $A_{\mathcal{D}_f^\mathcal{T}} \downarrow$ & Score $\uparrow$
             & $A_{\mathcal{D}_r^\mathcal{T}} \uparrow$ & $A_{\mathcal{D}_f^\mathcal{T}} \downarrow$ & Score $\uparrow$ \\
      \midrule
      Original (SF(DA)$^2$~\cite{hwang2024sf})
      & \val{75.8}{1.0} & \val{58.1}{2.2} & \val{0.48}{0.0}
      & \val{76.8}{1.2} & \val{90.1}{1.6} & \val{0.40}{0.0}
      & \val{67.7}{0.6} & \val{38.7}{2.9} & \val{0.50}{0.0} \\
      Retrain
      & \val{76.3}{1.2} & \val{0.0}{0.0} & \val{0.76}{0.0}
      & \val{77.4}{1.4} & \val{0.0}{0.0} & \val{0.77}{0.0}
      & \val{66.3}{1.3} & \val{0.0}{0.0} & \val{0.66}{0.0} \\ 
      \midrule
      Finetune  
      & \val{\textbf{76.1}}{0.9} & \val{49.2}{2.2} & \val{0.51}{0.0}
      & \val{\underline{76.7}}{1.4} & \val{82.8}{2.0} & \val{0.42}{0.0}
      & \val{\underline{66.5}}{0.9} & \val{20.4}{4.1} & \val{\underline{0.56}}{0.0} \\
      UNSIR~\cite{tarun2023fast} 
      & \val{35.0}{5.6} & \val{\textbf{0.0}}{0.0} & \val{0.35}{0.1}
      & \val{59.7}{5.0} & \val{\underline{16.3}}{3.3} & \val{0.53}{0.1}
      & \val{14.6}{3.6} & \val{\underline{0.4}}{0.6} & \val{0.14}{0.0} \\
      ZSMU~\cite{chundawat2023zero}
      & \val{71.1}{6.1} & \val{39.8}{10.} & \val{0.50}{0.0}
      & \val{77.0}{1.1} & \val{84.2}{3.3} & \val{0.42}{0.0}
      & \val{65.0}{1.6} & \val{34.5}{6.6} & \val{0.49}{0.0} \\
      Lipschitz~\cite{foster2024zero}
      & \val{58.6}{11.} & \val{25.4}{16.} & \val{0.46}{0.1}
      & \val{65.3}{9.5} & \val{28.4}{16.} & \val{0.52}{0.1}
      & \val{39.0}{11.} & \val{8.5}{8.6} & \val{0.35}{0.1} \\
      Nabla Tau~\cite{trippa2024tau}
      & \val{63.2}{2.9} & \val{\underline{1.5}}{2.0} & \val{\underline{0.62}}{0.0}
      & \val{73.4}{1.7} & \val{28.0}{7.0} & \val{\underline{0.59}}{0.0}
      & \val{44.4}{5.4} & \val{1.3}{1.9} & \val{0.44}{0.1} \\
Unlearned(+)~\cite{Ahmed_2025_CVPR} 
      & \val{\textbf{76.1}}{9.4} & \val{38.6}{7.3} & \val{0.54}{0.0}
      & \val{\textbf{77.9}}{0.5} & \val{85.3}{2.5} & \val{0.42}{0.0}
      & \val{\textbf{66.9}}{0.8} & \val{26.6}{4.8} & \val{0.54}{0.0} \\
      \midrule
      PADA~\cite{cao2018partial}
      & \val{74.0}{1.4} & \val{62.6}{1.3} & \val{0.45}{0.0}
      & \val{76.2}{1.6} & \val{89.1}{1.6} & \val{0.40}{0.0}
      & \val{61.9}{0.4} & \val{46.6}{1.1} & \val{0.43}{0.0} \\
      SHOT~\cite{liang2020we}
      & \val{73.7}{1.3} & \val{24.7}{1.2} & \val{0.59}{0.0}
      & \val{76.0}{1.2} & \val{68.0}{1.7} & \val{0.46}{0.0}
      & \val{67.8}{0.9} & \val{39.6}{2.7} & \val{0.50}{0.0} \\
      \midrule
      \rowcolor{green!40!white}
      Ours
      & \val{\underline{75.1}}{1.3} & \val{\textbf{0.0}}{0.0} & \val{\textbf{0.75}}{0.0}
      & \val{76.6}{1.2} & \val{\textbf{0.0}}{0.0} & \val{\textbf{0.77}}{0.0}
      & \val{65.6}{0.9} & \val{\textbf{0.0}}{0.0} & \val{\textbf{0.66}}{0.0} \\
      \bottomrule
    \end{tabular}
  }
  \vspace{-12pt}
\end{table*}

\vspace{2pt}
\textbf{Overall Optimization Objective.}
Combining the optimization objectives in Expr.~\ref{eq:adv-loss} and~\ref{eq:finalsecondobjective} gives our final joint optimization objective:
\(
\min_{w^{\mathcal{T}}} \min_{\hat{x}} \varphi(w^{\mathcal{T}}, \hat{x}, x^{\mathcal{T}}, \hat{y}) + \mathcal{L}_{\text{ADV}}(w^{\mathcal{T}}, \hat{x}) 
\).
We solve this joint optimization problem using an iterative approach. 
To ensure stable training and separate the effects of $ \mathcal{L}_{\text{ADV}} $ and $ \varphi $, we alternate between solving the following two problems:
\begin{equation}
 \label{eq: alternating}
 \min_{w^{\mathcal{T}}} \varphi(w^{\mathcal{T}}, \hat{x}, x^{\mathcal{T}}, \hat{y}) \quad \& \quad \min_{\hat{x}} \mathcal{L}_{\text{ADV}}(w^{\mathcal{T}}, \hat{x})
\end{equation}
A single SGD step is performed on each objective iteratively until convergence. Algorithm \ref{alg:minimax} and Figure \ref{fig:diagram} summarize our overall learning procedure.

\vspace{2pt}
\textbf{Connection to Objective (\ref{eq:main-4}).}
When the adversarial loss $\mathcal{L}_\text{ADV}$ (\ref{eq:adv-loss}) is sufficiently minimized on $\hat{x}$ ($\hat{x} \in \smash{\hat{\mathcal{D}}}$) using the source model $w^\mathcal{S}$, the adversarial samples effectively approximate the source forget dataset: 
\(
    p(\smash{\hat{\mathcal{D}}} \mid w^\mathcal{T}) \approx p(\mathcal{D}_{\smash{f}}^\mathcal{S} \mid w^\mathcal{T})
\). From Expr. \ref{eq:minmax-1} and the above approximation, our method is expressed as: 
\begin{equation}
\begin{split}
    &\min_{w^\mathcal{T}} \max_{\hat{x}} \bigl[ \log p(\mathcal{D}_{\smash{f}}^\mathcal{S} \mid w^\mathcal{T}) - \log p(w^\mathcal{T} \mid w^\mathcal{S},\mathcal{D}_r^\mathcal{T}) \bigr] \myequiv \\ 
    &\quad\max_{w^\mathcal{T}} \bigl[\log p(w^\mathcal{T} \mid w^\mathcal{S},\mathcal{D}_r^\mathcal{T}) - \log p(\mathcal{D}_{\smash{f}}^\mathcal{S} \mid w^\mathcal{T})\bigr] \label{eq:main-5} 
\end{split}
\end{equation}
The final step is obtained by assuming no further optimization on \(\hat{x}\) after initialization. Under these conditions, our method (as in Expr. \ref{eq:main-5}) becomes identical to Expr.~\ref{eq:main-4}. It may be noted that our final implementation remains to be Expr.~\ref{eq: alternating}, and Exprs.~\ref{eq:main-4},~\ref{eq:main-5} convey its intuition.

\vspace{2pt}
\textbf{Unlearning Multiple Classes.} We extend our approach to unlearn multiple classes by dividing each training epoch into $\left|\mathcal{C_F}\right|$ steps and unlearning each class $c_\mathcal{F}$ during its respective step. This process is detailed in \cref{alg:minimax}.

\vspace{3pt}
Our methodology assumes that datasets with disjoint label spaces are assumed to be conditionally independent given the model weights, similar to prior work (cf. Eq. 6 in~\cite{panda2025partially}). Studying this further, we compare the terms with and without this assumption, namely, $\log p(\mathcal{D}_f^\mathcal{S} \mid w)$ and $\log p(\mathcal{D}_f^\mathcal{S} \mid w, \mathcal{D}_r^\mathcal{S}, \mathcal{D}_r^\mathcal{T})$. The first term minimizes the posterior on the forget data, while the second conditions on retain data while minimizing the posterior on the forget data. The second term corresponds to a stronger form of unlearning: forgetting a class requires removing its influence from all remaining classes. For illustration, consider a land-use classification scenario. If the model was trained to classify “urban” and “industrial” scenes, and “industrial” must be forgotten for security reasons, the second setting would require removing all related information from “urban” as well. While more stringent, such a requirement is not typically motivated in privacy-driven unlearning, where the goal is to \textit{selectively forget only the specified class}. Thus, we adopt the first, more relaxed expression, which focuses solely on removing information related to the forget classes without conditioning on the retain data. Additionally, following earlier work in online and continual learning scenarios~\cite{ritter2018online,aich2021elastic}, we use $p(w\mid \mathcal{D}_r^\mathcal{S}, \mathcal{D}_r^\mathcal{T})\approx p(w\mid w_r^\mathcal{S}, \mathcal{D}_r^\mathcal{T})$, which comes from considering the posterior $p(w\mid \mathcal{D^S})$ as approximately Gaussian and centered at the MLE estimate $w^\mathcal{S}$ and Bayesian inference. More analysis of these conditions is included in the Appendix.

%Secondly, we assume that source datasets can be approximated by their corresponding trained models. This follows by assuming the posterior $p(w\mid \mathcal{D^S})$ is approximately Gaussian and centered at the MLE estimate $w^\mathcal{S}$. i.e., $p(w\mid \mathcal{D^S})\approx\mathcal{N}(w^\mathcal{S}, \Sigma)$. By Bayesian inference, $p(w\mid \mathcal{D^S}, \mathcal{D}_r^\mathcal{T})\propto p(\mathcal{D}_r^\mathcal{T}\mid w)\cdot p(w\mid \mathcal{D^S})$. Substituting the Gaussian approximation yields $p(w\mid \mathcal{D^S}, \mathcal{D}_r^\mathcal{T})\propto p(\mathcal{D}_r^\mathcal{T}\mid w)\cdot \mathcal{N}(w^\mathcal{S}, \Sigma)\approx p(w\mid w^\mathcal{S}, \mathcal{D}_r^\mathcal{T})$. Similarly, we obtain $p(w\mid \mathcal{D}_r^\mathcal{S}, \mathcal{D}_r^\mathcal{T})\approx p(w\mid w_r^\mathcal{S}, \mathcal{D}_r^\mathcal{T})$. Such approximations have been explored in online learning scenarios~\cite{ritter2018online}, and Elastic Weight Consolidation~\cite{aich2021elastic} also employs a (diagonal) Laplace approximation to estimate posteriors across sequential tasks.

\vspace{-2pt}
\subsection{Method Interpretation} \label{sec: theory}
\vspace{-4pt}
%\noindent \textbf{Connection of our method to gradients.} %\label{sec: theory}
To conceptually understand why our method works, we study the rescaled labels and adversarial optimization strategy w.r.t. the gradients of the MU loss \( \mathcal{L}_{\text{MU}}(w^{\mathcal{T}}, \hat{x}, \hat{y}) \) in Expr.~\ref{eq:mu-loss}. The following analysis studies how the proposed strategy differentiates between the \textit{retain} and \textit{forget} classes in the given setting, during the training (domain adaptation) process.

\vspace{-4pt}
\begin{restatable}{theorem}{gradients}
\label{thm:gradients}

Let \( \mathbf{\tau} \) denote the set of final layer weights of a neural network, and let \( \tau_c \) represent the weights corresponding to the output neuron of class \( c \). Consider two disjoint sets of classes: a forget set \( c_\mathcal{F} \) and a retain set \( c_\mathcal{R} \). If the proposed unlearning process (\cref{alg:minimax}) is applied with forget set \( c_\mathcal{F} \), then the gradient magnitude of the MU loss with respect to \( \tau_{c_\mathcal{R}} \) satisfies the following inequality:
\vspace{-4pt}
\[
\left\| \frac{\partial \mathcal{L}_{\text{MU}}}{\partial \tau_{c_\mathcal{R}}} \right\| \geq \left( \frac{1}{\delta} - 1 \right) \left\| \frac{\partial \mathcal{L}_{\text{MU}}}{\partial \tau_{c_\mathcal{F}}} \right\|
\]
for some constant \( \delta \in (0,1) \).
\end{restatable}
\vspace{-5pt}
Proof is included in the Appendix~\cref{sec:proof41}.

\vspace{-2pt}
\begin{observation} \label{obs:gradflow}
The above result shows that our method (adversarial sample generation with label rescaling) provably induces more gradient flow on the weights $\tau_{c_\mathcal{R}}$ as compared to $\tau_{c_{\mathcal{F}}}$. Hence, in the domain adaptation process, our model focuses on the weights $\tau_{c_\mathcal{R}}$, rather than $\tau_{c_{\mathcal{F}}}$, thus obtaining strong performance on the retain classes, while gradually (over the training iterations) unlearning the forget classes (due to weaker gradients in every iteration).
Our method induces the model to find a representation $\phi$ that aligns the minimization of all the proposed objectives: $\mathcal{L}_{\text{ADV}}$, $\mathcal{L}_{\text{MU}}$ and $\mathcal{L}_{\text{SFDA}}$. Our above result shows that $\mathcal{L}_{\text{MU}}$ explicitly focuses on learning the retain classes, while $\mathcal{L}_{\text{ADV}}$ focuses on generating forget samples on which $\mathcal{L}_{\text{MU}}$ is applied. $\mathcal{L}_{\text{SFDA}}$ is responsible for domain adaptation, like other DA methods, and further accentuates the learning with higher gradients for the classes retained in the target domain, thus coherently achieving our overall objective of simultaneous unlearning and domain adaptation in the SCADA-UL setting.
\end{observation}

\vspace{-6pt}
\subsection{Extensions to Additional Settings} \label{subsec:extending}
\vspace{-3pt}
\textbf{C-SCADA-UL} involves unlearning forget classes in sequential requests $\smash{\mathcal{C_F}^i}$ (see Appendix Defn~\ref{def:c-SCADA}).
Our method naturally extends to address this setting: for the initial unlearning request, the approach remains unchanged, and for subsequent requests, we use a small subset of the target train dataset and reduce the number of training epochs (as the model has already been adapted to the target domain).

\textbf{UC-SCADA-UL} poses an additional constraint where the forget classes $\smash{\mathcal{C_F}}$ are unknown but its cardinality is known (see Appendix Defn~\ref{def:uc-SCADA}). To address this setting, we draw inspiration from PADA~\cite{cao2018partial}, which identifies and down-weighs source-exclusive classes for Partial Domain Adaptation. Building on this idea, we adapt the use of the $\gamma$ term to our setting by selecting the bottom-ranked classes by $\gamma$ as the forget classes. The implementation details are provided in the Appendix.

\begin{table*}[t]
  \centering
  \hspace*{\fill}
  \begin{minipage}[t]{0.3085\textwidth}
    \centering
    \captionsetup{font=footnotesize}
        \caption{\footnotesize Results for SCADA-UL on medical DA benchmark: CheXpert $\rightarrow$ NIH Chest X-ray}
        \vspace{-6pt}
        \label{tab:medical}
        \resizebox{\textwidth}{!}{%
        \begin{tabular}{lccc}
        \toprule
        Method & $A_{\mathcal{D}_r^\mathcal{T}} \uparrow$ & $A_{\mathcal{D}_f^\mathcal{T}} \downarrow$ & Score $\uparrow$  \\
        \midrule 
        Original (SF(DA)$^2$~\cite{hwang2024sf})  & \val{36.9}{1.9} & \val{16.0}{8.8} & \val{0.32}{0.0} \\
        Retrain    & \val{39.6}{2.0} & \val{0.0}{0.0} & \val{0.40}{0.0} \\ \midrule
        Finetune   & \val{35.7}{1.4} & \val{9.6}{9.8} & \val{0.33}{0.0} \\
        UNSIR~\cite{tarun2023fast} & \val{27.5}{2.1} & \val{\textbf{0.0}}{0.0} & \val{0.27}{0.0} \\
        ZSMU~\cite{chundawat2023zero} & \val{26.9}{3.5} & \val{\textbf{0.0}}{0.0} & \val{0.27}{0.0} \\
        Lipschitz~\cite{foster2024zero} & \val{{37.6}}{0.2} & \val{\textbf{0.0}}{0.0} & \val{\textbf{0.38}}{0.0} \\
        Nabla Tau~\cite{trippa2024tau} & \val{28.5}{1.9} & \val{\underline{4.4}}{4.1} & \val{0.27}{0.0} \\ 
        Unlearned(+)~\cite{Ahmed_2025_CVPR} & \val{\textbf{42.0}}{1.44} & \val{15.1}{1.9} & \val{\underline{0.36}}{0.0} \\ \midrule
        PADA~\cite{cao2018partial} & \val{36.1}{0.4} & \val{\textbf{0.0}}{0.0} & \val{\underline{0.36}}{0.0} \\
        SHOT~\cite{liang2020we} & \val{31.1}{1.5} & \val{20.2}{1.4} & \val{0.26}{0.0} \\ \midrule
        \rowcolor{green!40!white}
        Ours       & \val{\underline{38.1}}{0.7} & \val{\textbf{0.0}}{0.0} & \val{\textbf{0.38}}{0.0} \\
\bottomrule
      \end{tabular}%
    }
  \end{minipage}%
  \hspace*{\fill}
  \begin{minipage}[t]{0.557\textwidth}
    \centering
    \captionsetup{font=footnotesize}
    \captionof{table}{\footnotesize Membership Inference Attack Accuracy (MIA\%) results for each task in the DomainNet Dataset}
    \vspace{-6pt}
    \label{tab:mia-dmnt}
    \resizebox{\textwidth}{!}{%
      \begin{tabular}{l@{\hskip 8pt}c@{\hskip 8pt}c@{\hskip 8pt}c@{\hskip 8pt}c@{\hskip 8pt}c@{\hskip 8pt}c@{\hskip 8pt}c@{\hskip 8pt}c}
        \toprule
        Method & s → p & c → s & p → c & p → r & r → s & r → c & r → p & Avg. \\
        \midrule
        Original (SF(DA)$^2$~\cite{hwang2024sf})  & \val{56.0}{2.1} & \val{74.4}{2.8} & \val{71.3}{1.6} & \val{52.5}{3.2} & \val{74.7}{1.6} & \val{60.0}{1.6} & \val{62.5}{7.3} & \val{64.5}{2.9} \\
        Retrain    & \val{40.9}{9.1} & \val{41.2}{6.3} & \val{66.8}{2.6} & \val{41.7}{9.2} & \val{58.6}{9.5} & \val{58.2}{4.5} & \val{64.1}{5.7} & \val{53.1}{6.7} \\ \midrule
        Finetune   & \val{54.2}{8.1} & \val{\underline{29.5}}{4.0} & \val{49.0}{2.5} & \val{\underline{27.9}}{3.1} & \val{72.0}{8.7} & \val{61.4}{3.0} & \val{67.9}{7.4} & \val{51.7}{5.3} \\
        UNSIR~\cite{tarun2023fast}      & \val{65.6}{18.} & \val{53.4}{21.} & \val{45.9}{2.4} & \val{56.9}{6.6} & \val{66.3}{6.9} & \val{54.2}{0.9} & \val{66.8}{1.5} & \val{58.5}{8.2} \\
        ZSMU~\cite{chundawat2023zero}       & \val{58.1}{11.} & \val{66.0}{8.2} & \val{68.5}{7.2} & \val{49.8}{7.3} & \val{41.6}{7.9} & \val{59.2}{6.0} & \val{56.6}{5.4} & \val{57.1}{7.6} \\
        Lipschitz~\cite{foster2024zero}  & \val{\textbf{24.7}}{18.} & \val{60.4}{18.} & \val{\textbf{40.1}}{9.6} & \val{\textbf{20.7}}{13.} & \val{\textbf{36.7}}{23.} & \val{\underline{39.4}}{10.} & \val{\textbf{48.2}}{8.9} & \val{\textbf{38.6}}{14.} \\
        Nabla Tau~\cite{trippa2024tau}  & \val{55.1}{6.7} & \val{37.6}{8.4} & \val{45.5}{6.7} & \val{49.3}{9.4} & \val{69.8}{2.4} & \val{\textbf{38.8}}{4.1} & \val{65.1}{6.1} & \val{51.6}{6.3} \\
        Unlearned(+)~\cite{Ahmed_2025_CVPR}  & \val{55.9}{2.4} & \val{48.8}{1.9} & \val{62.7}{7.2} & \val{37.6}{4.8} & \val{58.8}{24.} & \val{59.1}{5.1} & \val{58.4}{3.9} & \val{54.2}{7.7} \\ \midrule
        PADA~\cite{cao2018partial}       & \val{66.9}{0.8} & \val{81.0}{1.1} & \val{74.6}{0.7} & \val{75.1}{0.6} & \val{66.8}{1.5} & \val{64.8}{1.5} & \val{75.7}{2.5} & \val{72.1}{1.3} \\
        SHOT~\cite{liang2020we}       & \val{58.2}{7.9} & \val{76.6}{0.1} & \val{73.9}{2.8} & \val{54.4}{5.5} & \val{75.2}{5.2} & \val{59.5}{5.3} & \val{61.0}{4.6} & \val{65.5}{4.5} \\  \midrule
        \rowcolor{green!40!white}
        Ours       & \val{\underline{52.7}}{1.4} & \val{\textbf{22.0}}{2.6} & \val{\underline{40.3}}{3.1} & \val{28.8}{0.5} & \val{\underline{44.3}}{6.8} & \val{49.0}{0.9} & \val{\underline{58.0}}{8.1} & \val{\underline{42.2}}{3.3} \\
        \bottomrule
      \end{tabular}%
    }
  \end{minipage}%
  \vspace{-9pt}
\hspace*{\fill}
\end{table*}

\vspace{-3pt}
\section{Experiments and Results}
\label{sec:experiment}

\vspace{-3pt}
\noindent \textbf{Datasets.}
We perform experiments on three widely used domain adaptation datasets: \textit{OfficeHome:}~\cite{venkateswara2017deep} This comprises four domains: Art, Clipart, Product, and Real World. Each domain consists of 65 categories of common objects such as spoons, bicycles, and backpacks. \textit{Office31:}~\cite{saenko2010adapting} This contains three domains: Amazon, DSLR, Webcam. Each domain comprises 31 object categories encountered in office settings. \textit{DomainNet:}~\cite{peng2019moment} This is a large-scale dataset consisting of six domains and 345 categories. Following other works~\cite{hwang2024sf}, we experiment on seven tasks in DomainNet-126, using four domains: Clipart, Painting, Real and Sketch. Going beyond existing work, we also study two real-world datasets: \textit{CheXpert}~\cite{irvin2019chexpert} $\rightarrow$ \textit{NIH Chest X-ray}~\cite{wang2017chestx} (medical dataset with chest radiology images of patients labeled with the detected abnormality \cite{he2024domain}), and a land-use dataset \cite{yang2010bag,zou2015deep} consisting of 7 categories of scenes such as grasslands, residential areas, etc.

\noindent \textbf{Baselines.}
As we are the first to introduce a setting with simultaneous adaptation and unlearning tasks, we compare our work with some crude baselines: \textit{Original} (only SFDA) and adapted variants of existing MU methods: \textit{Retrain}, \textit{Finetune}, \textit{UNSIR}~\cite{tarun2023fast}, \textit{ZSMU}~\cite{chundawat2023zero}, \textit{Lipschitz unlearning}~\cite{foster2024zero}, \textit{Nabla Tau}~\cite{trippa2024tau}, \textit{Unlearned(+)}~\cite{Ahmed_2025_CVPR}, and PDA/SFPDA methods: \textit{PADA}~\cite{cao2018partial}, \textit{SHOT}~\cite{liang2020we}. Since these methods have not been designed for domain adaptation and unlearning, we carefully apply them to our setting; the implementations are provided in Appendix~\cref{sec:implementation}. 
We follow SF(DA)$^2$~\cite{hwang2024sf} for our $\mathcal{L}_\text{SFDA}$ loss. We include experiments with other SFDA loss terms such as SHOT~\cite{liang2020we} in Appendix~\cref{subsec:apd_add_loss}. 

\noindent \textbf{Performance Metrics.}
\label{subsubsec:metrics}
\textit{Target Forget Accuracy ($A_{\mathcal{D}_{\smash{f}}^\mathcal{T}}$)}: accuracy of the model on the target forget dataset (zero-shot). 
% While this dataset has never been seen by the model, inference on it helps estimate the information about forget classes in the target domain. A lower $A_{\mathcal{D}_{\smash{f}}^\mathcal{T}}$ indicates better unlearning. 
\textit{Target Retain Accuracy ($\smash{A_{\mathcal{D}_r^\mathcal{T}}}$)}: accuracy on the target retain dataset.
% It measures the model's usability after unlearning.
\textit{Unlearn Score}~\cite{devalapally2023simple}: defined as $\smash{A_{\mathcal{D}_r^\mathcal{T}}}/(100+\smash{A_{\mathcal{D}_{\smash{f}}^\mathcal{T}}})$

% is a score ranging from 0 to 1 that indicates the balance of two objectives: usability on retain classes and successful unlearning of forget classes.
\textit{Time consumption}: the total time taken to unlearn forget classes and adapt to the target domain. 
\textit{MIA Accuracy (MIA\%)} quantifies the effectiveness of a membership inference attack (MIA) model in identifying forget classes. Conventional MU methods typically follow the MIA training procedure proposed in \cite{golatkar2020forgetting}; we however adapt this approach for class-level unlearning by modifying the discrimination task to differentiate between entropies of retain data and unseen class data. %In practice, unseen class data are generated using non-overlapping classes from a different dataset. 
An ideal unlearning method would ensure forget classes are misclassified as unseen class data.

\noindent \textbf{Results.}
We perform experiments for four different settings: single class SCADA-UL, multi class SCADA-UL, UC-SCADA-UL and C-SCADA-UL. Each experiment is run three times and the mean and standard deviation of the metrics are reported in all studies. For all tables, the best score for each metric is in bold, excluding original and retrain. Due to space constraints, we present results for SCADA-UL on a medical dataset, multi-class SCADA-UL on benchmark datasets, MIA accuracy on DomainNet herein, and the other results in the Appendix.   

From \cref{tab:mc-SCADA}, we see that in the multi class SCADA-UL setting with $\mathcal{C_F}=\{1,2,3\}$, the original model as well as PDA methods (PADA, SHOT) exhibit strong zero-shot capabilities on $\mathcal{D}_{\smash{f}}^\mathcal{T}$; however, a high accuracy is undesirable in the context of unlearning. Retraining on the other hand attains zero accuracy on $\mathcal{D}_{\smash{f}}^\mathcal{T}$ for all tasks, while maintaining high accuracy on $\mathcal{D}_r^\mathcal{T}$. (This method serves only as a gold standard for comparison, since this data is otherwise unavailable). Finetuning fails to perform well, yielding results similar to the original model due to lack of unlearning on the forget class. Existing MU methods when adapted to our setting (as described in Appendix~\ref{subsec:adapting-works}) perform poorly too, either resulting in a significant drop in retain accuracy (UNSIR, Nabla Tau) or still maintaining a high forget accuracy (ZSMU, Lipschitz). We hypothesize methods like UNSIR, Nabla Tau and Unlearned (+) perform poorly in this setting since they were not designed to handle shift in the data distribution. In contrast, our method demonstrates strong performance, offering results comparable to retraining. 

Table~\ref{tab:medical} shows results on a real-world medical domain adaptation benchmark for chest disease classification: CheXpert → NIH Chest X-ray. In medical applications like these, regulatory bodies may mandate that source-exclusive medical conditions are not transferred to the target domain. Additionally, our present Membership Inference Attack Accuracy (MIA\%) in \cref{tab:mia-dmnt} shows that our method tends to either have the lowest or second lowest accuracy out of all baselines. Interestingly, we find that methods such as UNSIR, which achieve low forget accuracy, still have a high MIA\%, while methods with higher forget accuracy such as Lipschitz have a low MIA\%, corroborating the necessity of multiple metrics to evaluate such settings. 

% in Table~\ref{tab:shot-oda}, we provide results on the open-set DA setting, replacing the SFDA loss term with the open-set DA version of SHOT~\cite{liang2020we} as described in the paper. This is evaulated on the OfficeHome dataset with source-only classes $\{1,2,3\}$ and target-only classes $\{34, 35, 36, 37, 38\}$. We report retain accuracy or shared class accuracy (OS$^*$), target-only class accuracy (OS), forget accuracy, and unlearn score. The results show largely similar trends to the original SFPDA setting, where existing methods perform poorly by dropping retain accuracy significantly (UNSIR, ZSMU, Nabla Tau) or still maintaining high forget accuracy (Lipschitz). In contrast, our method demonstrates strong unlearning performance while maintaining high retain accuracy, achieving the best overall unlearn score.

Full result tables, as well as results for UC-SCADA-UL, C-SCADA-UL, time consumption, another real-world dataset (scene classification), and other SFDA loss functions are in the Appendix~\cref{sec:results-apd}.
\vspace{-3pt}
\section{Analysis and Ablation Studies}
\label{sec:analysis}
\vspace{-2pt}
We conduct comprehensive studies to evaluate key components of our method. All experiments measure multi-class SCADA-UL performance across 3 trials on the OfficeHome dataset, reporting retain accuracy \( (\text{Acc }{\mathcal{D}_r^\mathcal{T}})\) and forget accuracy \((\text{Acc }{\mathcal{D}_{\smash{f}}^\mathcal{T}})\) averaged across 12 tasks. Due to space constraints, we provide two ablations/studies here; the remaining studies (including on the $\alpha$ term and number of adversarial samples) are in Appendix.

\noindent \textbf{Stage at which algorithm is applied.} \cref{tab:ab-stage} shows our algorithm applied before, during, and after the domain adaptation process. In the after-adaptation scenario, we see a significant drop in accuracy due to \textit{catastrophic forgetting}. While $\mathcal{L}_\text{SFDA}$ helps reverse this in the before-adaptation scenario, it 

\begin{wraptable}{r}{0.65\columnwidth}
\centering
\scriptsize
\vspace{-1pt}
\captionsetup{font=footnotesize}
\caption{\footnotesize Study on stage of applying our algorithm}
\vspace{-5pt}
\label{tab:ab-stage}
\begin{tabular}{ccc}
\toprule
Stage & Acc ${\mathcal{D}_r^\mathcal{T}} \uparrow$  & Acc ${\mathcal{D}_f^\mathcal{T}} \downarrow$ \\ \midrule
Before Adaptation & \val{75.5}{1.6} & \val{52.0}{5.3} \\
\rowcolor{green!40!white}
During Adaptation & \val{75.1}{1.3} & \val{0.0}{0.0} \\
After Adaptation &  \val{58.2}{7.1} & \val{0.0}{0.0} \\
\bottomrule
\end{tabular}
\vspace{-5pt}
\end{wraptable}
has high forget accuracy which is undesirable. The proposed during -adaptation approach obtains good forget and retain performance via simultaneous optimization of unlearning and adaptation objectives.

\noindent\textbf{Labeling Strategy.} Our proposed labeling strategy  outperforms uniform and random labeling, as seen by the higher
\begin{wraptable}{r}{0.55\columnwidth}
\centering
\scriptsize
\vspace{-10pt}
\captionsetup{font=footnotesize}
\caption{\footnotesize Choice of relabeling strategy}
\vspace{-5pt}
\label{tab:ab-label}
\begin{tabular}{ccc}
\toprule

Strategy & Acc ${\mathcal{D}_r^\mathcal{T}} \uparrow$  &  Acc ${\mathcal{D}_f^\mathcal{T}} \downarrow$ \\ \midrule
\rowcolor{green!40!white}
Rescaled & ${75.1_{\color{gray}{\pm1.3}}}$ & ${0.0_{\color{gray}{\pm0.0}}}$ \\
Uniform & ${70.5_{\color{gray}{\pm3.6}}}$ & ${3.5_{\color{gray}{\pm4.1}}}$ \\
Random & ${1.5_{\color{gray}{\pm0.4}}}$ & ${0.0_{\color{gray}{\pm0.0}}}$ \\

\bottomrule
\end{tabular}
%\vspace{-4pt}
\vspace{-8pt}
\end{wraptable}
\(\text{Acc }{\mathcal{D}_r^\mathcal{T}}\) and lower Acc \({\mathcal{D}_{\smash{f}}^\mathcal{T}}\) in \cref{tab:ab-label}. Uniform labeling achieves moderate unlearning but underperforms our method, while random labeling fails in our setting, leading to catastrophic forgetting. This behavior is likely due to \(\mathcal{L}_\text{MU}\) dominating over \(\mathcal{L}_\text{SFDA}\), disrupting training stability. Furthermore, our labels align the post-unlearning model's output distributions more closely to that of the retrained model (See Appendix). This behavior likely explains the effectiveness of our labeling strategy -- it provides a more natural objective for $\mathcal{L}_\text{MU}$, that minimizes conflict with retain classes.
% leading to higher $A_{\mathcal{D}_r^\mathcal{T}}$ and lower $A_{\mathcal{D}_{\smash{f}}^\mathcal{T}}$.

\section{Conclusions}
\label{sec:conclusion}

In this work, we present a novel machine unlearning setting called \textbf{U}n\textbf{l}earning \textbf{S}ource-exclusive \textbf{C}l\textbf{A}sses in \textbf{D}omain \textbf{A}daptation (SCADA-UL), along with two variants UC-SCADA-UL and C-SCADA-UL. This setting addresses the task of adapting a source model to a target domain while unlearning the specific source-exclusive classes, and has increased relevance as learning models get commonly adapted across domains today. To address these settings, we propose a new unlearning algorithm based on optimizing the model to iteratively forget the best estimate of forget classes throughout the domain adaptation process. Our thorough empirical and theoretical analysis highlights the effectiveness of our approach in the settings. 
% We believe that this work opens up a new direction in machine unlearning, which can encourage further efforts in the community. 

\section*{Acknowledgements}
% Poornima is supported by the Google research Grant. We are also grateful to our anonymous reviewers for their valuable feedback in improving the paper presentation quality.
Arnav is supported by ACM India iKDD (Special Interest Group On Knowledge Discovery and Data Mining) and the University of Michigan Rackham Conference Travel Grant. Poornima is supported by the Google Research Grant and IIT-Hyderabad Conference Travel Support Grant. We are grateful for all the above institutions for their support to carry out this work. We are also grateful to our anonymous reviewers, program chairs and area chairs for their valuable feedback in improving the paper quality.

% \section{Conclusions}
% \label{sec:conclusion}
% % \vspace{-3pt}
% In this work, we present a novel machine unlearning setting called \textbf{U}n\textbf{l}earning \textbf{S}ource-exclusive \textbf{C}l\textbf{A}sses in \textbf{D}omain \textbf{A}daptation (SCADA-UL), along with two variants UC-SCADA-UL and C-SCADA-UL. This setting addresses the task of adapting a source model to a target domain while unlearning the specific source-exclusive classes, and has increased relevance as learning models get commonly adapted across domains today. To address these settings, we propose a new unlearning algorithm based on optimizing the model to iteratively forget the best estimate of forget classes throughout the domain adaptation process. Our thorough empirical and theoretical analysis highlights the effectiveness of our approach in the settings. We believe that this work opens up a new direction in machine unlearning, which can encourage further efforts in the community. 

% Buggy section, added to 6_analysis.tex
{
    \small
    \bibliographystyle{ieeenat_fullname}
    \bibliography{main}
}
\newpage
\onecolumn

\setcounter{section}{0}
\renewcommand{\thesection}{A.\arabic{section}}
\renewcommand{\thetable}{A.\arabic{table}}
\renewcommand{\thefigure}{A.\arabic{figure}}
\renewcommand{\theequation}{A.\arabic{equation}}

\newcommand{\ToCEntry}[3]{%
  \ifcase#1
    \noindent\ref{#3}. #2\hspace{1.5em}\dotfill\hspace{1.5em}\pageref{#3} \\ % Level 0
  \or
    \noindent\hspace*{2em}\ref{#3}. #2\hspace{1.5em}\dotfill\hspace{1.5em}\pageref{#3} \\ % Level 1
  \else
    \noindent\ref{#3}. #2\hspace{1.5em}\dotfill\hspace{1.5em}\pageref{#3} \\ % fallback
  \fi
}

\section*{\Large Appendix}

\section*{Contents}
\ToCEntry{0}{Proof of Theorem 4.1}{sec:proof41}
\ToCEntry{0}{Additional Algorithms}{sec:add_algs}
\ToCEntry{0}{Definitions and Motivation for UC-SCADA-UL and C-SCADA-UL}{sec:motivSCADA}
\ToCEntry{1}{Formal Definitions}{subsec:formal-def}
\ToCEntry{1}{Real-World Motivations for Variants}{subsec:real-world-mot}
\ToCEntry{0}{Additional Analysis}{sec:add_analysis}
\ToCEntry{1}{Applying Existing MU Methods in SCADA-UL}{subsec:apply_existing}
\ToCEntry{1}{Study on Adversarial Samples}{subsec:study_adv}
\ToCEntry{1}{Visualizing Unlearning and Adversarial Loss Terms}{subsec:visloss}
\ToCEntry{1}{Visualizing Outputs of Our Method with Different Labeling Strategies}{subsec:visoutputs}
\ToCEntry{1}{Failure Case Analysis: UC-SCADA-UL}{subsec:failure_case}
\ToCEntry{1}{Additional Ablation Studies}{subsec:additional-ablation}
\ToCEntry{1}{Discussion of Assumptions in Our Method}{subsec:assump}
\ToCEntry{0}{Additional Experimental Results}{sec:results-apd}
\ToCEntry{1}{Runtime Analysis of Methods}{subsec:time_cons}
\ToCEntry{1}{Experiments on Land-use Classification Dataset}{subsec:scenes}
\ToCEntry{1}{SCADA Unlearning}{subsec:apd_scada}
\ToCEntry{1}{UC-SCADA Unlearning}{subsec:apd_uc_scada}
\ToCEntry{1}{C-SCADA Unlearning}{subsec:apd_c_scada}
\ToCEntry{1}{Use of Other SFDA Loss Functions}{subsec:apd_add_loss}
\ToCEntry{1}{Extensions to Open-set Domain Adaptation}{subsec:apd_osda}
\ToCEntry{0}{Limitations}{sec:limitations}
\ToCEntry{0}{Implementation Details}{sec:implementation}
\ToCEntry{1}{Compute Resources}{subsec:compute}
\ToCEntry{1}{Adapting MU and PDA Methods to SCADA-UL}{subsec:adapting-works}
\ToCEntry{1}{Hyperparameters}{subsec:hyperparams}
\ToCEntry{1}{Metrics}{subsec:metrics_hpm}
\ToCEntry{1}{Real-world Dataset Implementations}{subsec:realworld_imple}
%\ToCEntry{1}{Selecting Unknown Classes}{subsec:selecting_uc}
\hrule

\section{Proof of Theorem 4.1}
\label{sec:proof41}
\gradients*
\begin{proof} 

Let the pre-final layer output of an adversarial sample $\hat{x}$ generated by the model being adapted, $w^{\mathcal{T}}$, be denoted as $\phi(\hat{x}) = \hat{\phi}$.

Then, the MU loss in Expr. (\ref{eq:mu-loss}) is given by

\begin{align}
    \label{eq: losses}
     \mathcal{L}_{\text{MU}} (w^{\mathcal{T}}, \hat{x}, \hat{y})
     &= \sum_{i \in c_\mathcal{R}} - \hat{y}_i \log( \frac{\exp(\hat{\phi}^T\tau_i)}{\sum_{j \in C} \exp(\hat{\phi}^T \tau_j)})
     \\
     &= \sum_{i \in c_{R}} - \hat{y}_i \hat{\phi}^T\tau_i  +\log ({\sum_{j \in C} \exp(\hat{\phi}^T \tau_j)})
\end{align}
The second equality holds due to our rescaled labeling strategy; $\sum_{i \in c_{R}} \hat{y}_i = 1$ (Expr. ~\ref{eq: bayesreescale}). 

The gradients for any class $c$ are given by

\begin{align}
    \label{eq: derivatives}
 \frac{\partial \mathcal{L}_{\text{MU}} (w^{\mathcal{T}}, \hat{x}, \hat{y})}{\partial \tau_c} &=  - \hat{y}_c \hat{\phi} \mathbf{1}_{c \notin c_\mathcal{F}} + \underbrace{\frac{\exp(\hat{\phi}^T\tau_c)}{\sum_{j \in C} \exp(\hat{\phi}^T \tau_j)}}_{f_{w^{\mathcal{T}}} (\hat{x}) = y_c} \hat{\phi}
 \\
    &= \hat{\phi}(y_c - \mathbf{1}_{ c \notin c_\mathcal{F} } \hat{y}_c ) 
\end{align}
Specifically, the norms of the gradients w.r.t. the weights in the case where $c \in c_\mathcal{F}$ or $c \in c_\mathcal{R}$, are given by
\begin{align} 
    \label{eq: derivative_comparison}
     \frac{\partial \mathcal{L}_{\text{MU}} (w^{\mathcal{T}}, \hat{x}, \hat{y})}{\partial \tau_{c_\mathcal{R}}} &= \| \hat{\phi} (y_c - \frac{y_c}{\sum_{j \in c_{R}} y_j} ) \| 
     \\
     \frac{\partial \mathcal{L}_{\text{MU}} (w^{\mathcal{T}}, \hat{x}, \hat{y})}{\partial \tau_{c_\mathcal{F}}} &= \| \hat{\phi}y_c \|
\end{align}

We crucially observe that because of the nature of our adversarial sample generation in Expr. (\ref{eq:adv-loss}),
\begin{equation*}
\sum_{j \in c_{R}} y_j \le \delta    
\end{equation*}
where $\delta \in (0, 1)$, because $P(\hat{x} \in c_{R}) = \sum_{j \in c_\mathcal{R}} y_j$ represents the softmax probability of the sample $\hat{x}$ belonging to any retain class, which is precisely what is minimized in $\mathcal{L}_{\text{ADV}}$. This leads to our main gradient flow inequality

\begin{equation}
    \label{eq:inequality}
    \norm{\frac{\partial \mathcal{L}_{\text{MU}} (w^{\mathcal{T}}, \hat{x}, \hat{y})}{\partial \tau_{c_\mathcal{R}}}} \ge  \left(\frac{1}{\delta} - 1\right) \norm{\frac{\partial \mathcal{L}_{\text{MU}} (w^{\mathcal{T}}, \hat{x}, \hat{y})}{\partial \tau_{c_\mathcal{F}}}}
\end{equation}

This inequality indicates that the gradient of the machine unlearning loss is more significant on the weights connected to the final-layer neurons of the classes that need to be retained, given that the adversarial sample is initially classified as part of a forget class $c_\mathcal{F}$ with a probability at least $1 - \delta$.
\end{proof}

\section{Additional Algorithms} \label{sec:add_algs}
\begin{algorithm}[t]
% \footnotesize
% \captionsetup{font=footnotesize}
\caption{ \textbf{Adversarial Optimization for SCADA-UL (AO\_SCADA\_UL: Detailed Algorithm)}}
\label{alg:extended-minimax}
\begin{algorithmic}[0]

\State \textbf{Inputs:}
\State \hspace{10pt} Source model $w^\mathcal{S}$, target data $\mathcal{D}^\mathcal{T}$ %source test data $\mathcal{D}^\mathcal{S}$
\State \textbf{Require:}
\State \hspace{10pt} Forget classes $\mathcal{C_F}$, SFDA loss $\mathcal{L}_\text{SFDA}$, loss trade-off $\alpha$
\State \hspace{10pt} Learning rates $\eta_1$ (model), $\eta_2$ (adv.\ samples), $\eta_{\text{init}}$ (init.\ step)
\State \hspace{10pt} Epochs $M$, total steps $N$, Initialization $T_{\text{init}}$

\State \textbf{Init:}
\State \hspace{10pt} $w^\mathcal{T} \gets w^\mathcal{S}$
\State \hspace{10pt} Initialize optimizer and LR scheduler for $w^\mathcal{T}$

\Statex

% -----------------------------
% ADV INIT
% -----------------------------
\For{each class $c_\mathcal{F} \in \mathcal{C_F}$}
  \State Initialize adversarial sample $\hat{x}_{c_\mathcal{F}}$ (random)
  \For{$t = 1$ to $T_\text{init}$}
    \State $\hat{x}_{c_\mathcal{F}} \gets 
      \hat{x}_{c_\mathcal{F}} 
      - \eta_{\text{init}} \nabla_{\hat{x}_{c_\mathcal{F}}}
      \mathcal{L}_\text{CE}(w^\mathcal{T}, \hat{x}_{c_\mathcal{F}}, c_\mathcal{F})$
      \Comment{Adversarial sample “seed”}
  \EndFor
\EndFor

\Statex

% -----------------------------
% EPOCHS
% -----------------------------
\For{$\text{epoch} = 1$ to $M$}

  \State Iterate through target data

  \For{each forget class $c_\mathcal{F} \in \mathcal{C_F}$}

    \For{$\text{step} = 1$ to $N / |\mathcal{C_F}|$}

      % -----------------------------
      % SFDA STEP
      % -----------------------------
      \State Sample target image $x^\mathcal{T} \sim \mathcal{D}^\mathcal{T}$

     \State Compute SFDA loss $\,\mathcal{L}_\text{SFDA}(w^\mathcal{T}, x^\mathcal{T})$

      % -----------------------------
      % FORWARD ON ADV SAMPLE
      % -----------------------------
      \State Obtain logits on adversarial sample:
      \[
        y = f_{w^\mathcal{T}}(\hat{x}_{c_\mathcal{F}})
      \]

      % -----------------------------
      % LABEL CONSTRUCTION
      % -----------------------------

    \State Set $\hat{y}_{c_\mathcal{F}} = 0$; renormalize remaining logits: 
    \[
      \hat{y}_i = \frac{y_i}{\sum_{j \neq c_\mathcal{F}} y_j}
    \]

      % -----------------------------
      % JOINT LOSS
      % -----------------------------
      \WrapState{%
      Set \(
      \varphi = 
      \mathcal{L}_\text{SFDA}(w^\mathcal{T}, x^\mathcal{T})
      + \alpha\,\mathcal{L}_\text{CE}(w^\mathcal{T}, \hat{x}_{c_\mathcal{F}}, \hat{y})
      \)
      }

      % -----------------------------
      % MODEL UPDATE
      % -----------------------------
      \State Update model parameters:
      \[
        w^\mathcal{T} \gets w^\mathcal{T} - \eta_1 \nabla_{w^\mathcal{T}} \varphi
      \]
      \State Apply LR scheduler

      % -----------------------------
      % ADVERSARIAL SAMPLE UPDATE
      % -----------------------------
      \State Update adversarial sample:
      \[
        \hat{x}_{c_\mathcal{F}} \gets \hat{x}_{c_\mathcal{F}}
        - \eta_2 \nabla_{\hat{x}_{c_\mathcal{F}}}\mathcal{L}_\text{CE}(w^\mathcal{T}, \hat{x}_{c_\mathcal{F}}, c_\mathcal{F})
      \]

    \EndFor
  \EndFor

\EndFor

\State \Return Final unlearned target model $w_u^\mathcal{T} = w^\mathcal{T}$

\end{algorithmic}
\end{algorithm}

\begin{figure}
\begin{center}
\begin{minipage}{0.8\textwidth}

\begin{algorithm}[H]
\caption{\textbf{Adversarial Optimization for UC-SCADA Unlearning}}
\label{alg:uc-minimax}
\begin{algorithmic}[0]
\State \textbf{Inputs:} Source Model $w^\mathcal{S}$, target dataset $\mathcal{D^T}$, number of forget classes $|\mathcal{C_F}|$, SFDA loss $\mathcal{L}_\text{SFDA}$, trade-off $\alpha$, learning rates $\eta_1$, $\eta_2$
\State
\State \textbf{Init:} $w^\mathcal{T} = w^\mathcal{S}$; $\gamma = 0$
\State
\For{each $x^\mathcal{T} \in \mathcal{D^T}$}
    \State $\gamma \gets \gamma + f_{w^\mathcal{T}}(x^\mathcal{T})$
\EndFor
\State
\State $\mathcal{C^*_F} \gets \text{Top\_K}(-\gamma, R \cdot |\mathcal{C_F}|)$ {\small \Comment{Lowest $R\cdot|\mathcal{C_F}|$ classes according to $\gamma$}}
\State $w_u^\mathcal{T} \gets$ AO\_SCADA\_UL($w^\mathcal{T}$, $\mathcal{D^T}$, $\mathcal{C^*_F}$, $\mathcal{L}_\text{SFDA}$, $\alpha$, $\eta_1$, $\eta_2$) {\small \Comment{Alg (\ref{alg:minimax})}}
\State
\State \textbf{Return:} $w_u^\mathcal{T}$
\end{algorithmic}
\end{algorithm}
\begin{algorithm}[H]
\caption{\textbf{Adversarial Optimization for C-SCADA Unlearning}}
\label{alg:c-minimax}
\begin{algorithmic}[0]
\State \textbf{Inputs:} Source Model $w^\mathcal{S}$, target dataset $\mathcal{D^T}$, set of forget classes $\mathcal{C_F}$, SFDA loss $\mathcal{L}_\text{SFDA}$, trade-off $\alpha$, learning rates $\eta_1$, $\eta_2$
\State
\For{each $i$, enumerating $\mathcal{C_F}^i \in \mathcal{C_F}$}
    \If{$i = 1$}
        \State $w_u^\mathcal{T} \gets$ AO\_SCADA\_UL($w^\mathcal{S}$, $\mathcal{D^T}$, $\mathcal{C_F}^i$, $\mathcal{L}_\text{SFDA}$, $\alpha$, $\eta_1$, $\eta_2$) {\small \Comment{Alg (\ref{alg:minimax})}}
    \Else
        \State $\mathcal{D}_\text{sub}^\mathcal{T} = \text{Subset}(\mathcal{D^T})$
        \State $w_u^\mathcal{T} \gets$ AO\_SCADA\_UL($w_u^\mathcal{T}$, $\mathcal{D}_\text{sub}^\mathcal{T}$, $\mathcal{C_F}^i$, $\mathcal{L}_\text{SFDA}$, $\alpha$, $\eta_1$, $\eta_2$)
    \EndIf
\EndFor
\State
\State \textbf{Return:} $w_u^\mathcal{T}$
\end{algorithmic}
\end{algorithm}

\end{minipage}
\end{center}
\end{figure}

We provide an extended version of Algorithm~\ref{alg:minimax} with full training details in Algorithm~\ref{alg:extended-minimax}. Whenever this algorithm is invoked by another algorithm, we refer to it as \emph{AO\_SCADA\_UL()} (\textbf{SCADA-UL} solved via \textbf{A}dversarial \textbf{O}ptimization) for convenience.
% \footnote{This is a temporary name introduced for readability; Algorithm~\ref{alg:minimax} is otherwise unnamed.}

\textbf{UC-SCADA-UL.} \cref{alg:uc-minimax} describes the procedure for our algorithm adapted for the UC-SCADA-UL setting, where the identity of the forget classes is not known (Definition~\ref{def:uc-SCADA}). It involves an additional forget class prediction step that estimates the classes that are most likely the forget classes based on the target dataset $\mathcal{D^T}$. This estimation step utilizes a term $\gamma \in \mathbb{R}^d$ where $d$ is the number of classes (similar to~\cite{cao2018partial}). This term provides a relative measure of class relevance for a dataset. The bottom $R\cdot\left|\mathcal{C_F}\right|$ classes associated with $\gamma$ are selected as the predicted forget classes and the original algorithm (Algorithm~\ref{alg:minimax}) is applied using these classes. 

\textbf{C-SCADA-UL.} \cref{alg:c-minimax} shows the adaptation of our algorithm for the C-SCADA-UL setting (Definition~\ref{def:c-SCADA}), where the unlearning requests can be received over multiple time steps. For the first set of classes $\mathcal{C_F}^1$, the process remains identical to the original algorithm (Algorithm~\ref{alg:minimax}). For subsequent classes, we use a subset of the target data $\mathcal{D}_\text{sub}^\mathcal{T}$ and apply the original algorithm to the previously adapted model $w_u^\mathcal{T}$.

\section{Definitions and Motivation for UC-SCADA-UL and C-SCADA-UL} \label{sec:motivSCADA}

\subsection{Formal Definitions} \label{subsec:formal-def}

\begin{definition}  
\label{def:uc-SCADA} \textbf{(UC-SCADA-UL).}  
Unknown Class SCADA-UL is the process of learning a function $\mathcal{U}: \{w^\mathcal{S}, \mathcal{D}_r^\mathcal{T}, \left|\mathcal{C_F}\right|\} \rightarrow w_u^\mathcal{T}$ that produces an unlearned, adapted model which behaves as though it were adapted from a source model not trained on $\mathcal{D}_{\smash{f}}^\mathcal{S}$, all without requiring knowledge of $\mathcal{C_F}$.
\end{definition}

\begin{definition}
\label{def:c-SCADA}
\textbf{(C-SCADA-UL).}  
Continual SCADA-UL is the process of learning a sequence of unlearning functions $\mathcal{U}^i(w^{\mathcal{T}, i-1}, \mathcal{D}_r^\mathcal{T}, \mathcal{C_F}^1 \cup \dots \cup \mathcal{C_F}^i)$ where $i>0, w^{\mathcal{T}, 0} = w^\mathcal{S}$. Each function produces an unlearned, adapted model $w_u^{\mathcal{T}, i}$ that behaves as though it were adapted from a source model that was not trained on the data associated with $\mathcal{C_F}^1 \cup \dots \cup \mathcal{C_F}^i$. 
\end{definition}
We note that the term ``continual'' here specifically refers to sequential unlearning requests and not traditional continual learning where new classes are learned over tasks. 

\subsection{Real-World Motivations for Variants} \label{subsec:real-world-mot}

In \cref{sec:intro}, we discuss real-world use cases that motivate this work. An increased use of finetuning and adapting models underlines the significance of the setting proposed in this work. To discuss this further, the SCADA unlearning setting (\cref{def:SCADA}) assumes prior knowledge of source-exclusive classes alongside the source model and unlabeled target dataset. However, this assumption may not hold in practice sometimes, and it may be difficult to infer the source-exclusive classes from the source model \( w^\mathcal{S} \) and the given unlabeled target retain data \( \mathcal{D}_r^\mathcal{T} = \{(x_{i,r}^{\mathcal{T}})\}_{i=1}^n \). (If the target dataset were labeled, identifying these classes would be straightforward as one could iterate through the dataset to compile the set of observed target labels \( \mathcal{C_T} \) and deduce the source-exclusive classes as \( \mathcal{C_F} = \mathcal{C_S} \setminus \mathcal{C_T} \).) %Without these labels, this approach is infeasible, motivating a novel set of algorithms to approach this task. 
Such a scenario occurs, for example, when one does not have access to the source domain label space (e.g. fraud categories of a given country in a fraud detection application). Our variants address such a constraint; we introduce the \textbf{UC-SCADA-UL} setting (\cref{def:uc-SCADA}), where only the number of source-exclusive classes is assumed to be known. Practically, this quantity can be estimated using domain knowledge, such as the sizes of the source and target label spaces: if \( |\mathcal{C_S}| \) and \( |\mathcal{C_T}| \) are known, \( |\mathcal{C_F}| \) can be computed as \( |\mathcal{C_S}| - |\mathcal{C_T}| \), even when specific class identities are unavailable. Additionally, in some scenarios only a subset of $\mathcal{C_F}$ may be known before the adaptation process, with additional classes uncovered over time (eg., through user-initiated data removal requests). The \textbf{C-SCADA-UL} setting (\cref{def:c-SCADA}) formalizes this process by allowing sequential unlearning of disjoint subsets $\mathcal{C_F}^i \subseteq \mathcal{C_F}$ across multiple steps, requiring the model to dynamically discard knowledge of $\mathcal{C_F}^i$.

\begin{figure}[t]
    \centering
    %---------------------- TOP ROW (REAL SAMPLES) ----------------------%
    \begin{subfigure}{0.01\linewidth}
        
    \end{subfigure}
    \hfill
    \begin{subfigure}{0.1\linewidth}
        \centering
        \includegraphics[width=\linewidth]{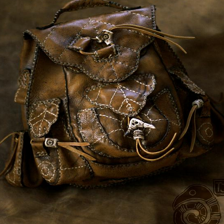}
    \end{subfigure}
    \hfill
    \begin{subfigure}{0.1\linewidth}
        \centering
        \includegraphics[width=\linewidth]{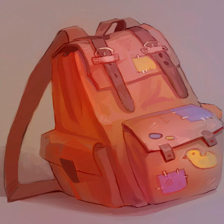}
    \end{subfigure}
    \hfill
    \begin{subfigure}{0.1\linewidth}
        \centering
        \includegraphics[width=\linewidth]{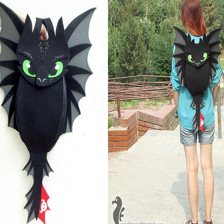}
    \end{subfigure}
    \hfill
    \begin{subfigure}{0.1\linewidth}
        \centering
        \includegraphics[width=\linewidth]{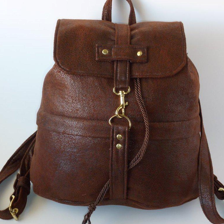}
    \end{subfigure}
    \hfill
    \begin{subfigure}{0.01\linewidth}
        
    \end{subfigure}

    \vspace{0.4cm}

    %---------------------- BOTTOM ROW (ADVERSARIAL) ----------------------%
    \hfill
    \begin{subfigure}{0.1\linewidth}
        \centering
        \includegraphics[width=\linewidth]{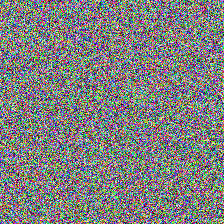}
    \end{subfigure}
    \hfill
    \begin{subfigure}{0.1\linewidth}
        \centering
        \includegraphics[width=\linewidth]{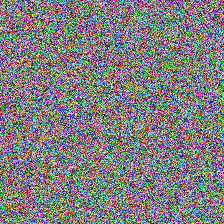}
    \end{subfigure}
    \hfill
    \begin{subfigure}{0.1\linewidth}
        \centering
        \includegraphics[width=\linewidth]{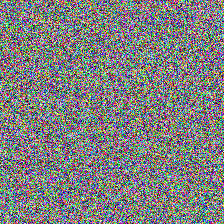}
    \end{subfigure}
    \hfill
    \begin{subfigure}{0.1\linewidth}
        \centering
        \includegraphics[width=\linewidth]{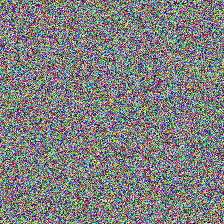}
    \end{subfigure}
    \hfill
    \begin{subfigure}{0.01\linewidth}
        
    \end{subfigure}
    % \captionsetup{font=footnotesize}
    \caption{ 
        \textbf{Visual appearance of Adversarial Samples.}  
        Top row shows examples of samples from the source forget class (Backpack) from OfficeHome Art domain, while the bottom row shows examples of adversarial samples of this class. Although these samples appear as random noise, they are confidently classified as a backpack by the model. Moreover, in additional studies (\cref{fig:tsne}), we show that these samples match the forget class even on a representational level)
    }
    \label{fig:real_vs_adv}

\end{figure}

\section{Additional Analysis} \label{sec:add_analysis}

\subsection{Applying Existing MU Methods in SCADA-UL} \label{subsec:apply_existing}
\begin{table}
    \centering
    % \captionsetup{font=footnotesize}
    % \caption{\footnotesize Existing Data-Free MU methods struggle in Domain Adaptation}
    \caption{\textbf{Existing Data-Free MU methods struggle in Domain Adaptation.} Existing methods \cite{tarun2023fast,chundawat2023zero,trippa2024tau} when applied either before (Source → Unlearn → Adapt), during (Source → (Unlearn + Adapt)), or after (Source → Adapt → Unlearn) the domain adaptation process perform poorly in our setting, motivating our method.}
    \label{tab:existing-methods}
    \vspace{3px}
    \begin{adjustbox}{max width=0.8\columnwidth}
    \small
    \begin{tabular}{lcccccc} \toprule
      \multicolumn{1}{c}{\multirow{2}{*}{\textbf{Method}}} & \multicolumn{2}{c}{\textbf{Source $\rightarrow$ Unlearn $\rightarrow$ Adapt}} & \multicolumn{2}{c}{\textbf{Source $\rightarrow$ (Unlearn $+$ Adapt)}} & \multicolumn{2}{c}{\textbf{Source $\rightarrow$ Adapt $\rightarrow$ Unlearn}} \\
      & $A_{\mathcal{D}_r^\mathcal{T}} \uparrow$ & $A_{\mathcal{D}_f^\mathcal{T}} \downarrow$  & $A_{\mathcal{D}_r^\mathcal{T}} \uparrow$ & $A_{\mathcal{D}_f^\mathcal{T}} \downarrow$ & $A_{\mathcal{D}_r^\mathcal{T}} \uparrow$ & $A_{\mathcal{D}_f^\mathcal{T}} \downarrow$ \\ \midrule
      Original (SF(DA)$^2$~\cite{hwang2024sf}) & \val{64.3}{1.9} & \val{35.9}{4.0} & N/A & N/A & N/A & N/A \\
      Retrain & \val{65.2}{1.9} & \val{0.0}{0.0} & N/A & N/A & N/A & N/A\\ \midrule
      UNSIR~\cite{tarun2023fast} & \val{20.6}{2.9} & \val{0.0}{0.0} & \val{60.6}{1.9} & \val{62.1}{6.8} & \val{12.2}{8.5} & \val{0.1}{0.1} \\
      ZSMU~\cite{chundawat2023zero} & \val{62.5}{1.6} & \val{29.6}{7.8} & \val{62.7}{0.7} & \val{53.7}{3.2} & \val{58.1}{13.} & \val{26.7}{10.} \\
      Nabla Tau~\cite{trippa2024tau} & \val{50.1}{3.7} & \val{1.2}{2.0} & \val{52.6}{1.8} & \val{34.1}{12.4} & \val{32.2}{5.9} & \val{1.0}{1.4} \\
      \bottomrule
    \end{tabular}
    \end{adjustbox}
\end{table}

We experiment with applying adapted versions of existing MU methods~\cite{tarun2023fast,chundawat2023zero,trippa2024tau} in the single class SCADA-UL setting on DomainNet with $c_\mathcal{F}=1$ (see \cref{sec:implementation} for implementation details). We test three approaches: applying the method on the source model and subsequently adapting that model to the target domain (``Source $\rightarrow$ Unlearn $\rightarrow$ Adapt'' in Table~\ref{tab:existing-methods}), applying the method during the adaptation process by adding loss terms (``Source $\rightarrow$ (Unlearn $+$ Adapt)'' in Table~\ref{tab:existing-methods}), and applying the method on the target adapted model (``Source $\rightarrow$ Adapt $\rightarrow$ Unlearn'' in Table~\ref{tab:existing-methods}). From the results in \cref{tab:existing-methods}, we see that all the approaches to applying existing methods perform poorly in the SCADA-UL setting. This limitation arises from their design, which does not account for varying data distributions. It highlights the need for a targeted MU solution for the SCADA-UL setting.

\subsection{Study on Adversarial Samples} \label{subsec:study_adv}
\begin{table}[t]
    % \captionsetup{font=footnotesize}
    % \caption{\footnotesize Cosine similarity ($S_C$) between features of adversarial samples and real samples}
    \caption{\textbf{Cosine similarity ($S_C$) between features of adversarial samples and real samples.} Adversarial sample features align more with those of the forget class both before and after adaptation. After adaptation, they align more with the target forget }
    \label{tab:cosine}
    \centering
    \vspace{3px}
    \resizebox{0.55\textwidth}{!}{%
    \begin{tabular}{lcc}
        \toprule
        Compared Features & $S_C$ Before Adaptation & $S_C$ After Adaptation \\
        \midrule
        Adversarial - Source Forget & \textbf{0.360} & \textbf{0.371} \\
        Adversarial - Source Retain & 0.326 & 0.315 \\
        Adversarial - Target Forget & \textbf{0.359} & \textbf{0.366} \\
        Adversarial - Target Retain & 0.335 & 0.307 \\
        \bottomrule
    \end{tabular}
    }
\end{table}

Adversarial samples are generated directly by the model through optimization in the input space. Visually, these samples differ significantly from real-samples (See \cref{fig:real_vs_adv}). However, the model confidently classifies these samples as a certain class.
In \cref{sec:method}, we mention that the model maximizes the probability of the adversarial samples belonging to the forget class. While the objective encourages adversarial samples to be classified as the forget class by maximizing its logit, no structural regularization is used to impose visual similarity of the adversarial samples to the forget class. In this section, we further analyze the features of adversarial samples to assess their similarity to the forget classes using cosine similarity. We experiment on a representative task from the DomainNet dataset: Clipart to Sketch.

\cref{tab:cosine} shows a 36\% cosine similarity of adversarial sample features with forget class features of both domains. This similarity increases by about 1\% after the domain adaptation process, showing that over time the adversarial samples evolve to approximate the forget class better. On the other hand, their similarity with retain class features starts at 32-33\% and decreases to 30-31\% after adaptation. These results support our claim that adversarial samples are representative of the forget classes in \cref{sec:method}.

For qualitative analysis, we also visualize the evolution of adversarial sample features using t-SNE, as shown in \cref{fig:tsne}. We plot centroids for source retain and forget class features along with target retain and forget centroids, and overlay the adversarial trajectory across iterations. The visualization shows that adversarial samples begin far from the forget-class features, then gradually move toward the source forget class before shifting further toward the target forget class. This illustrates how the adversarial optimization drives samples to mimic the features of the class to forget as the model evolves from the source to target domain.

\begin{figure*}[t]
    \centering
    \begin{subfigure}{0.48\linewidth}
        \centering
        \includegraphics[width=\linewidth]{fig/png/adv.png}
    \end{subfigure}
    \begin{subfigure}{0.48\linewidth}
        \centering
        \includegraphics[width=\linewidth]{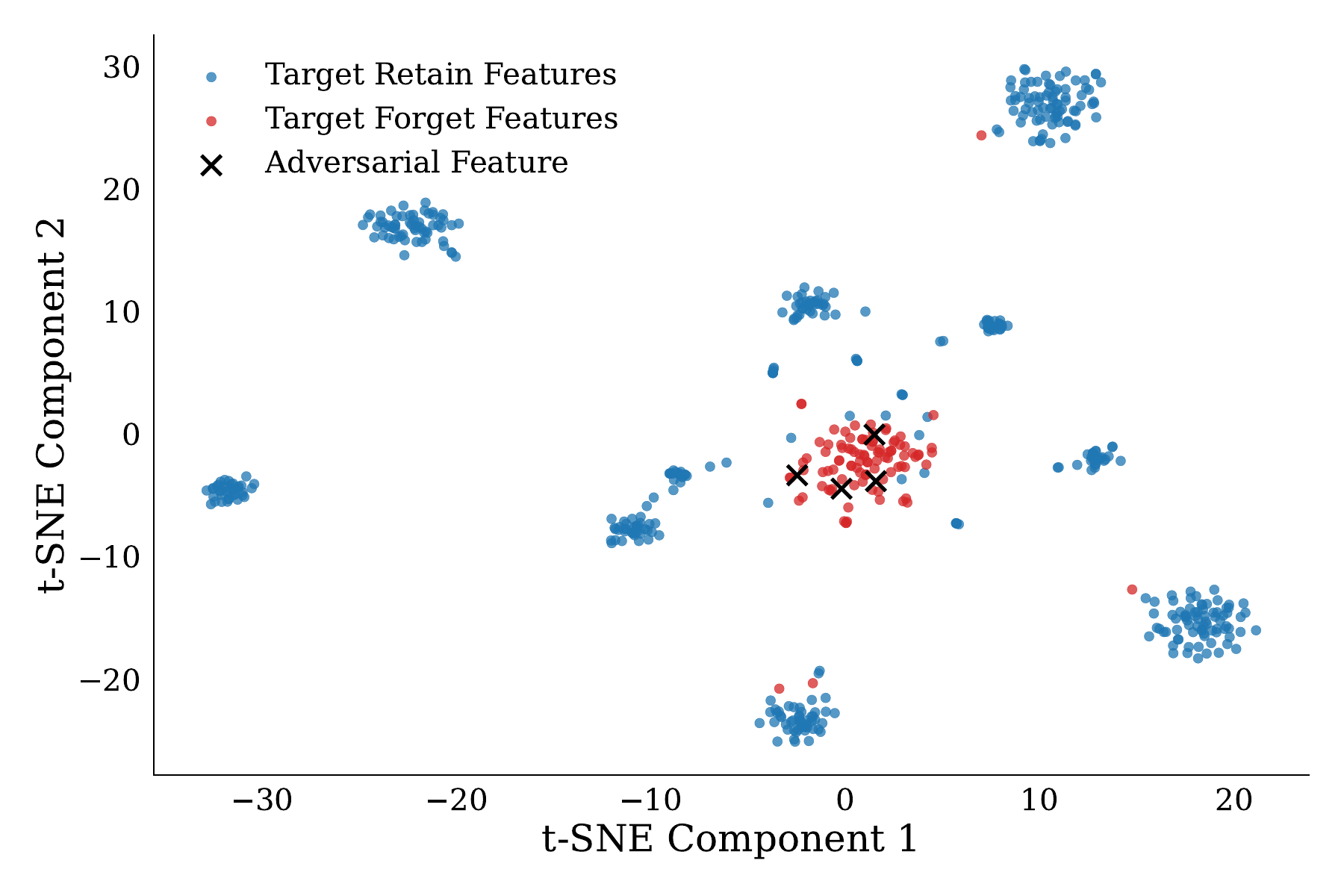}
    \end{subfigure}
    \scriptsize
    % \captionsetup{font=footnotesize}
    \caption{\textbf{t-SNE plots of adversarial, retain, and forget class samples.}
    Left: Evolution of an adversarial sample over time. Initially, the sample lies far from the forget-class centroids (closer to retain-class centroids). While being optimized on the source model, it moves closer to the source forget-class centroid, and over iterations gradually converges toward the target forget-class centroid. This shows that the samples evolve alongside the model to best fit the class to forget in the target domain. Right: Final adversarial samples compared to t-SNE embeddings of 9 randomly selected retain classes $\{29, 3, 44, 58, 36, 2, 59, 47, 14\}$ and the forget class $c_\mathcal{F}=1$. The final adversarial samples clearly align more closely with features from the forget class. Plotted for OfficeHome Art to Product.}
    \label{fig:tsne}
    \vspace{-8pt}
\end{figure*}

\begin{figure}[t]
    \centering
    \includegraphics[width=0.65\linewidth]{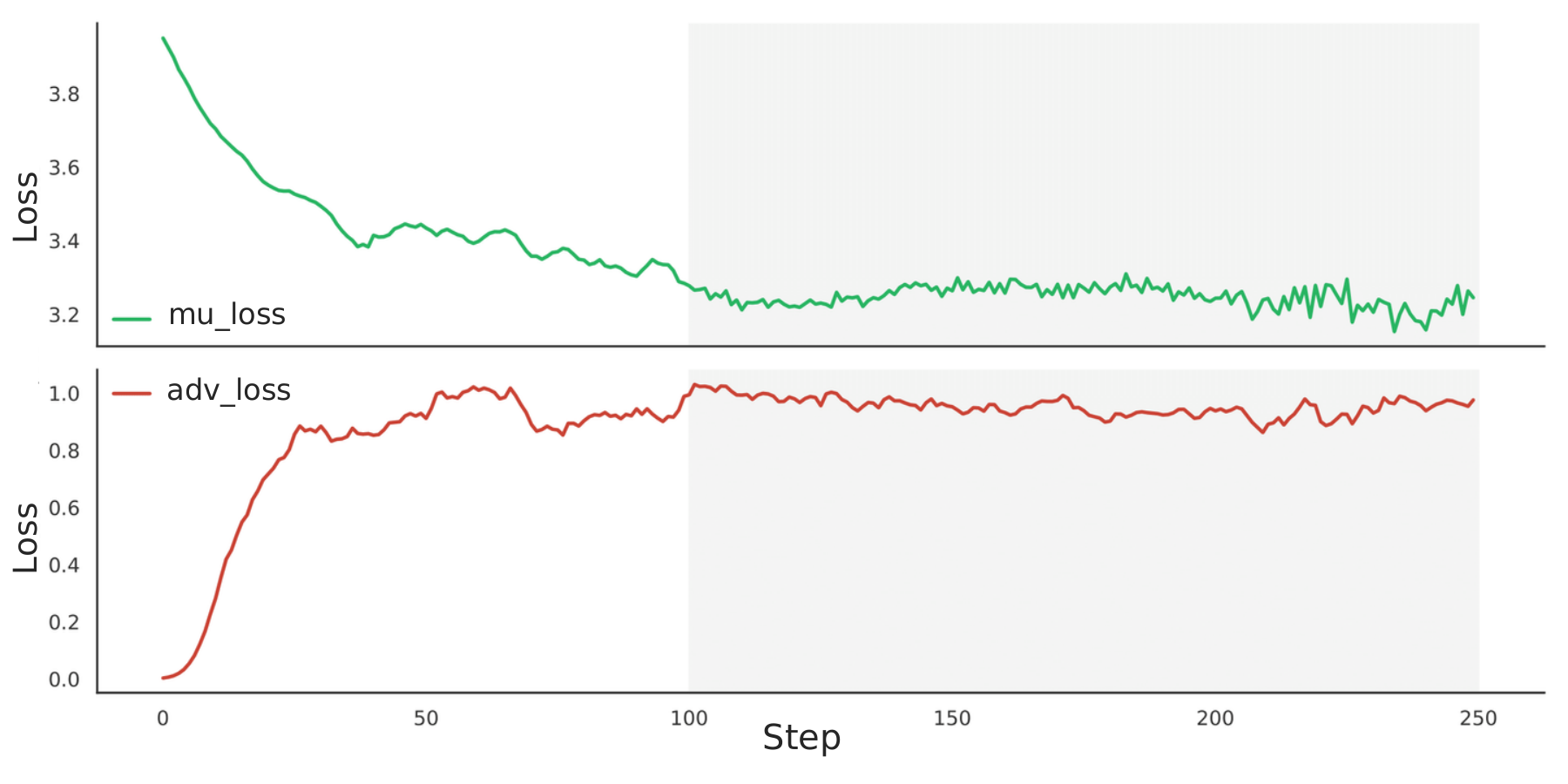}
    % \captionsetup{font=footnotesize}
    \caption{\textbf{MU Loss ($\mathcal{L}_\text{MU}$) and Adversarial Loss ($\mathcal{L_\text{ADV}}$).} Initially, $\mathcal{L}_\text{MU}$ conflicts with $\mathcal{L}_\text{ADV}$, but equilibrium is reached after $\sim 100$ steps. This demonstrates a measured approach to unlearning where the model gradually finds a new representation $\phi$ that aligns with both objectives.}%as the model does not immediately destroy performance on forget classes, instead finding a  all objectives.}
    \label{fig:vis-loss}
\vspace{-20pt}
    
\end{figure}
\begin{figure*}[t!]
    \centering
    \includegraphics[width=0.7\linewidth]{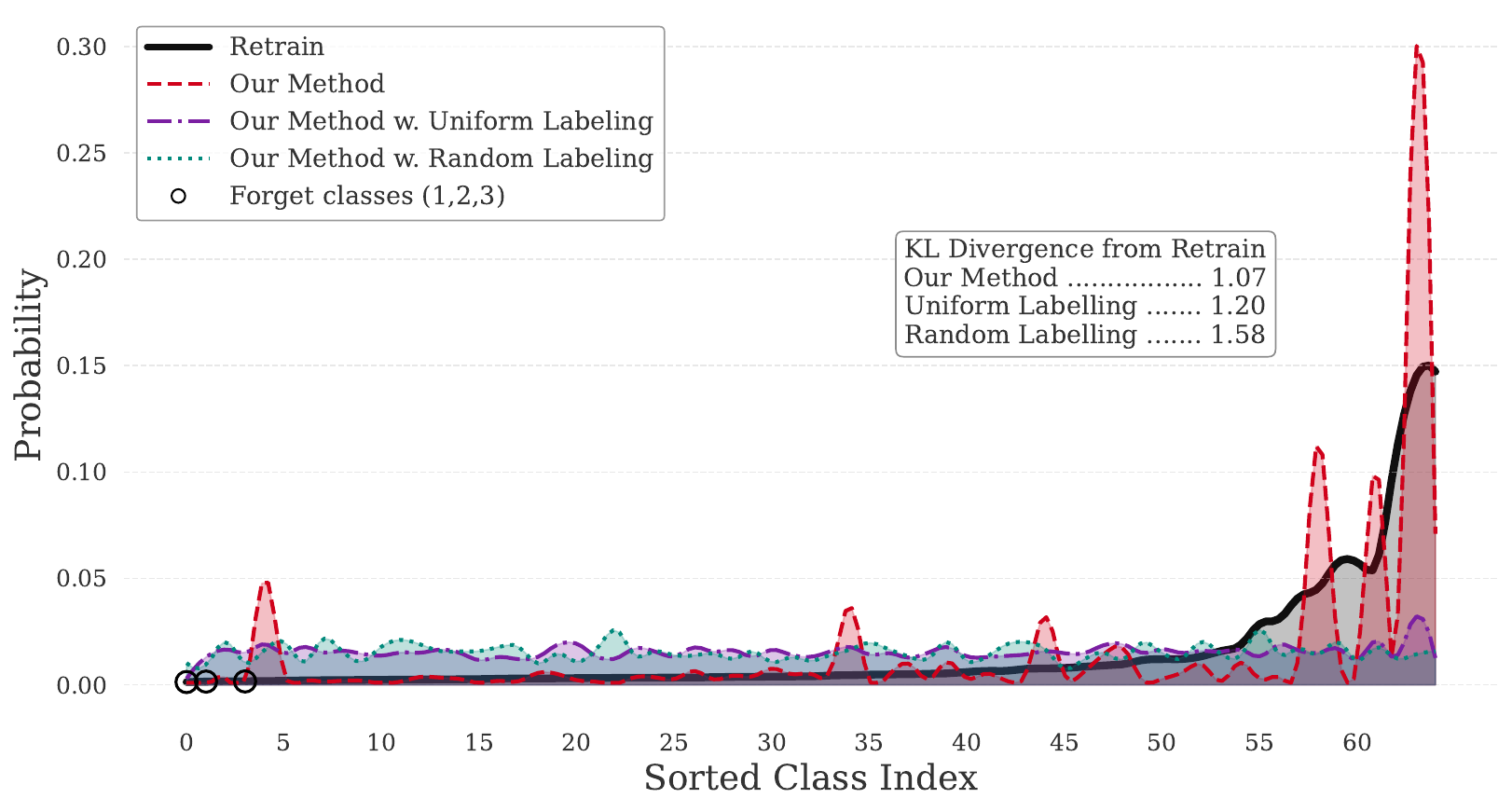}
    % \captionsetup{font=footnotesize}
    \caption{\textbf{Softmax outputs of the target forget set for the retrained model compared with our method using Rescaled, Uniform, and Random Labeling.} 
    After sorting classes by the retrained model’s outputs, our method with Rescaled Labeling best matches the retrained distribution, while Uniform and Random Labeling deviate more significantly, as also reflected in their higher KL Divergence values. Note that we use interpolation to make the distributions more clear.  }
    
    \label{fig:outputs-vis}
\end{figure*}

\subsection{Visualizing Unlearning and Adversarial Loss Terms} \label{subsec:visloss}

We show a plot of the Unlearning Loss ($\mathcal{L}_\text{MU}$) and Adversarial Loss ($\mathcal{L}_\text{ADV}$) in \cref{fig:vis-loss} for the OfficeHome Real-World-to-Art domain adaptation task. The figure reveals two key trends: firstly, the gradual decline in $\mathcal{L}_\text{MU}$ indicates a measured unlearning process, suggesting that the method does not aim to abruptly "erase" the forget classes (a process that still stores latent information about the forget classes). Instead, unlearning occurs adversarially, as evident from the mirrored relationship between $\mathcal{L}_\text{MU}$ and $\mathcal{L}_\text{ADV}$ where a decrease in one causes an increase in the other. These observations align with our discussion in \cref{sec: theory}, that is, a higher value of the gradients of $\mathcal{L}_\text{MU}$ for the retain classes gradually allows the unlearning of the forget classes during the training process. Together, these results demonstrate that our method performs thorough unlearning.

\subsection{Visualizing Outputs Under Different Labeling Strategies} \label{subsec:visoutputs}

We visualize the average softmax outputs over the entire target forget dataset $\mathcal{D}_f^\mathcal{T}$ for four models:  
(1) \textbf{Retrain}: a model retrained on the source data and adapted to the target domain,  
(2) \textbf{Our Algorithm \ref{alg:minimax}}: a model trained with our procedure using Rescaled Labeling,  
(3) \textbf{Uniform Labeling}: a variant where the rescaled labels $\hat{y}$ are replaced with a uniform distribution over the retain classes $c_\mathcal{R}$, and  
(4) \textbf{Random Labeling}: a variant where $\hat{y}$ is replaced with a randomly selected retain class.  

Figure~\ref{fig:outputs-vis} shows the sorted softmax outputs of all four models with respect to the retrained model’s ordering. By sorting classes according to the retrained model’s output distribution, the visualization highlights how closely each method matches the retrained behavior. Among the three variants of our method, the Rescaled Labeling strategy most closely follows the retrained curve, while Uniform and Random Labeling show more pronounced deviations. This further illustrates that Rescaled Labeling provides the best approximation to the output distribution of the gold-standard Retrain model.

% \subsection{Source Domain Performance of Methods}
% \label{subsec:source-domain-performance}
% \input{tab/source_perf_dmnt}

% In \cref{tab:source-perf}, we see the performance of SCADA-UL methods on the source domain dataset tested on multi class unlearning ($c_\mathcal{F}=\{1,2,3\}$) averaged over all 7 tasks on DomainNet. The results show largely similar trends as the target domain performance of our method where it achieves best results w.r.t unlearn score and forget accuracy metrics.

\subsection{Failure Case Analysis: UC-SCADA-UL} \label{subsec:failure_case}

In \cref{tab:uc-SCADA-dmnt}, specifically in the Sketch $\rightarrow$ Painting task, we observe that our method achieves a forget accuracy of 29.6\%, which is an improvement over the original model's accuracy of 67.6\%, but still far from the ideal 0\% forget accuracy achieved by retraining.

We identify two key reasons for this gap.

\begin{enumerate}
    \item \textbf{Low $\gamma$ values for retain classes due to domain shift or inherent difficulty.} For example, classes like Compass, Cow, and Pencil exhibit 0\% accuracy on the target domain and correspondingly low $\gamma$ values (see \cref{tab:gamma-val}). This misleads our algorithm into detecting them as forget classes even though they are actually retain classes.
    \item \textbf{Semantic or visual similarity between retain and forget classes.} Forget classes such as Great Wall of China receive higher $\gamma$ values due to contributions from semantically similar classes like castle and streetlight (see \cref{tab:cont-class}), making them harder to identify correctly as forget classes
\end{enumerate}

\begin{table*}[t]
% \captionsetup{font=footnotesize}
\makebox[\textwidth][c]{
  \hspace{\fill}
  \begin{minipage}[t]{0.3\textwidth}
    \caption{ \textbf{Target accuracy of lowest $\gamma$ value classes} (domain shift evident by low target accuracy)}
    \label{tab:gamma-val}
    \resizebox{\textwidth}{!}{%
      \begin{tabular}{lcc}
    \toprule
    Class Name  & Target Acc  &   $\gamma$ Value \\
    \midrule
    Compass     &   0.0                     &   0.010 \\
    Cow         &   0.0                     &   0.012 \\
    Pencil      &   0.0                     &   0.013 \\
    Peas        &   2.5                     &   0.018 \\
    Cell Phone  &   0.0                     &   0.022 \\
    \bottomrule
  \end{tabular}
    }
  \end{minipage}
  \hspace{\fill}
  \begin{minipage}[t]{0.5\textwidth}
    \caption{\textbf{Top contributing classes for $\gamma$ value of forget classes.} (semantic similarity is observed)}
    \label{tab:cont-class}
    \resizebox{\textwidth}{!}{%
      \begin{tabular}{lccc}
        \toprule
        Forget Class        & Candidate 1   & Candidate 2   & Candidate 3 \\
        \midrule
        Great Wall of China & castle        & train         & streetlight \\
        Aircraft Carrier    & submarine     & cruise\_ship  & helicopter \\
        Alarm Clock         & compass       & watermelon    & dog \\
        \bottomrule
\end{tabular}
    }
  \end{minipage}
  \hspace{\fill}
}
\end{table*}

\subsection{Additional Ablation Studies}
\label{subsec:additional-ablation}

\begin{wraptable}{r}{5cm}
\centering
\vspace{-12pt}
\captionsetup{font=footnotesize}
\caption{\footnotesize Initializing adv samples}
\label{tab:ab-init}
\scriptsize
\vspace{-5pt}
\begin{tabular}{ccc}
\toprule
Initialize & $A_{\mathcal{D}_r^\mathcal{T}} \uparrow$  & $A_{\mathcal{D}_f^\mathcal{T}} \downarrow$ \\ \midrule
\rowcolor{green!40!white}
\ding{51} & ${75.1_{\color{gray}{\pm1.3}}}$ & ${0.0_{\color{gray}{\pm0.0}}}$ \\
\ding{55} & ${75.0_{\color{gray}{\pm1.2}}}$ & ${15.2_{\color{gray}{\pm4.1}}}$ \\
\bottomrule
\end{tabular}
\vspace{-15pt}
\end{wraptable}
We performed additional ablation studies on the initialization of adversarial samples, loss trade-off parameter $\alpha$ and the number of adversarial samples, which are reported below. All experiments measure multi-class SCADA-UL performance across 3 trials on the OfficeHome dataset, reporting retain accuracy \( (A_{\mathcal{D}_r^\mathcal{T}})\) and forget accuracy \((A_{\mathcal{D}_{\smash{f}}^\mathcal{T}})\) averaged across 12 tasks.

\begin{wraptable}{r}{5cm}
\centering
\setlength{\tabcolsep}{4pt} 
\scriptsize
\vspace{-12pt}
\captionsetup{font=footnotesize}
\caption{Effect of varying Loss Trade-off ($\alpha$) hyperparameter study}
\label{tab:ab-alpha}
\vspace{-5pt}
\begin{tabular}{ccc}
\toprule
$\alpha$ & $A_{\mathcal{D}_r^\mathcal{T}} \uparrow$  & $A_{\mathcal{D}_f^\mathcal{T}} \downarrow$ \\ \midrule
$1.0$ & ${75.0_{\color{gray}{\pm1.2}}}$ & ${15.2_{\color{gray}{\pm4.1}}}$ \\
$5.0$ & ${75.2_{\color{gray}{\pm1.1}}}$ & ${2.6_{\color{gray}{\pm2.0}}}$ \\
\rowcolor{green!40!white}
$10.0$ & ${75.1_{\color{gray}{\pm1.3}}}$ & ${0.0_{\color{gray}{\pm0.0}}}$ \\
$20.0$ & ${75.2_{\color{gray}{\pm1.1}}}$ & ${0.0_{\color{gray}{\pm0.0}}}$ \\
\bottomrule
\end{tabular}
%\vspace{-4pt}
\vspace{-14pt}
\end{wraptable}
\textbf{Initialization of Adversarial Samples.} 
\cref{tab:ab-init} shows that initializing samples by a complete minimization of $\mathcal{L}_\text{ADV}$ on the source model is crucial for effective unlearning. This step enables the samples to be good representations of the forget classes, which are subsequently refined via gradient updates.

\textbf{Loss Trade-off (\(\alpha\)).} As seen in \cref{tab:ab-alpha}, balancing \(\mathcal{L}_\text{SFDA}\) and \(\mathcal{L}_\text{MU}\) requires \(\alpha > 5.0\) to ensure thorough unlearning of forget classes, while values  10.0 and 20.0 yield comparable performance. This behavior can be explained by the distinct optimization objectives of the two losses, which stabilize after initial iterations (see \cref{fig:vis-loss}). 
\begin{wraptable}{r}{5cm}
\scriptsize
\centering
\vspace{-22pt}
\captionsetup{font=footnotesize}
\caption{\footnotesize Study on \# adversarial samples used during training}
\vspace{-5pt}
\label{tab:ab-nsamples}
\begin{tabular}{ccc}
\toprule
$N_\text{adv}$ & $A_{\mathcal{D}_r^\mathcal{T}} \uparrow$  & $A_{\mathcal{D}_f^\mathcal{T}} \downarrow$ \\ \midrule
$1$ & ${70.9_{\color{gray}{\pm3.6}}}$ & ${0.8_{\color{gray}{\pm1.4}}}$ \\
$2$ & ${72.9_{\color{gray}{\pm3.7}}}$ & ${0.0_{\color{gray}{\pm0.0}}}$ \\
\rowcolor{green!40!white}
$4$ & ${75.1_{\color{gray}{\pm1.3}}}$ & ${0.0_{\color{gray}{\pm0.0}}}$ \\
$8$ & ${75.2_{\color{gray}{\pm1.2}}}$ & ${0.0_{\color{gray}{\pm0.0}}}$ \\
$16$ & ${75.3_{\color{gray}{\pm1.1}}}$ & ${0.0_{\color{gray}{\pm0.0}}}$ \\
\bottomrule
\end{tabular}
\vspace{-10pt}
\end{wraptable}

% \begin{wraptable}{r}{5cm}
% \centering
% \setlength{\tabcolsep}{4pt} 
% \scriptsize
% \vspace{-12pt}
% \captionsetup{font=footnotesize}
% \caption{Effect of varying Loss Trade-off ($\alpha$) hyperparameter study}
% \label{tab:ab-alpha}
% \vspace{-5pt}
% \begin{tabular}{ccc}
% \toprule
% $\alpha$ & $A_{\mathcal{D}_r^\mathcal{T}} \uparrow$  & $A_{\mathcal{D}_f^\mathcal{T}} \downarrow$ \\ \midrule
% $1.0$ & ${75.0_{\color{gray}{\pm1.2}}}$ & ${15.2_{\color{gray}{\pm4.1}}}$ \\
% $5.0$ & ${75.2_{\color{gray}{\pm1.1}}}$ & ${2.6_{\color{gray}{\pm2.0}}}$ \\
% \rowcolor{green!40!white}
% $10.0$ & ${75.1_{\color{gray}{\pm1.3}}}$ & ${0.0_{\color{gray}{\pm0.0}}}$ \\
% $20.0$ & ${75.2_{\color{gray}{\pm1.1}}}$ & ${0.0_{\color{gray}{\pm0.0}}}$ \\
% \bottomrule
% \end{tabular}
% %\vspace{-4pt}
% \vspace{-14pt}
% \end{wraptable}
 Below the threshold (5.0–10.0), \(\mathcal{L}_\text{SFDA}\) marginally dominates \(\mathcal{L}_\text{MU}\), leading to suboptimal unlearning performance.

\textbf{Number of Adversarial Samples.} \cref{tab:ab-nsamples} shows that using very few adversarial samples, even 2 in this case, results in good unlearning performance. As the number of adversarial samples increases, retain accuracy improves, reaching close to its maximum value at 4 samples (which we use for our experiments in the main paper). Beyond this point, increasing the number of samples further does not lead to considerable improvement.

\begin{wraptable}{r}{5cm}
\scriptsize
\centering
\vspace{-6pt}
\captionsetup{font=footnotesize}
\caption{\footnotesize Study on number of train epochs}
\vspace{-5pt}
\label{tab:ab-epochs}
\begin{tabular}{ccc}
\toprule
Epochs & $A_{\mathcal{D}_r^\mathcal{T}} \uparrow$  & $A_{\mathcal{D}_f^\mathcal{T}} \downarrow$ \\ \midrule
$1$ & ${75.6_{\color{gray}{\pm1.4}}}$ & ${18.2_{\color{gray}{\pm8.1}}}$ \\
\rowcolor{green!40!white} 
$5$ & ${75.1_{\color{gray}{\pm1.3}}}$ & ${0.0_{\color{gray}{\pm0.0}}}$ \\
$10$ & ${74.9_{\color{gray}{\pm1.2}}}$ & ${0.4_{\color{gray}{\pm0.3}}}$ \\
\bottomrule
\end{tabular}
\vspace{-10pt}
\end{wraptable}
\textbf{Number of Training Epochs.} From \cref{tab:ab-epochs}, we observe that using fewer epochs (e.g., 1) leads to incomplete forgetting, as seen by the higher forget accuracy. On the other hand, a larger number of epochs (e.g., 10) results in a slight drop in retain accuracy, likely due to continued application of $\mathcal{L}_{MU}$ negatively impacting the retained classes after the forget classes have already been unlearned. We use 5 epochs in our main experiments, as it provides a favorable trade-off, achieving complete forgetting while maintaining high retain accuracy.

\subsection{Discussion of Assumptions in Our Method}
\label{subsec:assump}
Our method follows that (i) datasets with disjoint label spaces are conditionally independent given the model weights and (ii) source datasets can be approximated by their corresponding trained models. The use of (i) is explained and justified with an example at the end of Section~\ref{subsec:ao-scada-ul}. We further elaborate on (ii) here. For the approximation used in (ii), we follow such approximations used in Elastic Weight Consolidation~\cite{aich2021elastic} which employs a (diagonal) Laplace approximation to estimate posteriors across sequential tasks, and in online learning scenarios~\cite{ritter2018online}. This approximation follows by assuming the posterior $p(w\mid \mathcal{D^S})$ is approximately Gaussian and centered at the MLE estimate $w^\mathcal{S}$. i.e., $p(w\mid \mathcal{D^S})\approx\mathcal{N}(w^\mathcal{S}, \Sigma)$. By Bayesian inference, $p(w\mid \mathcal{D^S}, \mathcal{D}_r^\mathcal{T})\propto p(\mathcal{D}_r^\mathcal{T}\mid w)\cdot p(w\mid \mathcal{D^S})$. Substituting the Gaussian approximation yields $p(w\mid \mathcal{D^S}, \mathcal{D}_r^\mathcal{T})\propto p(\mathcal{D}_r^\mathcal{T}\mid w)\cdot \mathcal{N}(w^\mathcal{S}, \Sigma)\approx p(w\mid w^\mathcal{S}, \mathcal{D}_r^\mathcal{T})$. Similarly, we obtain $p(w\mid \mathcal{D}_r^\mathcal{S}, \mathcal{D}_r^\mathcal{T})\approx p(w\mid w_r^\mathcal{S}, \mathcal{D}_r^\mathcal{T})$. 

\section{Additional Experimental Results}
\label{sec:results-apd}

\subsection{Runtime Analysis of Methods} \label{subsec:time_cons}

\Cref{tab:time-taken} presents the runtime comparison of all methods for unlearning and adaptation to the target domain. While our method does not consistently achieve the lowest runtime across all datasets, its time cost remains significantly lower than that of retraining. We especially note that on the larger dataset setting (DomainNet), our method achieves the lowest training time, indicating its potential efficiency in large-scale settings. Additionally, the inference time is identical across all methods, as no extra computation is required during model inference.

\begin{table*}[t]
% \captionsetup{font=footnotesize}
\makebox[\textwidth][c]{
  \hspace{\fill}
  \begin{minipage}[t]{0.425\textwidth}
    \caption{Training Time for Each Method (in seconds)}
    \label{tab:time-taken}
    \resizebox{\textwidth}{!}{%
      \begin{tabular}{l@{\hskip 5pt}c@{\hskip 5pt}c@{\hskip 5pt}c}
    \toprule
    Method & OfficeHome & Office31 & DomainNet \vspace{2px} \\ 
    \midrule
    Original (SF(DA)$^2$~\cite{hwang2024sf})  & \val{307.9}{0.6} & \val{239.9}{2.0} & \val{623.6}{5.3} \\
    Retrain    & \val{696.7}{1.0} & \val{665.8}{2.7} & \val{1109.3}{1.9} \\ \midrule
    Finetune   & \val{616.2}{2.0} & \val{476.7}{0.7} & \val{1239.0}{2.3} \\
    UNSIR~\cite{tarun2023fast}       & \val{\underline{364.0}}{3.3} & \val{\underline{296.1}}{1.9} & \val{710.3}{2.6} \\
    ZSMU~\cite{chundawat2023zero}       & \val{403.9}{0.2} & \val{334.5}{1.3} & \val{738.2}{3.2} \\
    Lipschitz~\cite{foster2024zero}  & \val{484.8}{2.8} & \val{419.4}{0.4} & \val{789.5}{0.5} \\
    Nabla Tau~\cite{trippa2024tau}  & \val{370.3}{2.6} & \val{301.6}{0.5} & \val{757.1}{4.3} \\ 
    Unlearned(+)~\cite{Ahmed_2025_CVPR}  & \val{1073.8}{514} & \val{396.1}{59.} & \val{3017}{1319} \\ \midrule
    PADA~\cite{cao2018partial}       & \val{\textbf{319.3}}{1.1} & \val{\textbf{253.6}}{2.1} & \val{\underline{691.}2}{1.0} \\
    SHOT~\cite{liang2020we}       & \val{560.1}{0.3} & \val{538.8}{1.3} & \val{730.6}{1.6} \\ \midrule
    \rowcolor{green!40!white}    
    Ours       & \val{382.2}{5.4} & \val{319.6}{2.0} & \val{\textbf{691.0}}{2.7} \\
    \bottomrule
  \end{tabular}
    }
  \end{minipage}
  \hspace{\fill}
  \begin{minipage}[t]{0.4\textwidth}
    \caption{SCADA-UL performance on Land-use Classification: UCMerced $\rightarrow$ RSSCN7}
    \label{tab:land-use}
    \resizebox{\textwidth}{!}{%
      \begin{tabular}{lccc}
        \toprule
        Method & $A_{\mathcal{D}_r^\mathcal{T}} \uparrow$ & $A_{\mathcal{D}_f^\mathcal{T}} \downarrow$ & Score $\uparrow$  \\
        \midrule 
        Original (SF(DA)$^2$~\cite{hwang2024sf})  & \val{76.2}{1.3} & \val{24.3}{28.} & \val{0.64}{0.1} \\
        Retrain    & \val{76.2}{3.8} & \val{0.0}{0.0} & \val{0.76}{0.0} \\ \midrule
        Finetune   & \val{\textbf{76.2}}{1.0} & \val{35.3}{23.} & \val{0.57}{0.1} \\
        UNSIR~\cite{tarun2023fast}       & \val{63.6}{10.} & \val{\textbf{0.0}}{0.0} & \val{0.64}{0.1} \\
        ZSMU~\cite{chundawat2023zero}       & \val{\textbf{76.2}}{1.8} & \val{\textbf{0.0}}{0.0} & \val{\textbf{0.76}}{0.0} \\
        Lipschitz~\cite{foster2024zero}  & \val{\underline{75.9}}{1.3} & \val{18.3}{24.} & \val{\underline{0.66}}{0.1} \\
        Nabla Tau~\cite{trippa2024tau}  & \val{62.3}{2.5} & \val{\textbf{0.0}}{0.0} & \val{0.62}{0.0} \\ 
    Unlearned(+)~\cite{Ahmed_2025_CVPR}  & \val{63.4}{0.8} & \val{0.9}{1.3} & \val{0.63}{0.0} \\ \midrule
        PADA~\cite{cao2018partial}       & \val{64.4}{1.4} & \val{\underline{2.7}}{3.8} & \val{0.63}{0.0} \\
        SHOT~\cite{liang2020we}       & \val{65.2}{1.9} & \val{14.5}{0.9} & \val{0.57}{0.0} \\ \midrule
        \rowcolor{green!40!white}
        Ours       & \val{\underline{75.9}}{1.0} & \val{\textbf{0.0}}{0.0} & \val{\textbf{0.76}}{0.0} \\
        \bottomrule
\end{tabular}
    }
  \end{minipage}
  \hspace{\fill}
}
\end{table*}

\subsection{Experiments on Land-use Classification Dataset} \label{subsec:scenes}

In \cref{tab:land-use}, we studied our approach on a land-use classification domain adaptation benchmark UC Merced~\cite{yang2010bag} $\rightarrow$ RSSCN7~\cite{zou2015deep}, similar to \cite{song2019domain}. Privacy is a critical concern in such settings as certain categories of scenes (for e.g., government facilities) present in the source domain trained model must not be adapted to the target domain.~\Cref{tab:land-use} indicates our method achieves good results on this real-world dataset.

\subsection{SCADA Unlearning} \label{subsec:apd_scada}

\Cref{tab:sc-SCADA-ofhm,tab:sc-SCADA-dmnt,tab:sc-SCADA-of31} present results for single-class SCADA unlearning on the OfficeHome, DomainNet and Office31 datasets respectively. Multi-class SCADA unlearning results are presented in \cref{tab:mc-SCADA-ofhm,tab:mc-SCADA-dmnt,tab:mc-SCADA-of31} tested on $\mathcal{C_F}=\{1,2,3\}$. This is an extended version of \cref{tab:mc-SCADA} presented in the main paper, showing results for all the tasks within each dataset. Our method outperforms baselines w.r.t. all the metrics by achieving accuracies and unlearn score close to that of the retrained model in both single-class and multi-class unlearning settings.

\subsection{UC-SCADA Unlearning} \label{subsec:apd_uc_scada}

\Cref{tab:uc-SCADA-ofhm,tab:uc-SCADA-dmnt,tab:uc-SCADA-of31} show the results in the UC-SCADA-UL setting (Def~\ref{def:uc-SCADA}) with $\mathcal{C_F}=\{1,2,3\}$, but these source-exclusive classes unknown to the model. Results show that our method outperforms the baselines w.r.t. the unlearn score metric on the OfficeHome and Office31 datasets. It maybe noted that finetuning is typically found to perform the best w.r.t. $\smash{A_{\mathcal{D}_r^\mathcal{T}}}$ but its performance w.r.t. $\smash{A_{\mathcal{D}_f^\mathcal{T}}}$ is poor. Similarly, w.r.t. $\smash{A_{\mathcal{D}_f^\mathcal{T}}}$, UNSIR performs well in the UC-SCADA-UL setting, but it performs poorly w.r.t. $\smash{A_{\mathcal{D}_r^\mathcal{T}}}$. On the DomainNet dataset, our method struggles with $\smash{A_{\mathcal{D}_f^\mathcal{T}}}$, which is likely due to the inaccuracy in identifying the forget classes reliably using the $\gamma$ parameter. We elaborate more on this in \cref{sec:limitations}.
% Arnav: It might be better to remove the final two lines ("On the.. "), we have already mentioned this as a limitation, this could bring unnessasary eyes to it.

\subsection{C-SCADA Unlearning} \label{subsec:apd_c_scada}

\Cref{tab:c-dasec-of31,tab:c-SCADA-dmnt,tab:c-SCADA-ofhm} show the results in the C-SCADA-UL setting (Def~\ref{def:c-SCADA}) where the source-exclusive classes are revealed over multiple time steps. For the experiments, we use $\smash{\mathcal{C_F}^1=\{1,2\}}$, $\smash{\mathcal{C_F}^2=\{3,4\}}$, $\smash{\mathcal{C_F}^3=\{5,6\}}$. The accuracies presented are cumulative over all forget classes until the current time step (for e.g., $\smash{A_{\mathcal{D}_f^\mathcal{T}}}$ in T2 is the forget accuracy over classes $\smash{\mathcal{C_F}^1 \cup \mathcal{C_F}^2 = \{1,2,3,4\}}$). This allows us to evaluate how effectively the model forgets the newly designated class sets while maintaining unlearning of the older forget classes. The results show that our method consistently achieves thorough unlearning of the new class sets at each step, outperforming the baselines w.r.t. the unlearn score metric.

\subsection{Use of Other SFDA Loss Functions} \label{subsec:apd_add_loss}

% \Cref{tab:shot-SCADA-ofhm} shows Multi-class SCADA Unlearning performance of seven baselines: Original, Retrain, Finetune, UNSIR~\cite{tarun2023fast}, ZSMU~\cite{chundawat2023zero}, Lipschitz~\cite{foster2024zero}, Nabla Tau~\cite{trippa2024tau}, and our method using the loss terms from SHOT~\cite{liang2020we} as $\mathcal{L}_\text{SFDA}$. Similar to other multi-class SCADA unlearning experiments, the forget classes are $\mathcal{C_F}=\{1,2,3\}$. The results indicate that our method outperforms baseline approaches even when different SFDA loss terms are used, demonstrating its robustness and effectiveness.

\Cref{tab:shot-ucon-sfda} shows results of experiments with two other SFDA methods: SHOT~\cite{liang2020we} and UCon-SFDA~\cite{ucon_sfda}. The target domain performance while using just the SFDA method \textit{(Original)} and the target domain performance while using this method as $\mathcal{L}_\text{SFDA}$ in our proposed method are provided in terms of percentage improvement over the corresponding Unlearn Score metric of the previous method going from top to bottom. For example, Rows 5, 6 show our method achieves a 6.1\% improvement in the Score while using UCon-SFDA as $\mathcal{L}_\text{SFDA}$ instead of (SF(DA)$^2$, when UCon-SFDA shows a 15.7\% improvement in target domain performance over (SF(DA)$^2$. This indicates that while our method works well with different SFDA losses, the performance of our method improves with the SFDA method used. Table~\ref{tab:shot-ucon-sfda} reports the metrics over all 7 source, target pairs in DomainNet-126 for a single forget class.

\begin{table}[H]
  % \captionsetup{font=footnotesize}
  \caption{\textbf{Improved SFDA methods improve our method as well.} We observe that as the underlying SFDA method achieves better performance, our method also improves. This shows the adaptability of our approach to new SFDA loss terms and its potential to further benefit as future SFDA methods advance.}
  \label{tab:shot-ucon-sfda}
  \vspace{3px}
  \centering
  \resizebox{0.95\textwidth}{!}{%
    \begin{tabular}{l l l c c c c c c c c c}
      \toprule
      SFDA Loss & Method & Metric
      & s $\rightarrow$ p & c $\rightarrow$ s & p $\rightarrow$ c & p $\rightarrow$ r
      & r $\rightarrow$ s & r $\rightarrow$ c & r $\rightarrow$ p & Average & \shortstack{\% Improved} \\
      \midrule

      % ================== SHOT ==================
      \multirow{6}{*}{SHOT~\cite{liang2020we}}
        & \multirow{3}{*}{Original}
          & $A_{\mathcal{D}_r^\mathcal{T}} \uparrow$ & \val{70.1}{1.1} & \val{57.4}{1.9} & \val{61.1}{1.0} & \val{83.7}{0.4} & \val{60.6}{3.4} & \val{69.1}{1.0} & \val{75.5}{0.6} & \val{68.2}{1.3} &  \\
        &  & $A_{\mathcal{D}_f^\mathcal{T}} \downarrow$ & \val{88.7}{0.9} & \val{47.0}{20.6} & \val{57.8}{2.5} & \val{86.7}{0.1} & \val{27.7}{4.8} & \val{41.8}{0.6} & \val{43.1}{2.3} & \val{56.1}{4.5} & - \\
        &  & Score $\uparrow$ & \val{0.37}{0.0} & \val{0.39}{0.0} & \val{0.39}{0.0} & \val{0.45}{0.0} & \val{0.47}{0.0} & \val{0.49}{0.0} & \val{0.53}{0.0} & \val{0.44}{0.0} &  \\
      \cmidrule(lr){2-12}
        & \multirow{3}{*}{Our Method}
          & $A_{\mathcal{D}_r^\mathcal{T}} \uparrow$ & \val{56.0}{1.1} & \val{36.6}{4.0} & \val{42.7}{0.2} & \val{77.9}{5.5} & \val{27.1}{5.3} & \val{49.6}{0.8} & \val{66.4}{5.2} & \val{50.9}{3.2} &  \\
        &  & $A_{\mathcal{D}_f^\mathcal{T}} \downarrow$ & \val{0.0}{0.0} & \val{0.0}{0.0} & \val{0.0}{0.0} & \val{0.0}{0.0} & \val{0.0}{0.0} & \val{0.0}{0.0} & \val{0.0}{0.0} & \val{0.0}{0.0} & - \\
        &  & Score $\uparrow$ & \val{\underline{0.56}}{0.0} & \val{0.37}{0.0} & \val{0.43}{0.0} & \val{\underline{0.78}}{0.1} & \val{0.27}{0.1} & \val{0.50}{0.0} & \val{0.66}{0.1} & \val{0.51}{0.0} &  \\
      \midrule

      % ================== SF(DA)^2 ==================
      \multirow{6}{*}{SF(DA)$^{2}$~\cite{hwang2024sf}}
        & \multirow{3}{*}{Original}
          & $A_{\mathcal{D}_r^\mathcal{T}} \uparrow$ & \val{71.3}{0.3} & \val{66.6}{0.4} & \val{61.7}{0.8} & \val{78.3}{0.2} & \val{55.9}{1.5} & \val{65.1}{0.9} & \val{75.0}{0.4} & \val{67.7}{0.6}  \\
        &  & $A_{\mathcal{D}_f^\mathcal{T}} \downarrow$ & \val{77.6}{5.7} & \val{36.9}{6.9} & \val{35.6}{1.3} & \val{74.1}{12.1} & \val{3.2}{0.4} & \val{17.5}{1.0} & \val{6.1}{0.3} & \val{35.9}{4.0} & +15.9\% \\
        &  & Score $\uparrow$ & \val{0.40}{0.0} & \val{0.48}{0.0} & \val{0.46}{0.0} & \val{0.45}{0.0} & \val{0.54}{0.0} & \val{0.55}{0.0} & \val{0.71}{0.0} & \val{0.51}{0.0} &  \\
      \cmidrule(lr){2-12}
        & \multirow{3}{*}{Our Method}
      %     & $A_{\mathcal{D}_r^\mathcal{T}} \uparrow$ & \val{67.3}{1.4} & \val{63.3}{0.7} & \val{60.8}{1.8} & \val{77.6}{0.2} & \val{52.4}{2.3} & \val{63.3}{0.5} & \val{72.7}{0.8} & \val{65.3}{1.1}  \\
      %   &  & $A_{\mathcal{D}_f^\mathcal{T}} \downarrow$ & \val{29.6}{9.8} & \val{36.8}{3.1} & \val{36.3}{7.1} & \val{28.7}{5.9} & \val{4.5}{5.4} & \val{0.7}{1.0} & \val{3.6}{5.1} & \val{20.1}{5.3} & +7.8\%  \\
      %   &  & Score $\uparrow$ & \val{0.52}{0.0} & \val{\underline{0.46}}{0.0} & \val{\underline{0.45}}{0.0} & \val{0.60}{0.0} & \val{\underline{0.50}}{0.0} & \val{\underline{0.63}}{0.0} & \val{\underline{0.70}}{0.0} & \val{\underline{0.55}}{0.0} & \\
      % \midrule

      & $A_{\mathcal{D}_r^\mathcal{T}} \uparrow$
      & \val{67.8}{2.3} & \val{62.7}{1.0} & \val{58.9}{1.0} & \val{77.0}{0.4} & \val{54.0}{1.7} & \val{62.1}{0.8} & \val{{74.0}}{1.0} & \val{65.2}{1.1} \\
      % \rowcolor{green!40!white}
      & & $A_{\mathcal{D}_f^\mathcal{T}} \downarrow$
      & \val{0.0}{0.0} & \val{0.0}{0.0} & \val{0.0}{0.0} & \val{0.0}{0.0} & \val{0.0}{0.0} & \val{0.0}{0.0} & \val{0.0}{0.0} & \val{0.0}{0.0} & +27.4\% \\
      % \rowcolor{green!40!white}
      & & Score $\uparrow$
      & \val{0.68}{0.0} & \val{0.63}{0.0} & \val{0.59}{0.0} & \val{0.77}{0.0} & \val{{0.54}}{0.0} & \val{0.62}{0.0} & \val{0.74}{0.0} & \val{0.65}{0.0} \\ \bottomrule

      % ================== UCon-SFDA ==================
      \multirow{6}{*}{UCon\mbox{-}SFDA~\cite{ucon_sfda}}
        & \multirow{3}{*}{Original}
          & $A_{\mathcal{D}_r^\mathcal{T}} \uparrow$ & \val{80.1}{0.8} & \val{73.9}{0.1} & \val{77.7}{1.2} & \val{88.6}{0.0} & \val{73.7}{0.2} & \val{77.3}{1.2} & \val{82.1}{0.1} & \val{79.0}{0.5} &  \\
        &  & $A_{\mathcal{D}_f^\mathcal{T}} \downarrow$ & \val{78.7}{14.2} & \val{52.7}{4.8} & \val{56.5}{1.8} & \val{21.8}{10.0} & \val{14.5}{2.4} & \val{28.9}{4.2} & \val{15.7}{0.0} & \val{38.4}{5.4} &  +15.7\% \\
        &  & Score $\uparrow$ & \val{0.45}{0.0} & \val{0.48}{0.0} & \val{0.50}{0.0} & \val{0.73}{0.1} & \val{0.64}{0.0} & \val{0.60}{0.0} & \val{0.71}{0.0} & \val{0.59}{0.0} & \\
      \cmidrule(lr){2-12}
        & \multirow{3}{*}{Our Method}
          & $A_{\mathcal{D}_r^\mathcal{T}} \uparrow$ & \val{71.7}{2.4} & \val{67.8}{0.2} & \val{60.6}{0.6} & \val{83.6}{0.4} & \val{56.1}{0.5} & \val{64.9}{0.7} & \val{75.4}{0.1} & \val{68.6}{0.7} &  \\
        &  & $A_{\mathcal{D}_f^\mathcal{T}} \downarrow$ & \val{0.0}{0.0} & \val{0.0}{0.0} & \val{0.0}{0.0} & \val{0.0}{0.0} & \val{0.0}{0.0} & \val{0.0}{0.0} & \val{0.0}{0.0} & \val{0.0}{0.0} & +6.1\% \\
        % \rowcolor{green!40!white}
        &  & Score $\uparrow$ & \val{0.72}{0.0} & \val{0.68}{0.0} & \val{0.61}{0.0} & \val{0.84}{0.0} & \val{0.56}{0.0} & \val{0.65}{0.0} & \val{0.75}{0.0} & \val{0.69}{0.0} &  \\
      \bottomrule
    \end{tabular}
  }
\end{table}

\subsection{Extensions to Open-set Domain Adaptation} \label{subsec:apd_osda}

in Table~\ref{tab:osda-SCADA-ofhm}, we provide results on the open-set DA setting, replacing the SFDA loss term with the open-set DA version of SHOT~\cite{liang2020we} as described in the paper. This is evaulated on the OfficeHome dataset with source-only classes $\{1,2,3\}$ and target-only classes $\{34, 35, 36, 37, 38\}$. We report full target class accuracy (OS$^*$), shared class accuracy (OS), forget accuracy, and unlearn score. The results show largely similar trends to the original SFDA setting, where existing methods perform poorly by dropping retain accuracy significantly (UNSIR, Lipschitz) or still maintaining high forget accuracy (ZSMU). In contrast, our method demonstrates strong unlearning performance while maintaining high retain accuracy, achieving the best overall unlearn score.

\section{Limitations} \label{sec:limitations}

%\textbf{Addressing Non- tasks.} 
Our work addresses the task of adapting to new domains while unlearning a subset of classes present only in the source domain. Although our proposed method demonstrates strong performance in the SCADA-UL setting (including its extensions C-SCADA-UL and UC-SCADA-UL), our studies have been currently limited to image classification-based domain adaptation tasks. Domain adaptation is also popular for other tasks such as semantic segmentation or generative modeling. For example, adapting language models to company-specific documents to enhance accuracy and avoid irrelevant or incorrect responses, or adapting road sign segmentation models to new geographies containing only a subset of signs. Extending our method to such directions needs to be studied carefully, and would be interesting directions of future work. %would be aCurrently, there is no clear direction to extend our method to such scenarios.

%\textbf{Accurate identification of SEC in UC-SCADA.} 
Beyond the above, in our current implementation, we adapt the $\gamma$ term from PADA~\cite{cao2018partial} for identification of source-exclusive classes. However, we find that this is not always accurate, especially in datasets with large label-spaces such as DomainNet. Our solution for this is to select an excess number of classes and unlearn all of them to increase the likelihood for source-exclusive classes to be included. %However, this may lead to unintended unlearning of those classes. 
Better approaches to identify forget classes or unified methods (that don't require knowledge of forget classes) may help the UC-SCADA-UL setting in particular. Finally, although we present a theorem linking our method to gradients, establishing mathematical guarantees for unlearning itself is challenging, and would be a potential direction for future work.

\section{Implementation Details} \label{sec:implementation}

\subsection{Compute Resources} \label{subsec:compute}

Our experiments were executed on a Linux-based compute cluster each using a single Tesla V100-SXM3-32GB GPU limited to 20 CPU workers. The time taken for running each experiment ranges from 300-1500 seconds.

\subsection{Adapting MU and PDA Methods to SCADA-UL} 
\label{subsec:adapting-works}
We applied the baseline methods during the domain adaptation process when feasible, similar to our approach. Otherwise, we evaluated both the ``before DA'' and ``after DA'' variants of each method, and adopted the variant that performed better with respect to the Unlearn Score metric. This turned out to be ``before DA'' for all these methods, since this variant preserved the majority of retain accuracy (and hence higher unlearn score). As none of the existing MU methods have been implemented on the OfficeHome, Office31 or DomainNet datasets, we found the best hyperparameters for all the baseline MU methods on these datasets. The hyperparameters were tuned to maximize the unlearn score of each method.
\textbf{Retrain.} The source model is retrained without the forget data and this model is domain adapted. Since source data is inaccessible in our setting, this only serves as a gold standard for comparison.
\textbf{Finetune.} The source model is finetuned on the retain data using the SFDA loss. 
\textbf{UNSIR (Unlearning by Selective Impair and Repair).} employs an impair step to reduce the model's performance on the forget class followed by a repair step to restore performance on the retain set. As this method was originally designed for unlearning where labeled retain data is available, we adapt it for application to our setting. While the noise generation step remains unchanged, both impair and repair steps use target data with pseudo-labels as $\mathcal{D}_{r_{sub}}$, as this represents the closest approximation to data in our setting. The algorithm is applied to the source model before the domain adaptation process.
\textbf{ZSMU (Zero-shot Machine Unlearning).} It uses error minimizing-maximizing noise to achieve data-free unlearning. This method is already data-free and is therefore applied directly to the source model before the DA process.
\textbf{Lipschitz Unlearning.} It achieves unlearning by enforcing local Lipschitz regularization, minimizing the change in model outputs with respect to perturbations of the forget samples. As this method utilizes a simple loss function, this loss term is directly swapped with $\mathcal{L}_\text{MU}$ in our method. Moreover, since it requires forget samples, they were replaced with adversarial samples generated from our method.
\textbf{Nabla Tau.} It applies adaptive gradient ascent to forget data along with gradient descent on retain data. This method was also applied to the source model and adversarial forget and retain samples were used.
\textbf{Unlearned(+).} It estimates the influence of forget data on model parameters and removes it. The influence function involves a gradient computed over the forget set and hessian estimated over the retain set. To apply the method in our setting, adversarial samples were generated to compute the forget data gradient over the source model, and pseudo-labeled retain data was used to estimate the hessian. The method was applied on the source model, and after a short SFDA warmup run, it was re-applied with gentler hyperparameters, and then SFDA was resumed until convergence. Taking the additional, small MU step helped to reduce the forget accuracy more, while maintaining retain accuracy. 
\textbf{PADA.} It introduces a $\gamma$ term to downweight the contributions of source-exclusive classes during PDA. The method was modified to fit into our setting by applying the proposed class weight vector on all loss terms in SF(DA)$^2$, alleviating the need for access to source data.
\textbf{SHOT.} It freezes the source classifier (hypothesis) and adapts the feature encoder for target domain learning. We used the SFPDA version of SHOT as described in the paper by setting the $\beta$ term to 0.

 \subsection{Hyperparameters} \label{subsec:hyperparams}
 
 \textbf{Forget Classes.} For single class unlearning in SCADA-UL, we set the forget class $c_\mathcal{F}=1$, for multiple class forgetting in SCADA-UL, we set it to $\mathcal{C_F}=\{1, 2, 3\}$, and for C-SCADA-UL, we set $\mathcal{C_F}^1=\{1,2\}$, $\mathcal{C_F}^2=\{3,4\}$, $\mathcal{C_F}^3=\{5,6\}$. For scenes dataset, we used $c_\mathcal{F}=1$ corresponding to class Grasslands. For medical dataset, we used $c_\mathcal{F}=3$ corresponding to class Edema. For ablation studies, we used $\mathcal{C_F}=\{1, 2, 3\}$ on the OfficeHome dataset.
 \textbf{Data split.} We used an 80-20 split for train and test data in all experiments.
 \textbf{Backbone.} vit-base-patch16-224 pretrained on ImageNet-1K. \textbf{Optimizer.} SGD optimizer with learning rate 1e-2, momentum 0.9, weight decay 1e-3 and nesterov set to True.
 \textbf{LR Scheduler.} Lambda scheduler with gamma set to 1e-3 and decay set to 0.9. 
 \textbf{Epochs.} Source model is trained for 10 Epochs with 1000 steps per epoch using label smoothing with coefficient 0.1 (As it is standard practice in SFDA~\cite{hwang2024sf}). For SCADA-UL, we use 5 epochs and 1000 steps per epoch on OfficeHome, DomainNet, Medical, Scene datasets, and 10 epochs for Office31 on most methods including ours. 
 \textbf{Loss Trade-off ($\alpha$).} We used $\alpha=10.0$ for most experiments with SF(DA)$^2$~\cite{hwang2024sf} as $\mathcal{L}_\text{SFDA}$, for SHOT~\cite{liang2020we}, we used $\alpha=150.0$.
 \textbf{Method Hyperparameters.} As none of the baseline MU methods had implementations on any of our tested datasets, we found the best hyperparameters on these datasets. The hyperparameters were tuned to maximize retain-set accuracy and hence the unlearn score of each method. 
\textbf{Our Method.} We used 4 adversarial samples to compute $\mathcal{L}_\text{MU}$ in our method.
 \textbf{UNSIR.} We used 256 pseudo-labeled target samples with 32 noisy samples, trained over 5 epochs and 8 steps per epoch with learning rate 1e-1, alpha 2e-3, and mean vector $[1, 2, 3]$ . We used SGD optimizer with learning rate 1e-1, momentum 0.9, weight decay 1e-3, and nesterov set to True for both impair and repair steps. Impair was done for only 8 batches of data as we found this led to minimum degradation in retain class performance.
 \textbf{ZSMU.} The error minimizing-maximizing noise consisted of 64 error minimizing samples and 32 error maximizing samples. We used SGD optimizer with learning rate 2e-1, momentum 0.9, weight decay 1e-3 and nesterov set to True and a Lambda scheduler with gamma set to 1e-3 and decay set to 0.75. The ZSMU process was applied for 4 epochs with 8 batches per epoch.
 \textbf{Lipschitz.} Optimizer and LR scheduler are identical to those in SCADA-UL. 4 adversarial samples were generated per forget class and every iteration, each of these samples were iterated through and perturbed five times and the loss was computed. For example, for 3 forget classes, a total of 3 x 4 = 12 adversarial samples were generated and each of these samples were perturbed 5 times (12 x 5 = 60 perturbations). The losses for each perturbation were added together giving us $\mathcal{L}_\text{MU}$ for that step. In the UC-SCADA-UL setting for domain net, we use 3 perturbations instead of 5 due to VRAM limitations.
 \textbf{Nabla Tau.} SGD optimizer with learning rate 5e-2, momentum 0.9, weight decay 1e-3 and nesterov set to True, lambda scheduler with gamma set to 1e-3 and decay set to 0.75. The alpha term was set to 0.02. This method was run for only 100 steps due to its poor performance for larger number of steps.
\textbf{Unlearned(+).} We froze the backbone and linearized only the final ViT classifier layer, ran the Hessian solver for 10 inner iterations with a step size of 0.10 and an L2 curvature penalty of 3e-3; we scaled the one-hot forget targets to 0.7, clipped the tangent update at a global norm of 1e-3, and injected Gaussian noise of 1e-5 before writing the update back. Each forget class was represented by 12 adversarial samples that we re-optimized every iteration, and we drew 12 such minibatches per MU phase. We ran a single-epoch SFDA warmup and then applied a small MU phase configured identically except for a smaller step: weight decay 5e-3, step size 0.03, max update norm 8e-4; post this the SFDA run was continued.
\textbf{PADA.} Identical optimizer and scheduler as SCADA-UL, the gamma term was computed over the entire target retain train dataset.
\textbf{SHOT.} The method used a SGD optimizer with learning rate of backbone 1e-1, and fully connected and bottleneck and rotation classifier with 1.0. Lambda lr scheduler was used with gamma 1e-3 and decay 0.9. The adaptation process was run for 10 epochs with 1000 steps per epoch.

%\textbf{Selecting Unknown Classes.} %\label{subsec:selecting_uc}
% Due to the noisy nature of $\gamma$, we include an excess number of classes beyond the forget classes. We find that selecting $3 \cdot \left|\mathcal{C_F}\right|$ classes works best.
In UC-SCADA-UL, given the inherent noise in the estimation of $\gamma$, we conservatively select a larger set of classes beyond the forget set. Empirically, we find that selecting the $3 \cdot \left|\mathcal{C_F}\right|$ bottom-ranked classes yields the best results.

\subsection{Metrics} \label{subsec:metrics_hpm}

\textit{MIA Accuracy (MIA\%)}. As mentioned in \cref{sec:experiment}, we extend the Membership Inference Attack Accuracy (MIA\%) metric~\cite{golatkar2020forgetting} to suit class-level unlearning. The model is trained to discriminate between the output entropies of retain class data and unseen or out-of-domain (OOD) data. This is implemented by selecting semantically non-overlapping class samples from other datasets (e.g., selecting OfficeHome samples of classes such as calendar, curtains, etc., when experimenting on DomainNet). An ideal method (such as retraining) would be unable to distinguish between the forget class and such OOD classes.

\par
\textit{Forget Class False Negative Rate (Forget FNR)}: This gives the ratio of forget class samples classified as retain samples to the total forget class samples. Given a threshold $k$, the sample is classified as forget class if the predicted softmax score is less than $k$, else as retain class. Figure~\ref{fig:fnrfpr} shows as we increase the threshold, the forget FNR decreases as fewer forget samples are classified as retain classes, as expected.

\par
\textit{Retain Class False Positive Rate (Retain FPR)}: This gives the ratio of retain class samples classified as forget samples to the total retain class samples. Given a threshold $k$, the sample is classified as forget class if the predicted softmax score is less than $k$, else as retain class. Figure~\ref{fig:fnrfpr} shows as we increase the threshold, the retain FPR increases as more retain samples are classified as forget classes, but it rises upto a maximum value of 0.2 in Clipart $\rightarrow$ Art and 0.12 in Art $\rightarrow$ Product, showing that the retain FPR of our method remains low across multiple thresholds. This indicates our method maintains retain class performance robustly across thresholds.

Figure~\ref{fig:fnrfpr} also shows the forget FNR vs retain FPR plot. The plot indicates it is easy to choose a threshold which maximizes forget FNR with very low retain FPR, for instance 0.18 in Art $\rightarrow$ Product which maintains the forgetting to model utility trade-off in our method.

\par
\textit{Calibration Error}: Expected Calibration Error (ECE) quantifies the alignment between the model’s predicted confidence and its actual accuracy (lower is better). We plot ECE on source model and the model after applying our method for both retain and forget splits: retain bars remain low after unlearning with our method, indicating the retain-class calibration is maintained, i.e., does not degrade reliability on retain classes, whereas forget ECE rises substantially, showing the model loses calibration on forget classes, thus becoming uncertain on forget class samples,

\begin{figure}
    \centering
\includegraphics[width=\linewidth]{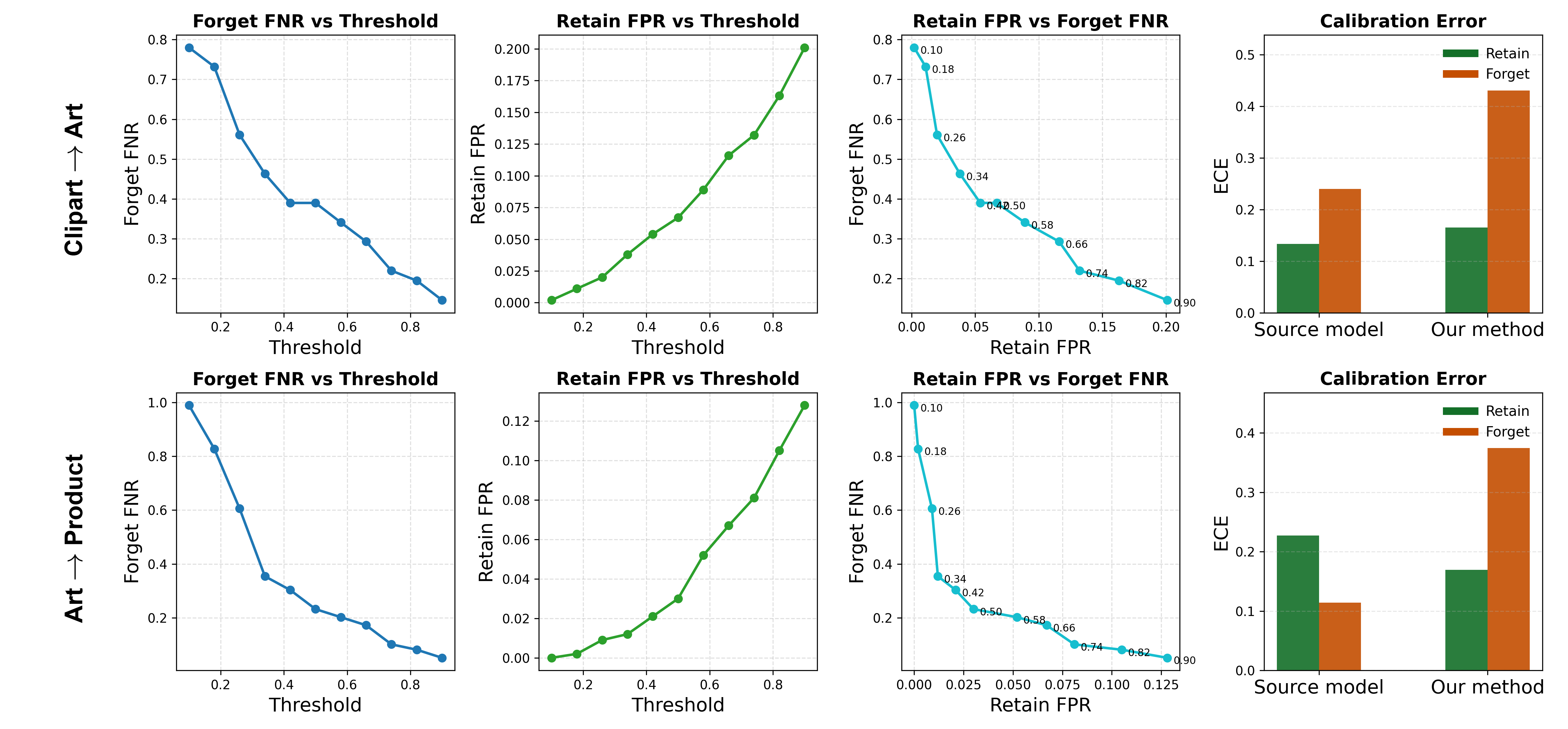}
    \caption{\textbf{Forget/Retain Error and Calibration Analysis.} The figure shows the False Negative Rate (Forget FNR) of our method on forget-class samples across thresholds (first column); the False Positive Rate (Retain FPR) on retain-class samples across thresholds (second column); Forget FNR vs. Retain FPR (third column); and calibration error on retain and forget samples before and after unlearning (fourth column) for two OfficeHome settings. Our method achieves high Forget FNR (the model is confused on forget samples) and low Retain FPR (retain samples remain correctly classified) at appropriate thresholds. Calibration error is high for both forget and retain classes on the source model; after unlearning, retain-class calibration error stays low while forget-class calibration error increases, indicating greater model uncertainty on forget samples.}
    
    \label{fig:fnrfpr}
\end{figure}

\subsection{Real-world Dataset Implementations} \label{subsec:realworld_imple}
\textbf{Scenes Dataset.} Similar to~\cite{song2019domain}, we map 7 similar classes of UC Merced and RSSCN7 datasets (golf course $\rightarrow$ grassland, agricultural $\rightarrow$ farmland, storage tanks $\rightarrow$ industrial region, river $\rightarrow$ river and lake, forest $\rightarrow$ forest, dense residential $\rightarrow$ residential region, parking lot $\rightarrow$ parking lot). The forget class in our experiments is grassland/golf course.
\textbf{Medical Dataset.} Similar to~\cite{he2024domain}, we used 5 overlapping classes in both the datasets to conduct our experiments (Atelectasis, Cardiomegaly, Effusion, Consolidation and Edema). These datasets had some samples with multiple labels, we discarded these samples and retained samples with only single labels. Moreover, we used subsets of the datasets: $\sim$80{,}000 images for CheXpert~\footnote{https://www.kaggle.com/datasets/ashery/chexpert}, and $\sim$12{,}000 images for NIH Chest X-ray~\footnote{https://www.kaggle.com/datasets/khanfashee/nih-chest-x-ray-14-224x224-resized}.

\newpage
\begin{table}[H]
  \captionsetup{font=footnotesize}
  \caption{\footnotesize \textbf{Results for Multi-Class SCADA Unlearning on OfficeHome.} Forget classes are $\mathcal{C_F}=\{1,2,3\}$ \textit{(Best result in bold, second-best underlined)}}
  \vspace{3px}
  \label{tab:mc-SCADA-ofhm}
  \centering
  \resizebox{\textwidth}{!}{%
\begin{tabular}{lcccccccccccccc} \toprule
      Method & Metric & A $\rightarrow$ C & A $\rightarrow$ P & A $\rightarrow$ R & C $\rightarrow$ A & C $\rightarrow$ P & C $\rightarrow$ R & P $\rightarrow$ A & P $\rightarrow$ C & P $\rightarrow$ R & R $\rightarrow$ A & R $\rightarrow$ C & R $\rightarrow$ P & Average \\ \midrule
      \multirow{3}{*}{Original (SF(DA)$^2$~\cite{hwang2024sf})}
      & $A_{\mathcal{D}_r^\mathcal{T}} \uparrow$
      & \val{61.0}{2.0} & \val{82.2}{1.9} & \val{82.3}{0.4} 
      & \val{76.9}{2.1} & \val{82.9}{1.2} & \val{82.5}{1.1} 
      & \val{75.2}{1.7} & \val{58.2}{1.4} & \val{84.3}{1.6} 
      & \val{78.5}{0.2} & \val{58.4}{1.8} & \val{87.0}{0.2} 
      & \val{75.8}{1.3} \\
      & $A_{\mathcal{D}_f^\mathcal{T}} \downarrow$
      & \val{28.7}{2.3} & \val{71.1}{6.2} & \val{86.1}{1.2} 
      & \val{57.1}{1.4} & \val{69.0}{1.6} & \val{76.7}{2.8} 
      & \val{47.8}{2.3} & \val{22.0}{0.5} & \val{72.8}{6.9} 
      & \val{53.1}{4.2} & \val{35.5}{2.3} & \val{77.1}{2.0} 
      & \val{58.1}{2.8} \\
      & Score $\uparrow$ 
      & \val{0.47}{0.0} & \val{0.48}{0.0} & \val{0.44}{0.0} 
      & \val{0.49}{0.0} & \val{0.49}{0.0} & \val{0.47}{0.0} 
      & \val{0.51}{0.0} & \val{0.48}{0.0} & \val{0.49}{0.0} 
      & \val{0.51}{0.0} & \val{0.43}{0.0} & \val{0.49}{0.0} 
      & \val{0.48}{0.0} \\ \midrule
      
      \multirow{3}{*}{Retrain}
      & $A_{\mathcal{D}_r^\mathcal{T}} \uparrow$
      & \val{60.8}{1.4} & \val{83.1}{0.8} & \val{82.2}{0.7} 
      & \val{78.5}{2.5} & \val{82.6}{1.2} & \val{83.3}{1.8} 
      & \val{76.4}{1.2} & \val{59.7}{3.1} & \val{84.4}{0.7} 
      & \val{77.6}{1.6} & \val{59.9}{2.0} & \val{87.3}{0.5} 
      & \val{76.3}{1.5} \\
      & $A_{\mathcal{D}_f^\mathcal{T}} \downarrow$
      & \val{0.0}{0.0} & \val{0.0}{0.0} & \val{0.0}{0.0} 
      & \val{0.0}{0.0} & \val{0.0}{0.0} & \val{0.0}{0.0} 
      & \val{0.0}{0.0} & \val{0.0}{0.0} & \val{0.0}{0.0} 
      & \val{0.0}{0.0} & \val{0.0}{0.0} & \val{0.0}{0.0} 
      & \val{0.0}{0.0} \\
\bottomrule
      
    \end{tabular}
  }
\end{table}

\begin{table}[H]
  \captionsetup{font=footnotesize}
  \caption{\footnotesize \textbf{Results for Multi-Class SCADA Unlearning on DomainNet.} Forget classes are $\mathcal{C_F}=\{1,2,3\}$ \textit{(Best result in bold, second-best underlined)}}
  \label{tab:mc-SCADA-dmnt}
  \vspace{3px}
  \centering
  \resizebox{0.7\textwidth}{!}{%
    \begin{tabular}{lccccccccc} \toprule
      Method & Metric & s $\rightarrow$ p & c $\rightarrow$ s & p $\rightarrow$ c & p $\rightarrow$ r & r $\rightarrow$ s & r $\rightarrow$ c & r $\rightarrow$ p & Average \\ \midrule
      \multirow{3}{*}{Original (SF(DA)$^2$~\cite{hwang2024sf})}
      & $A_{\mathcal{D}_r^\mathcal{T}} \uparrow$
      & \val{71.3}{0.4} & \val{66.6}{0.5} & \val{61.7}{1.0} & \val{78.3}{0.2} & \val{55.9}{1.9} & \val{65.1}{1.1} & \val{75.0}{0.5} & \val{67.7}{0.8} \\
      & $A_{\mathcal{D}_f^\mathcal{T}} \downarrow$
      & \val{67.6}{7.0} & \val{55.9}{5.9} & \val{42.9}{2.7} & \val{60.6}{2.7} & \val{6.5}{1.3} & \val{14.6}{1.1} & \val{22.8}{3.8} & \val{38.7}{3.5} \\
      & Score $\uparrow$
      & \val{0.43}{0.0} & \val{0.43}{0.0} & \val{0.43}{0.0} & \val{0.49}{0.0} & \val{0.53}{0.0} & \val{0.57}{0.0} & \val{0.61}{0.0} & \val{0.50}{0.0} \\ \midrule

      \multirow{3}{*}{Retrain}
      & $A_{\mathcal{D}_r^\mathcal{T}} \uparrow$
      & \val{71.0}{0.4} & \val{65.3}{2.8} & \val{58.6}{2.9} & \val{78.2}{1.1} & \val{54.2}{1.3} & \val{62.9}{1.9} & \val{74.1}{1.3} & \val{66.3}{1.6} \\
      & $A_{\mathcal{D}_f^\mathcal{T}} \downarrow$
      & \val{0.0}{0.0} & \val{0.0}{0.0} & \val{0.0}{0.0} & \val{0.0}{0.0} & \val{0.0}{0.0} & \val{0.0}{0.0} & \val{0.0}{0.0} & \val{0.0}{0.0} \\
      & Score $\uparrow$
      & \val{0.71}{0.0} & \val{0.65}{0.0} & \val{0.59}{0.0} & \val{0.78}{0.0} & \val{0.54}{0.0} & \val{0.63}{0.0} & \val{0.74}{0.0} & \val{0.66}{0.0} \\ \midrule

      \multirow{3}{*}{Finetune}
      & $A_{\mathcal{D}_r^\mathcal{T}} \uparrow$
      & \val{68.9}{0.1} & \val{\underline{64.9}}{1.3} & \val{62.2}{1.4} & \val{77.3}{0.2} & \val{52.1}{2.3} & \val{\underline{65.5}}{1.6} & \val{74.2}{1.0} & \val{\underline{66.5}}{1.1} \\
\bottomrule        

    \end{tabular}
  }
\end{table}
\begin{table}[H]
  \captionsetup{font=footnotesize}
  \caption{\footnotesize \textbf{Results for Multi-Class SCADA Unlearning on Office 31.} Forget classes are $\mathcal{C_F}=\{1,2,3\}$ \textit{(Best result in bold, second-best underlined)}}
  \label{tab:mc-SCADA-of31}
  \vspace{3px}
  \centering
  \resizebox{0.65\textwidth}{!}{%
    \begin{tabular}{lcccccccc} \toprule
      Method & Metric & A $\rightarrow$ D & A $\rightarrow$ W & D $\rightarrow$ A & D $\rightarrow$ W & W $\rightarrow$ A & W $\rightarrow$ D & Average \\ \midrule
      \multirow{3}{*}{Original (SF(DA)$^2$~\cite{hwang2024sf})}
      & $A_{\mathcal{D}_r^\mathcal{T}} \uparrow$
      & \val{77.1}{1.0} & \val{86.3}{1.9} & \val{67.0}{1.2} & \val{79.9}{2.2} & \val{72.0}{0.6} & \val{78.5}{1.8} & \val{76.8}{1.4} \\
      & $A_{\mathcal{D}_f^\mathcal{T}} \downarrow$
      & \val{84.8}{2.8} & \val{99.5}{0.9} & \val{93.6}{1.0} & \val{88.5}{2.8} & \val{97.7}{2.0} & \val{76.6}{2.3} & \val{90.1}{2.0} \\
      & Score $\uparrow$
      & \val{0.42}{0.0} & \val{0.43}{0.0} & \val{0.35}{0.0} & \val{0.42}{0.0} & \val{0.36}{0.0} & \val{0.44}{0.0} & \val{0.40}{0.0} \\ \midrule

      \multirow{3}{*}{Retrain}
      & $A_{\mathcal{D}_r^\mathcal{T}} \uparrow$
      & \val{76.0}{5.0} & \val{87.0}{0.7} & \val{68.2}{1.2} & \val{83.3}{1.0} & \val{72.0}{1.3} & \val{78.0}{1.5} & \val{77.4}{1.8} \\
      & $A_{\mathcal{D}_f^\mathcal{T}} \downarrow$
      & \val{0.0}{0.0} & \val{0.0}{0.0} & \val{0.0}{0.0} & \val{0.0}{0.0} & \val{0.0}{0.0} & \val{0.0}{0.0} & \val{0.0}{0.0} \\
      & Score $\uparrow$
      & \val{0.76}{0.1} & \val{0.87}{0.0} & \val{0.68}{0.0} & \val{0.83}{0.0} & \val{0.72}{0.0} & \val{0.78}{0.0} & \val{0.77}{0.0} \\ \midrule

      \multirow{3}{*}{Finetune}
      & $A_{\mathcal{D}_r^\mathcal{T}} \uparrow$
      & \val{\underline{77.2}}{0.7} & \val{\underline{86.5}}{2.0} & \val{66.7}{2.7} & \val{79.7}{2.8} & \val{\underline{72.0}}{0.6} & \val{\underline{78.4}}{1.2} & \val{76.7}{1.7} \\
\bottomrule

    \end{tabular}
  }
\end{table}

\begin{table}[H]
  \captionsetup{font=footnotesize}
  \caption{\footnotesize \textbf{Results for Single-Class SCADA Unlearning on OfficeHome.} Forget class is $c_\mathcal{F}=1$ \textit{(Best result in bold, second-best underlined)}}
  \label{tab:sc-SCADA-ofhm}
  \vspace{3px}
  \centering
  \resizebox{\textwidth}{!}{%
    \begin{tabular}{lcccccccccccccc} \toprule
      Method & Metric & A $\rightarrow$ C & A $\rightarrow$ P & A $\rightarrow$ R & C $\rightarrow$ A & C $\rightarrow$ P & C $\rightarrow$ R & P $\rightarrow$ A & P $\rightarrow$ C & P $\rightarrow$ R & R $\rightarrow$ A & R $\rightarrow$ C & R $\rightarrow$ P & Average \\ \midrule
      \multirow{3}{*}{Original (SF(DA)$^2$~\cite{hwang2024sf})}
      & $A_{\mathcal{D}_r^\mathcal{T}} \uparrow$
      & \val{60.9}{0.5} & \val{82.2}{1.0} & \val{85.5}{1.8}
      & \val{74.4}{2.4} & \val{83.5}{1.0} & \val{85.6}{1.2}
      & \val{72.1}{3.3} & \val{59.5}{0.8} & \val{86.0}{0.7}
      & \val{75.8}{1.4} & \val{60.0}{2.0} & \val{88.0}{1.4}
      & \val{76.1}{1.5} \\
      & $A_{\mathcal{D}_f^\mathcal{T}} \downarrow$
      & \val{51.2}{9.8} & \val{88.2}{1.5} & \val{89.2}{1.5}
      & \val{42.3}{7.5} & \val{72.7}{2.7} & \val{85.5}{1.2}
      & \val{54.5}{8.6} & \val{50.0}{6.4} & \val{88.6}{2.9}
      & \val{39.9}{2.8} & \val{80.4}{3.6} & \val{93.6}{1.5}
      & \val{69.7}{4.2} \\
      & Score $\uparrow$
      & \val{0.40}{0.0} & \val{0.44}{0.0} & \val{0.45}{0.0}
      & \val{0.52}{0.0} & \val{0.48}{0.0} & \val{0.46}{0.0}
      & \val{0.47}{0.0} & \val{0.40}{0.0} & \val{0.46}{0.0}
      & \val{0.54}{0.0} & \val{0.33}{0.0} & \val{0.45}{0.0}
      & \val{0.45}{0.0} \\ \midrule

      \multirow{3}{*}{Retrain}
      & $A_{\mathcal{D}_r^\mathcal{T}} \uparrow$
      & \val{62.1}{2.4} & \val{83.4}{1.1} & \val{86.0}{1.9}
      & \val{74.7}{3.2} & \val{82.6}{0.6} & \val{86.2}{1.0}
      & \val{72.4}{1.9} & \val{60.0}{3.1} & \val{86.5}{0.9}
      & \val{75.0}{2.9} & \val{61.1}{1.6} & \val{88.3}{1.0}
      & \val{76.5}{1.8} \\
      & $A_{\mathcal{D}_f^\mathcal{T}} \downarrow$
      & \val{0.0}{0.0} & \val{0.0}{0.0} & \val{0.0}{0.0}
      & \val{0.0}{0.0} & \val{0.0}{0.0} & \val{0.0}{0.0}
      & \val{0.0}{0.0} & \val{0.0}{0.0} & \val{0.0}{0.0}
      & \val{0.0}{0.0} & \val{0.0}{0.0} & \val{0.0}{0.0}
      & \val{0.0}{0.0} \\
\bottomrule

    \end{tabular}
  }
\end{table}
\begin{table}[H]
  \captionsetup{font=footnotesize}
  \caption{\footnotesize \textbf{Results for Single-Class SCADA Unlearning on DomainNet.} Forget class is $c_\mathcal{F}=1$ \textit{(Best result in bold, second-best underlined)}}
  \vspace{3px}
  \label{tab:sc-SCADA-dmnt}
  \centering
  \resizebox{0.7\textwidth}{!}{%
    \begin{tabular}{lccccccccc} \toprule
      Method & Metric & s $\rightarrow$ p & c $\rightarrow$ s & p $\rightarrow$ c & p $\rightarrow$ r & r $\rightarrow$ s & r $\rightarrow$ c & r $\rightarrow$ p & Average \\ \midrule
      \multirow{3}{*}{Original (SF(DA)$^2$~\cite{hwang2024sf})}
      & $A_{\mathcal{D}_r^\mathcal{T}} \uparrow$
      & \val{71.1}{0.6} & \val{65.5}{0.4} & \val{62.2}{0.2} & \val{78.2}{0.3} & \val{55.9}{2.6} & \val{65.1}{0.8} & \val{75.1}{1.2} & \val{67.6}{0.9} \\
      & $A_{\mathcal{D}_f^\mathcal{T}} \downarrow$
      & \val{77.6}{5.7} & \val{36.9}{6.9} & \val{35.6}{1.3} & \val{74.1}{12.1} & \val{3.2}{0.4} & \val{17.5}{1.0} & \val{6.1}{0.3} & \val{35.9}{4.0} \\
      & Score $\uparrow$
      & \val{0.40}{0.0} & \val{0.48}{0.0} & \val{0.46}{0.0} & \val{0.45}{0.0} & \val{0.54}{0.0} & \val{0.55}{0.0} & \val{0.71}{0.0} & \val{0.51}{0.0} \\ \midrule

      \multirow{3}{*}{Retrain}
      & $A_{\mathcal{D}_r^\mathcal{T}} \uparrow$
      & \val{71.7}{1.0} & \val{66.6}{1.1} & \val{61.3}{1.9} & \val{78.3}{0.3} & \val{56.0}{3.2} & \val{65.4}{1.1} & \val{74.0}{1.3} & \val{67.6}{1.4} \\
      & $A_{\mathcal{D}_f^\mathcal{T}} \downarrow$
      & \val{0.0}{0.0} & \val{0.0}{0.0} & \val{0.0}{0.0} & \val{0.0}{0.0} & \val{0.0}{0.0} & \val{0.0}{0.0} & \val{0.0}{0.0} & \val{0.0}{0.0} \\
      & Score $\uparrow$
      & \val{0.72}{0.0} & \val{0.67}{0.0} & \val{0.61}{0.0} & \val{0.78}{0.0} & \val{0.56}{0.0} & \val{0.65}{0.0} & \val{0.74}{0.0} & \val{0.68}{0.0} \\ \midrule

      \multirow{3}{*}{Finetune}
      & $A_{\mathcal{D}_r^\mathcal{T}} \uparrow$
      & \val{68.0}{0.3} & \val{\underline{63.3}}{1.6} & \val{\textbf{62.8}}{0.5} & \val{77.1}{0.3} & \val{54.2}{0.3} & \val{\underline{65.8}}{1.7} & \val{73.5}{0.5} & \val{66.4}{0.8} \\
\bottomrule        

    \end{tabular}
  }
\end{table}
\begin{table}[H]
  \captionsetup{font=footnotesize}
  \caption{\footnotesize \textbf{Results for Single-Class SCADA Unlearning on Office 31.} Forget class is $c_\mathcal{F}=1$ \textit{(Best result in bold, second-best underlined)}}
  \label{tab:sc-SCADA-of31}
  \vspace{3px}
  \centering
  \resizebox{0.65\textwidth}{!}{%
    \begin{tabular}{lcccccccc} \toprule
      Method & Metric & A $\rightarrow$ D & A $\rightarrow$ W & D $\rightarrow$ A & D $\rightarrow$ W & W $\rightarrow$ A & W $\rightarrow$ D & Average \\ \midrule
      \multirow{3}{*}{Original (SF(DA)$^2$~\cite{hwang2024sf})}
      & $A_{\mathcal{D}_r^\mathcal{T}} \uparrow$
      & \val{78.7}{2.0} & \val{81.2}{1.1} & \val{61.1}{0.0} & \val{76.4}{1.3} & \val{67.0}{0.6} & \val{80.3}{1.9} & \val{74.1}{1.2} \\
      & $A_{\mathcal{D}_f^\mathcal{T}} \downarrow$
      & \val{91.9}{0.7} & \val{100.}{0.0} & \val{100.}{0.0} & \val{100.}{0.0} & \val{100.}{0.0} & \val{90.6}{0.8} & \val{97.1}{0.2} \\
      & Score $\uparrow$
      & \val{0.41}{0.0} & \val{0.41}{0.0} & \val{0.31}{0.0} & \val{0.38}{0.0} & \val{0.34}{0.0} & \val{0.42}{0.0} & \val{0.38}{0.0} \\ \midrule

      \multirow{3}{*}{Retrain}
      & $A_{\mathcal{D}_r^\mathcal{T}} \uparrow$
      & \val{79.0}{1.0} & \val{82.0}{0.4} & \val{63.9}{0.6} & \val{78.4}{2.1} & \val{67.4}{0.0} & \val{78.1}{2.8} & \val{74.8}{1.1} \\
      & $A_{\mathcal{D}_f^\mathcal{T}} \downarrow$
      & \val{0.0}{0.0} & \val{0.0}{0.0} & \val{0.0}{0.0} & \val{0.0}{0.0} & \val{0.0}{0.0} & \val{0.0}{0.0} & \val{0.0}{0.0} \\
      & Score $\uparrow$
      & \val{0.79}{0.0} & \val{0.82}{0.0} & \val{0.64}{0.0} & \val{0.78}{0.0} & \val{0.67}{0.0} & \val{0.78}{0.0} & \val{0.75}{0.0} \\ \midrule

      \multirow{3}{*}{Finetune}
      & $A_{\mathcal{D}_r^\mathcal{T}} \uparrow$
      & \val{\textbf{79.1}}{1.8} & \val{81.2}{1.3} & \val{59.6}{0.6} & \val{76.0}{1.7} & \val{67.0}{0.6} & \val{\textbf{79.8}}{1.7} & \val{73.8}{1.3} \\
\bottomrule

    \end{tabular}
  }
\end{table}

\begin{table}[H]
  \captionsetup{font=footnotesize}
  \caption{\footnotesize \textbf{Results for UC-SCADA Unlearning on OfficeHome.} Forget classes are $\mathcal{C_F}=\{1,2,3\}$ \textit{(Best result in bold, second-best underlined)}}
  \label{tab:uc-SCADA-ofhm}
  \vspace{3px}
  \centering
  \resizebox{\textwidth}{!}{%
    \begin{tabular}{lcccccccccccccc} \toprule
      Method & Metric & A $\rightarrow$ C & A $\rightarrow$ P & A $\rightarrow$ R & C $\rightarrow$ A & C $\rightarrow$ P & C $\rightarrow$ R & P $\rightarrow$ A & P $\rightarrow$ C & P $\rightarrow$ R & R $\rightarrow$ A & R $\rightarrow$ C & R $\rightarrow$ P & Average \\ \midrule
      \multirow{3}{*}{Original (SF(DA)$^2$~\cite{hwang2024sf})} & $A_{\mathcal{D}_r^\mathcal{T}} \uparrow$ & \val{61.0}{1.6} & \val{82.2}{1.5} & \val{82.3}{0.3} & \val{76.9}{1.7} & \val{82.9}{1.0} & \val{82.5}{0.9} & \val{75.2}{1.4} & \val{58.2}{1.2} & \val{84.3}{1.3} & \val{78.5}{0.2} & \val{58.6}{1.1} & \val{87.0}{0.1} & \val{75.8}{1.0} \\
      & $A_{\mathcal{D}_f^\mathcal{T}} \downarrow$ & \val{28.7}{1.9} & \val{71.1}{5.0} & \val{86.1}{1.0} & \val{57.1}{1.1} & \val{69.0}{1.3} & \val{76.7}{2.3} & \val{47.8}{1.9} & \val{22.0}{0.4} & \val{72.8}{5.7} & \val{53.1}{3.4} & \val{35.8}{1.4} & \val{77.1}{1.7} & \val{58.1}{2.2} \\
      & Score $\uparrow$ & \val{0.47}{0.0} & \val{0.48}{0.0} & \val{0.44}{0.0} & \val{0.49}{0.0} & \val{0.49}{0.0} & \val{0.47}{0.0} & \val{0.51}{0.0} & \val{0.48}{0.0} & \val{0.49}{0.0} & \val{0.51}{0.0} & \val{0.43}{0.0} & \val{0.49}{0.0} & \val{0.48}{0.0} \\ \midrule

      \multirow{3}{*}{Retrain} & $A_{\mathcal{D}_r^\mathcal{T}} \uparrow$ & \val{60.8}{1.2} & \val{83.1}{0.6} & \val{82.2}{0.6} & \val{78.5}{2.1} & \val{82.6}{0.9} & \val{83.3}{1.5} & \val{76.4}{0.9} & \val{59.7}{2.5} & \val{84.4}{0.6} & \val{77.6}{1.3} & \val{59.9}{1.6} & \val{87.3}{0.4} & \val{76.3}{1.2} \\

      & $A_{\mathcal{D}_f^\mathcal{T}} \downarrow$ & \val{0.0}{0.0} & \val{0.0}{0.0} & \val{0.0}{0.0} & \val{0.0}{0.0} & \val{0.0}{0.0} & \val{0.0}{0.0} & \val{0.0}{0.0} & \val{0.0}{0.0} & \val{0.0}{0.0} & \val{0.0}{0.0} & \val{0.0}{0.0} & \val{0.0}{0.0} & \val{0.0}{0.0} \\
\bottomrule

    \end{tabular}
  }
\end{table}

\begin{table}[H]
  \captionsetup{font=footnotesize}
  \caption{\footnotesize \textbf{Results for UC-SCADA Unlearning on DomainNet.} Forget classes are $\mathcal{C_F}=\{1,2,3\}$ \textit{(Best result in bold, second-best underlined)}}
  \label{tab:uc-SCADA-dmnt}
  \vspace{3px}
  \centering
  \resizebox{0.7\textwidth}{!}{%
    \begin{tabular}{lccccccccc} \toprule
      Method & Metric & s $\rightarrow$ p & c $\rightarrow$ s & p $\rightarrow$ c & p $\rightarrow$ r & r $\rightarrow$ s & r $\rightarrow$ c & r $\rightarrow$ p & Average \\ \midrule
      \multirow{3}{*}{Original (SF(DA)$^2$~\cite{hwang2024sf})} & $A_{\mathcal{D}_r^\mathcal{T}} \uparrow$ & \val{71.3}{0.3} & \val{66.6}{0.4} & \val{61.7}{0.8} & \val{78.3}{0.2} & \val{55.9}{1.5} & \val{65.1}{0.9} & \val{75.0}{0.4} & \val{67.7}{0.6} \\

     & $A_{\mathcal{D}_f^\mathcal{T}} \downarrow$ & \val{67.6}{5.7} & \val{55.9}{4.8} & \val{42.9}{2.2} & \val{60.6}{2.2} & \val{6.5}{1.1} & \val{14.6}{0.9} & \val{22.8}{3.1} & \val{38.7}{2.9} \\

     & Score $\uparrow$ & \val{0.43}{0.0} & \val{0.43}{0.0} & \val{0.43}{0.0} & \val{0.49}{0.0} & \val{0.53}{0.0} & \val{0.57}{0.0} & \val{0.61}{0.0} & \val{0.50}{0.0} \\ \midrule

      \multirow{3}{*}{Retrain} & $A_{\mathcal{D}_r^\mathcal{T}} \uparrow$ & \val{71.0}{0.3} & \val{65.3}{2.3} & \val{58.6}{2.4} & \val{78.2}{0.9} & \val{54.2}{1.0} & \val{62.9}{1.5} & \val{74.1}{1.0} & \val{66.3}{1.3} \\

      & $A_{\mathcal{D}_f^\mathcal{T}} \downarrow$ & \val{0.0}{0.0} & \val{0.0}{0.0} & \val{0.0}{0.0} & \val{0.0}{0.0} & \val{0.0}{0.0} & \val{0.0}{0.0} & \val{0.0}{0.0} & \val{0.0}{0.0} \\

      & Score $\uparrow$ & \val{0.71}{0.0} & \val{0.65}{0.0} & \val{0.59}{0.0} & \val{0.78}{0.0} & \val{0.54}{0.0} & \val{0.63}{0.0} & \val{0.74}{0.0} & \val{0.66}{0.0} \\ \midrule

      \multirow{3}{*}{Finetune} & $A_{\mathcal{D}_r^\mathcal{T}} \uparrow$ & \val{68.9}{0.1} & \val{\underline{64.9}}{1.1} & \val{62.2}{1.2} & \val{77.3}{0.2} & \val{52.1}{1.9} & \val{\underline{65.5}}{1.3} & \val{\underline{74.2}}{0.8} & \val{\underline{66.5}}{0.9} \\
\bottomrule

    \end{tabular}
  }
\end{table}
\begin{table}[H]
  \captionsetup{font=footnotesize}
  \caption{\footnotesize \textbf{Results for UC-SCADA Unlearning on Office 31.} Forget classes are $\mathcal{C_F}=\{1,2,3\}$ \textit{(Best result in bold, second-best underlined)}}
  \vspace{3px}
  \label{tab:uc-SCADA-of31}
  \centering
  \resizebox{0.65\textwidth}{!}{%
    \begin{tabular}{lcccccccc} \toprule
      Method & Metric & A $\rightarrow$ D & A $\rightarrow$ W & D $\rightarrow$ A & D $\rightarrow$ W & W $\rightarrow$ A & W $\rightarrow$ D & Average \\ \midrule
      \multirow{3}{*}{Original (SF(DA)$^2$~\cite{hwang2024sf})} & $A_{\mathcal{D}_r^\mathcal{T}} \uparrow$ & \val{67.0}{0.9} & \val{79.9}{1.8} & \val{77.1}{0.9} & \val{86.3}{1.5} & \val{78.5}{1.5} & \val{72.0}{0.5} & \val{76.8}{1.2} \\

      & $A_{\mathcal{D}_f^\mathcal{T}} \downarrow$ & \val{93.6}{0.8} & \val{88.5}{2.3} & \val{84.8}{2.3} & \val{99.5}{0.8} & \val{76.6}{1.9} & \val{97.7}{1.6} & \val{90.1}{1.6} \\

      & Score $\uparrow$ & \val{0.35}{0.0} & \val{0.42}{0.0} & \val{0.42}{0.0} & \val{0.43}{0.0} & \val{0.44}{0.0} & \val{0.36}{0.0} & \val{0.40}{0.0} \\ \midrule

      \multirow{3}{*}{Retrain} & $A_{\mathcal{D}_r^\mathcal{T}} \uparrow$ & \val{68.2}{0.9} & \val{83.3}{0.8} & \val{76.0}{4.1} & \val{87.0}{0.6} & \val{78.0}{1.2} & \val{72.0}{1.0} & \val{77.4}{1.4} \\

      & $A_{\mathcal{D}_f^\mathcal{T}} \downarrow$ & \val{0.0}{0.0} & \val{0.0}{0.0} & \val{0.0}{0.0} & \val{0.0}{0.0} & \val{0.0}{0.0} & \val{0.0}{0.0} & \val{0.0}{0.0} \\

      & Score $\uparrow$ & \val{0.68}{0.0} & \val{0.83}{0.0} & \val{0.76}{0.0} & \val{0.87}{0.0} & \val{0.78}{0.0} & \val{0.72}{0.0} & \val{0.77}{0.0} \\ \midrule

      \multirow{3}{*}{Finetune} & $A_{\mathcal{D}_r^\mathcal{T}} \uparrow$ & \val{66.7}{2.2} & \val{79.7}{2.3} & \val{\underline{77.2}}{0.6} & \val{\textbf{86.5}}{1.6} & \val{\underline{78.4}}{1.0} & \val{72.0}{0.5} & \val{\textbf{76.7}}{1.4} \\
\bottomrule

    \end{tabular}
  }
\end{table}

\begin{table}[H]
\captionsetup{font=footnotesize}
\caption{\footnotesize \textbf{C-SCADA Unlearning performance on the OfficeHome dataset.} Forget classes are $\mathcal{C_F}^1=\{1,2\}, \mathcal{C_F}^2=\{3,4\}, \mathcal{C_F}^3=\{5,6\}$} \label{tab:c-SCADA-ofhm}
\vspace{3px}
\centering
\resizebox{0.75\textwidth}{!}{%
    \begin{tabular}{cccccccccccccccc} \toprule
    Method & Task & Acc. & A $\rightarrow$ C & A $\rightarrow$ P & A $\rightarrow$ R & C $\rightarrow$ A & C $\rightarrow$ P & C $\rightarrow$ R & P $\rightarrow$ A & P $\rightarrow$ C & P $\rightarrow$ R & R $\rightarrow$ A & R $\rightarrow$ C & R $\rightarrow$ P & Average \\ \midrule

    \multirow{9}{*}{Original (SF(DA)$^2$~\cite{hwang2024sf})} & \multirow{3}{*}{T1} & $A_{\mathcal{D}_r^\mathcal{T}} \uparrow$ & \val{63.6}{0.4} & \val{80.3}{0.3} & \val{86.7}{0.8} & \val{78.6}{0.4} & \val{81.9}{0.9} & \val{87.4}{0.4} & \val{76.5}{1.1} & \val{62.2}{0.3} & \val{89.5}{0.4} & \val{81.3}{1.4} & \val{63.4}{1.5} & \val{87.1}{0.2} & \val{78.2}{0.7} \\

    &  & $A_{\mathcal{D}_f^\mathcal{T}} \downarrow$ & \val{43.8}{1.1} & \val{61.4}{2.4} & \val{75.4}{1.1} & \val{59.8}{3.7} & \val{55.7}{1.3} & \val{66.9}{3.2} & \val{46.0}{2.6} & \val{26.2}{0.8} & \val{62.1}{1.2} & \val{58.4}{4.8} & \val{40.3}{0.6} & \val{63.4}{0.6} & \val{55.0}{2.0} \\

    &  & Score $\uparrow$ & \val{0.44}{0.0} & \val{0.50}{0.0} & \val{0.49}{0.0} & \val{0.49}{0.0} & \val{0.53}{0.0} & \val{0.52}{0.0} & \val{0.52}{0.0} & \val{0.49}{0.0} & \val{0.55}{0.0} & \val{0.51}{0.0} & \val{0.45}{0.0} & \val{0.53}{0.0} & \val{0.50}{0.0} \\

    & \multirow{3}{*}{T2} & $A_{\mathcal{D}_r^\mathcal{T}} \uparrow$ & \val{63.6}{0.4} & \val{80.3}{0.3} & \val{86.7}{0.8} & \val{78.6}{0.4} & \val{81.9}{0.9} & \val{87.4}{0.4} & \val{76.5}{1.1} & \val{62.2}{0.3} & \val{89.5}{0.4} & \val{81.3}{1.4} & \val{63.4}{1.5} & \val{87.1}{0.2} & \val{78.2}{0.7} \\

    &  & $A_{\mathcal{D}_f^\mathcal{T}} \downarrow$ & \val{44.2}{1.1} & \val{62.2}{2.4} & \val{76.1}{1.0} & \val{60.6}{3.7} & \val{56.3}{1.3} & \val{67.5}{3.2} & \val{46.7}{2.6} & \val{26.5}{0.9} & \val{62.6}{1.2} & \val{59.2}{4.9} & \val{40.6}{0.6} & \val{64.2}{0.6} & \val{55.6}{2.0} \\
\end{tabular}

}

\end{table}
\begin{table}[H]
\captionsetup{font=footnotesize}
\caption{\footnotesize \textbf{C-SCADA Unlearning performance on the DomainNet dataset.} Forget classes are $\mathcal{C_F}^1=\{1,2\}, \mathcal{C_F}^2=\{3,4\}, \mathcal{C_F}^3=\{5,6\}$} \label{tab:c-SCADA-dmnt}
\vspace{3px}
\centering
\setlength{\tabcolsep}{4pt}
\resizebox{0.5\textwidth}{!}{%
    \begin{tabular}{ccccccccccc}
    \toprule
    Method & Task & Acc. & s $\rightarrow$ p & c $\rightarrow$ s & p $\rightarrow$ c & p $\rightarrow$ r & r $\rightarrow$ s & r $\rightarrow$ c & r $\rightarrow$ p & Average \\ \midrule

    \multirow{9}{*}{Original (SF(DA)$^2$~\cite{hwang2024sf})} & \multirow{3}{*}{T1} & $A_{\mathcal{D}_r^\mathcal{T}} \uparrow$ & \val{71.8}{0.5} & \val{64.9}{0.0} & \val{62.1}{0.6} & \val{78.2}{0.1} & \val{55.6}{0.8} & \val{65.8}{1.6} & \val{75.4}{0.7} & \val{67.7}{0.6} \\

    &  & $A_{\mathcal{D}_f^\mathcal{T}} \downarrow$ & \val{37.2}{1.6} & \val{37.2}{2.9} & \val{37.3}{1.6} & \val{63.2}{3.6} & \val{17.3}{2.1} & \val{10.0}{0.6} & \val{25.5}{3.5} & \val{32.5}{2.3} \\

    &  & Score $\uparrow$ & \val{0.52}{0.0} & \val{0.47}{0.0} & \val{0.45}{0.0} & \val{0.48}{0.0} & \val{0.47}{0.0} & \val{0.60}{0.0} & \val{0.60}{0.0} & \val{0.51}{0.0} \\

    & \multirow{3}{*}{T2} & $A_{\mathcal{D}_r^\mathcal{T}} \uparrow$ & \val{71.8}{0.5} & \val{64.9}{0.0} & \val{62.1}{0.6} & \val{78.2}{0.1} & \val{55.6}{0.8} & \val{65.8}{1.6} & \val{75.4}{0.7} & \val{67.7}{0.6} \\

    &  & $A_{\mathcal{D}_f^\mathcal{T}} \downarrow$ & \val{37.3}{1.6} & \val{37.4}{3.0} & \val{37.7}{1.6} & \val{63.3}{3.6} & \val{17.4}{2.1} & \val{10.1}{0.6} & \val{25.7}{3.5} & \val{32.7}{2.3} \\

    &  & Score $\uparrow$ & \val{0.52}{0.0} & \val{0.47}{0.0} & \val{0.45}{0.0} & \val{0.48}{0.0} & \val{0.47}{0.0} & \val{0.60}{0.0} & \val{0.60}{0.0} & \val{0.51}{0.0} \\

    & \multirow{3}{*}{T3} & $A_{\mathcal{D}_r^\mathcal{T}} \uparrow$ & \val{71.8}{0.5} & \val{64.9}{0.0} & \val{62.1}{0.6} & \val{78.2}{0.1} & \val{55.6}{0.8} & \val{65.8}{1.6} & \val{75.4}{0.7} & \val{67.7}{0.6} \\
\end{tabular}
}
\end{table} 
\begin{table}[H]
\captionsetup{font=footnotesize}
\caption{\footnotesize \textbf{C-SCADA unlearning performance on the Office31 dataset.} Forget classes are $\mathcal{C_F}^1=\{1,2\}, \mathcal{C_F}^2=\{3,4\}, \mathcal{C_F}^3=\{5,6\}$} \label{tab:c-dasec-of31}
\vspace{3px}
\centering
\begin{adjustbox}{max width=0.5\textwidth}
\begin{tabular}{cccccccccc}
\toprule
Method & Task & Acc. & A$\rightarrow$D & A$\rightarrow$W & D$\rightarrow$A & D$\rightarrow$W & W$\rightarrow$A & W$\rightarrow$D & Average \\ \midrule

\multirow{9}{*}{Original (SF(DA)$^2$~\cite{hwang2024sf})} & \multirow{3}{*}{T1} & $A_{\mathcal{D}_r^\mathcal{T}} \uparrow$ & \val{72.6}{2.2} & \val{89.9}{1.3} & \val{73.5}{1.2} & \val{97.4}{1.0} & \val{75.3}{0.3} & \val{80.6}{0.6} & \val{81.5}{1.1} \\

&  & $A_{\mathcal{D}_f^\mathcal{T}} \downarrow$ & \val{74.1}{1.2} & \val{75.6}{0.8} & \val{75.3}{1.5} & \val{83.2}{1.1} & \val{75.6}{0.5} & \val{90.3}{0.9} & \val{79.0}{1.0} \\

&  & Score $\uparrow$ & \val{0.42}{0.0} & \val{0.51}{0.0} & \val{0.42}{0.0} & \val{0.53}{0.0} & \val{0.43}{0.0} & \val{0.42}{0.0} & \val{0.46}{0.0} \\

& \multirow{3}{*}{T2} & $A_{\mathcal{D}_r^\mathcal{T}} \uparrow$ & \val{72.6}{2.2} & \val{89.9}{1.3} & \val{73.5}{1.2} & \val{97.4}{1.0} & \val{75.3}{0.3} & \val{80.6}{0.6} & \val{81.5}{1.1} \\

&  & $A_{\mathcal{D}_f^\mathcal{T}} \downarrow$ & \val{77.4}{1.2} & \val{77.8}{0.8} & \val{76.0}{1.5} & \val{85.7}{1.1} & \val{76.3}{0.5} & \val{94.3}{0.9} & \val{81.2}{1.0} \\

&  & Score $\uparrow$ & \val{0.41}{0.0} & \val{0.51}{0.0} & \val{0.42}{0.0} & \val{0.52}{0.0} & \val{0.43}{0.0} & \val{0.41}{0.0} & \val{0.45}{0.0} \\

& \multirow{3}{*}{T3} & $A_{\mathcal{D}_r^\mathcal{T}} \uparrow$ & \val{72.6}{2.2} & \val{89.9}{1.3} & \val{73.5}{1.2} & \val{97.4}{1.0} & \val{75.3}{0.3} & \val{80.6}{0.6} & \val{81.5}{1.1} \\
\end{tabular}
\end{adjustbox}
\end{table}

\begin{table}[H]
  % \captionsetup{font=footnotesize}
  \caption{ \textbf{Results for Multi-Class SCADA Unlearning on Openset Domain Adaptation on OfficeHome.} Implemented with SHOT-ODA as the open-set domain adaptation loss} % \textit{(Best result in bold, second-best underlined)}}
  \label{tab:osda-SCADA-ofhm}
  \vspace{3px}
  \centering
  \resizebox{\textwidth}{!}{%
    \begin{tabular}{lcccccccccccccc} \toprule
      Method & Metric & A $\rightarrow$ C & A $\rightarrow$ P & A $\rightarrow$ R & C $\rightarrow$ A & C $\rightarrow$ P & C $\rightarrow$ R & P $\rightarrow$ A & P $\rightarrow$ C & P $\rightarrow$ R & R $\rightarrow$ A & R $\rightarrow$ C & R $\rightarrow$ P & Average \\ \midrule

      \multirow{4}{*}{Original (SHOT~\cite{liang2020we})} & OS Acc. $\uparrow$
      & \val{54.5}{1.0} & \val{69.8}{1.6} & \val{75.9}{0.9} & \val{70.3}{1.3} & \val{72.5}{0.5} & \val{76.9}{0.5} & \val{67.7}{2.6} & \val{49.1}{1.8} & \val{75.3}{0.8} & \val{72.8}{1.5} & \val{55.4}{0.8} & \val{76.5}{1.3} & \val{68.0}{1.2} \\

      & OS* Acc. $\uparrow$
      & \val{60.2}{1.2} & \val{77.0}{1.8} & \val{81.8}{0.9} & \val{75.5}{1.3} & \val{79.9}{0.6} & \val{83.0}{0.5} & \val{72.7}{2.8} & \val{54.3}{1.9} & \val{81.3}{0.9} & \val{78.2}{1.7} & \val{61.2}{0.8} & \val{84.3}{1.4} & \val{74.1}{1.3} \\

      & Forget Acc. $\downarrow$
      & \val{43.3}{0.5} & \val{85.6}{1.2} & \val{86.2}{0.7} & \val{68.2}{2.7} & \val{78.9}{1.3} & \val{79.1}{1.5} & \val{65.4}{0.5} & \val{36.5}{4.0} & \val{83.6}{1.5} & \val{70.7}{1.1} & \val{45.0}{3.3} & \val{89.9}{1.6} & \val{69.4}{1.6} \\

      & Score $\uparrow$
      & \val{0.38}{0.0} & \val{0.38}{0.0} & \val{0.41}{0.0} & \val{0.42}{0.0} & \val{0.41}{0.0} & \val{0.43}{0.0} & \val{0.41}{0.0} & \val{0.36}{0.0} & \val{0.41}{0.0} & \val{0.43}{0.0} & \val{0.38}{0.0} & \val{0.40}{0.0} & \val{0.40}{0.0} \\ \midrule

      \multirow{4}{*}{Retrain} & OS Acc. $\uparrow$
        & \val{51.7}{5.5} & \val{67.9}{1.2} & \val{75.2}{1.2} & \val{70.3}{0.9} & \val{71.5}{1.2} & \val{76.4}{0.8} & \val{70.7}{2.8} & \val{51.5}{3.7} & \val{77.2}{0.4} & \val{74.2}{0.2} & \val{58.0}{1.5} & \val{77.4}{2.4} & \val{68.5}{1.8} \\
\end{tabular}
  }
  
\end{table}
% \input{tab/ucon_sfda_domainnet}
% \input{tab/ucon_sfda_office31}

% WARNING: do not forget to delete the supplementary pages from your submission 
% \input{sec/X_suppl}

\end{document}